\documentclass[english]{article}
\usepackage{lmodern}

\usepackage[T1]{fontenc}
\usepackage[utf8]{inputenc}
\usepackage{geometry}
\usepackage{color}
\usepackage{babel}
\usepackage{units}
\usepackage{amsmath}
\usepackage{amsthm}
\usepackage{amssymb}

\usepackage{stmaryrd}
\usepackage{graphicx}
\usepackage{setspace}
\usepackage[authoryear]{natbib}
\usepackage{caption}
\usepackage{subcaption}
\usepackage{algorithm}
\usepackage{phaistos}
\usepackage{protosem}
\usepackage{hieroglf}

\usepackage[unicode=true,
 bookmarks=true,bookmarksnumbered=false,bookmarksopen=false,
 breaklinks=false,pdfborder={0 0 1},backref=page,colorlinks=true]
 {hyperref}
\hypersetup{pdftitle={PAC-Bayes},
 linkcolor=Blue, urlcolor=Blue, citecolor=Blue}

\makeatletter

\numberwithin{equation}{section}
\numberwithin{figure}{section}
\numberwithin{table}{section}
\theoremstyle{plain}
\newtheorem{thm}{\protect\theoremname}
\newtheorem{dfn}{Definition}
\newtheorem{lem}{Lemma}
\newtheorem{cor}{Corollary}
\newtheorem{asm}{Assumption}
\newtheorem{prop}{Proposition}
\theoremstyle{definition}
\newtheorem{rem}{Remark}
\newtheorem{example}{Example}

\DeclareMathOperator*{\argmax}{arg\,max}
\DeclareMathOperator*{\argmin}{arg\,min}

\providecommand{\theoremname}{Theorem}

\usepackage{dsfont}
\usepackage[dvipsnames,svgnames,x11names]{xcolor}
\usepackage{amsmath}

\newcommand{\KLbar}{\overline{KL}}
\newcommand{\Dbar}{\overline{\mathcal{D}}_{\alpha}}
\newcommand{\Vbar}{\overline{\mathcal{V}}}

\makeatother

\providecommand{\theoremname}{Theorem}

\begin{document}
\title{Convergence of Statistical Estimators via Mutual Information Bounds}
\author{EL Mahdi Khribch \& Pierre Alquier
\\
ESSEC Business School}
\maketitle

\begin{abstract}
Recent advances in statistical learning theory have revealed profound connections between mutual information (MI) bounds, PAC-Bayesian theory, and Bayesian nonparametrics.

This work introduces a mutual information bound for statistical models, and derives from it convergence rates for fractional posteriors, for their variational approximations, and for the maximum likelihood estimator. The observations are assumed independent but not identically distributed, and the model is not assumed well-specified, so that the bounds are oracle inequalities and cover regression with a fixed or a conditioned design; the independent and identically distributed, well-specified case is recovered by dropping an index. We illustrate the method on two applications. In the Gaussian sequence model the rate is minimax in both the radius of the Sobolev ball and the sample size, which only appears through its product with the temperature. In logistic regression, where the model is not conjugate and the variational approximation is computed by a stochastic gradient method, the bound applies to the output of the algorithm rather than to an idealized minimizer, and attains the parametric order in the R\'enyi risk with no logarithmic factor, the statistical and the optimization error being separated.
\end{abstract}

\tableofcontents

\section{Introduction and Motivation}

\subsection{Mutual Information Bounds in Machine Learning}
Mutual information (MI) bounds have become a powerful tool for analyzing generalization in machine learning. They provide an information-theoretic perspective on the generalization capabilities of learning algorithms.
\vspace{0.15cm}

The use of information-theoretic approaches in machine learning goes back to Rissanen's Minimum Description Length (MDL) principle~\citep{ris1978}, see~\cite{grunwald2007minimum} for a more recent view on the topic. A related approach is PAC-Bayes bounds~\citep{mcallester1999pac,mcallester1999some}. These bounds were initially developed as Bayesian-inspired generalization bounds in machine learning, but the connection to information theory was highlighted in subsequent works by~\cite{zhang2006information} and~\cite{Catoni2007}. The Kullback-Leibler divergence terms in the PAC-Bayes bound can be interpreted in terms of length of the description of a model, and more generally as a complexity measure. In particular, \cite{Catoni2007} proved that an optimal prior choice in the PAC-Bayes bound makes the Kullback-Leibler divergence term equal to the mutual information between the sample and the parameter. The same quantity controls the bias of adaptive data analysis~\citep{russo2016controlling}, and bounds on the generalization error in terms of the mutual information between the sample and the output of the algorithm~\citep{xu2017information,russo2019much} have received considerable attention in the past years~\citep{negrea2019information,lugosi2022generalization}.
\vspace{0.15cm}

In parallel, it became clear that PAC-Bayes bounds can be used to analyze various kinds of machine learning algorithms. These include empirical risk minimization~\citep{catoni2017dimension}, various kinds of robust estimators~\citep{catoni2012challenging}, and generalized posteriors and their variational approximations~\citep{alquier2016properties}. We refer the reader to~\citet{alquier2024user} for a recent overview on PAC-Bayes, and to~\citet{hellstrom2023generalization} for a global presentation of all information bounds.
\vspace{0.15cm}

PAC-Bayes bounds were also recently used to analyze estimators in a more statistical context, where the objective is inference on parameters. \cite{bhattacharya2016bayesian} analyzed Bayesian fractional posteriors with related tools. This was extended to variational
approximations of tempered and untempered posteriors by~\cite{alquier2019concentration}, \cite{yang2020alpha} and~\cite{zhang2020convergence}. See also~\cite{pmlr-v96-cherief-abdellatif19a,ohn2024adaptive} for results on model selection. Since mutual information bounds in machine learning are optimized PAC-Bayes bounds, a natural question arises: Can we similarly optimize these statistical bounds to achieve tighter results? In this paper, we prove a mutual information version of the PAC-Bayes bounds of~\cite{bhattacharya2016bayesian} and~\cite{alquier2019concentration}, and use it to prove convergence rates for fractional posteriors, variational approximations, and the maximum likelihood estimator (MLE). When applied to fractional posteriors and their variational approximations, it removes a logarithmic factor from the fixed-prior bounds of~\cite{bhattacharya2016bayesian,alquier2019concentration,yang2020alpha} at the parametric endpoint. It thus extends the existing literature on MI bounds, and provides new tools to study a wide range of statistical methods.

\subsection{Bayesian Nonparametrics and Posterior Contraction}
The frequentist properties of Bayesian procedures have attracted sustained interest, particularly in nonparametric settings. Two approaches have been developed. The first is the analysis of posterior concentration rates in the Bayesian nonparametric literature. The second is the PAC-Bayesian framework from statistical learning theory. The two have largely developed independently, and connecting them yields new insights and sharper theoretical guarantees.
\vspace{0.15cm}

The study of posterior concentration rates has a rich history in Bayesian nonparametrics. Seminal work by~\citet{ghosal2000convergence} and~\citet{shen2001rates} established general conditions for posterior consistency and convergence rates. These results were later extended to non-identically distributed observations by~\citet{ghosal2007convergence} and to Gaussian process priors by~\citet{van2008rates}. A comprehensive overview of these developments can be found in~\citet{rousseau2016frequentist,ghosal2017fundamentals} and~\citet{castillo2024bayesian}.
\vspace{0.15cm}

Concurrently, the PAC-Bayesian approach of~\citet{mcallester1999pac} has proven versatile for deriving generalization bounds. It has also received attention in the Bayesian community, as a new perspective on the analysis of Bayesian methods.
\vspace{0.15cm}

The two have already been connected at several points. \citet{grunwald2018fast} provided a unified analysis of excess risk bounds for general loss functions, connecting empirical risk minimization and generalized Bayesian approaches. Moreover,~\citet{zhang2006klentropy,bhattacharya2016bayesian,alquier2019concentration} used PAC-Bayesian techniques to prove the contraction of fractional posteriors under essentially only a prior mass condition, and~\citet{banerjee2021information} derived both PAC-Bayes and mutual information bounds from the same exponential inequality.

\subsection{Our Contributions}

We derive a new MI bound tailored for statistical estimation,
Theorem~\ref{thm:main-thm}. It provides guarantees in terms of statistical
inference, rather than on the generalization risk as
in~\citet{russo2016controlling,xu2017information}. The setting has two
features throughout. First, the model is \emph{misspecified}: the
data-generating distributions $P_i^{\circ}$ need not belong to
$\{P_{i,\theta}:\theta\in\Theta\}$. The bounds therefore take the form of
oracle inequalities, following the misspecified analysis
of~\citet[Section 2.4]{alquier2019concentration}. The well-specified
contraction results are recovered when the data-generating distribution
belongs to the model. Second, the observations are only assumed to be
\emph{independent}, and need not be identically distributed. Every divergence between
the truth and the model is computed observation by observation and then
averaged, in the shorthands $\KLbar$, $\Dbar$ and $\Vbar$
of~\eqref{eq:averages} and~\eqref{eq:varaverage}. In particular, this allows us
to treat regression with a fixed design, or conditionally on a random one. Under the
divergence-equivalence condition of Assumption~\ref{asm:asm1}, the resulting
oracle inequalities involve the Kullback-Leibler divergence on both sides.
This holds both in expectation (Theorem~\ref{thm:thm1}) and with large
probability (Theorem~\ref{thm:thm1-p}).
\vspace{0.15cm}

The applications of Section~\ref{sec:examples} illustrate this generality in three
directions. The first one is computational. Subsection~\ref{subsec:logistic}
treats binary logistic regression with a bounded fixed design. This model is
neither conjugate nor identically distributed across observations, the
localized prior is not in the Gaussian variational family, and the variational
approximation is computed by a projected stochastic gradient method.
Theorem~\ref{thm:logistic} only depends on the estimator through an expected
objective gap $\Delta$. Thus, the statistical and the optimization errors are
separated. Applied to the iterate produced after $T$ steps, it gives the
parametric order $d/n$ in the R\'enyi risk, up to an optimization error that
decays in $T$ (Corollary~\ref{cor:logistic-sgd}). For the same model, prior and
variational family,~\citet[Corollary 3.3]{alquier2019concentration} obtain
$d\log(nd)/n$. We are not aware of a localized-prior bound stated for the output
of a finite-step method. The second direction is nonparametric. In the Gaussian
sequence model over the Sobolev ball
\begin{align*}
    S_b(L)=\Bigl\{\theta:\sum_{\ell}\ell^{2b}\theta_{\ell}^2\le L\Bigr\},
\end{align*}
Corollary~\ref{cor:seq:combined} gives a bound of order
$L^{1/(2b+1)}(n\alpha)^{-2b/(2b+1)}$, by rescaling the prior variances and
optimizing the scale. Both exponents, in the radius $L$ and in the effective
sample size $n\alpha$ of~\citet{lhuillier2023semiparametric}, are minimax over
$S_b(L)$~\citep[Theorem 6.3.8 and Remark 6.3.6]{gine2016mathematical}. The third
direction concerns the MLE. Subsection~\ref{subsec:mle} applies
Theorem~\ref{thm:main-thm} directly, without a prior, to a projected gradient
method with Gaussian noise added to each gradient, in an exponential family with
a strongly convex log-partition function. At a suitable noise level, the mutual
information between the sample and the output of the method is at most $d/2$.
After $T\ge n$ steps, Corollary~\ref{cor:MLE:stoch} then gives a bound of order
$d/n$ in the R\'enyi risk, without logarithmic factor, for instance in the
product Bernoulli model.
\vspace{0.15cm}

The mutual-information form of the PAC-Bayes bound goes back
to~\citet{Catoni2007}, with closely related information-theoretic bounds
in~\citet{zhang2006information}.
Theorem~\ref{thm:main-thm} is its sharp form in the present setting. It is
obtained by optimizing over the prior. From it, we
derive explicit rates under Assumptions~\ref{asm:asm1} to~\ref{asm:asm4}, in
Theorems~\ref{thm:thm1} and~\ref{thm:thm1-p}, and the applications of
Section~\ref{sec:examples}.
\vspace{0.15cm}

Let us now compare our rates with the existing ones. At the parametric
endpoint $\kappa_\pi=1$, and for bounds stated at a fixed $\alpha\in(0,1)$ with
finite-sample constants, the localized prior removes a factor $\log(n)$ from
the complexity term of the corresponding fixed-prior and variational bounds.
This is due to the localization. A fixed prior has to control the complexity
$KL(\rho\|\pi)$ of a competitor concentrated at scale $1/n$. On the other hand,
the localized prior only involves the moment
$\mathbb{E}_{\theta\sim\pi_{-\beta}}[\KLbar(\theta)]$ of
Assumption~\ref{asm:asm2}, and, when a competitor is needed as in
Assumption~\ref{asm:asm4}, the partition function compensates its complexity.
Subsection~\ref{subsec:logistic} works this out in detail. When
$\kappa_\pi<1$, there is no logarithm to remove. Indeed, the classical
contraction results, see for instance~\citet[Theorem 9.21]{ghosal2017fundamentals},
and those of~\citet{lhuillier2023semiparametric}, are already without
logarithmic factors over the Sobolev class, and
Corollary~\ref{cor:seq:combined} matches their exponent in the effective sample
size. Its contribution there is the dependence on the radius $L$. Note also
that these comparisons concern non-asymptotic bounds at a fixed $\alpha$. They
do not concern the asymptotic Bernstein-von Mises statements behind the
parametric rates of the exact posterior. We come back to these points in
Subsections~\ref{subsec:ess} and~\ref{subsec:logistic-discussion}, and to
the case $\alpha=1$ in Remark~\ref{rem:alpha-one}.
\vspace{0.15cm}

The paper is organized as follows. In Section~\ref{section:prelim}, we
introduce our notation and state our MI bound in its most general form,
Theorem~\ref{thm:main-thm}. In Section~\ref{section:main}, we specify this
bound under assumptions on the statistical model, and derive convergence rates
for fractional posteriors and for variational inference, in expectation and
with large probability, together with the form these rates take under
independent and identically distributed observations.
Section~\ref{sec:examples} treats concrete examples: the Gaussian mean and
the Gaussian sequence model, exponential families, smooth models, logistic
regression with its stochastic-gradient variational approximation, and (stochastic approximation of the) MLE. All proofs are provided in Section~\ref{sec:proof}, where we also clarify
the connection between our MI bound and PAC-Bayes bounds.
Section~\ref{sec:discussion} collects three longer discussions: the sense in
which $d_\pi$ is a dimension, the effective sample size, and the logistic
example.
Appendix~\ref{app:compact} treats the same logistic model on a compact
parameter set, where it gives an example of Assumption~\ref{asm:asm4} and a
Kullback-Leibler counterpart of the rate, and Appendix~\ref{app:klrenyi}
explains why that counterpart restricts the support of the prior.

\section{Notations and Preliminaries}
\label{section:prelim}

Consider $n$ independent random variables gathered in a sample $\mathcal{S}=(X_1,\dots,X_n)$. The variable $X_i$ takes its values in a measurable space $(\mathbf{X}_i,\mathcal{X}_i)$ and has (unknown) distribution $P_i^{\circ}$. For each $i$, let $\{P_{i,\theta}, \theta\in\Theta\}$ be a statistical model on that space, that is, a collection of probability distributions indexed by a parameter set $\Theta$ common to all the observations. Let $\mathcal{M}_1^+$ be the set of probability measures on $\Theta$ equipped with a suitable $\sigma$-algebra $\mathcal{B}$.
\vspace{0.15cm}

We do \emph{not} assume that the observations are identically distributed. Regression is the situation we have in mind. With a fixed design, or after conditioning on a random one, the $i$-th observation is drawn from the conditional distribution of the response given the $i$-th covariate. These distributions differ from one observation to the next, while remaining independent. The independent and identically distributed case is recovered throughout by dropping the index $i$, that is, by taking $(\mathbf{X}_i,\mathcal{X}_i)=(\mathbf{X},\mathcal{X})$, $P_i^{\circ}=P^{\circ}$ and $P_{i,\theta}=P_{\theta}$ for every $i$. We collect the corresponding statements in Subsection~\ref{subsec:iid}.
\vspace{0.15cm}

We do \emph{not} assume either that the model is well-specified. The data-generating distributions $P_i^{\circ}$ need not belong to $\{P_{i,\theta} : \theta\in\Theta\}$, that is, the model may be \emph{misspecified}. This is the relevant setting, for instance, in model selection, where all the candidate models but (at most) one are misspecified. The well-specified case is recovered as a special case, namely when there is a $\theta_0\in\Theta$ with $P_{i,\theta_0}=P_i^{\circ}$ for every $i$. Whenever we restrict to it, we write $\theta_0$ for the true parameter.
\vspace{0.15cm}

We assume that the models are dominated, that is, all the $P_{i,\theta}$'s have a density $p_{i,\theta}$ with respect to a specified measure $\mu_i$ (for example, the Lebesgue measure or the counting measure). We also assume that the data-generating distributions are dominated, $P_i^{\circ}\ll\mu_i$, with densities $p_i^{\circ}=\mathrm{d}P_i^{\circ}/\mathrm{d}\mu_i$.
Let $\mathcal{L}_n(\theta)$ denote the likelihood function, and $r_n(\theta)$ the negative log-likelihood ratio with respect to the truth. That is:
\begin{align*}
    \mathcal{L}_n(\theta) &= \prod_{i=1}^n p_{i,\theta}(X_i) \\
    r_n(\theta) &= \sum_{i=1}^n \log\frac{p_i^{\circ}(X_i)}{p_{i,\theta}(X_i)},
\end{align*}
with the convention $\log (p_i^{\circ}(X_i)/0) = +\infty$ if $p_i^{\circ}(X_i)>0$ and $\log(0/0)=0$. Note that the (generalized) posteriors defined below do not depend on the $p_i^{\circ}$'s, since $e^{-\alpha r_n(\theta)}\propto \mathcal{L}_n^\alpha(\theta)$. In the well-specified case, we recover the usual
\begin{align*}
    r_n(\theta)=r_n(\theta,\theta_0)=\sum_{i=1}^n\log\frac{p_{i,\theta_0}(X_i)}{p_{i,\theta}(X_i)}.
\end{align*}
\vspace{0.15cm}

We recall the definitions of the $\alpha$-R\'enyi divergence $\mathcal{D}_{\alpha}$, for $\alpha\in(0,1)$, and of the Kullback-Leibler (KL) divergence, $KL$. Given two probability distributions $P$ and $Q$ defined on the same measurable space, and $\mu$ a reference measure such that both $P$ and $Q$ are absolutely continuous with respect to $\mu$,
\begin{align*}
    \mathcal{D}_{\alpha}(P\|Q)&=\frac{1}{\alpha-1}\log\left(\displaystyle\int \left(\frac{{\rm d} P}{{\rm d}\mu}(x)\right)^{\alpha}\left(\frac{{\rm d} Q}{{\rm d}\mu}(x)\right)^{1-\alpha}\mu({\rm d} x) \right) ,\\
    KL(P\|Q)&= \displaystyle\int  \log\left(\frac{{\rm d} P}{{\rm d}Q}(x)\right) P({\rm d} x)
\end{align*}
with the convention that $KL(P\|Q)=+\infty$ if $P$ is not absolutely continuous with respect to $Q$. The divergence $\mathcal{D}_{\alpha}(P\|Q)$ does not depend on the choice of the reference measure $\mu$.
\vspace{0.15cm}

Each divergence between the truth and the model is computed observation by observation and then averaged. We write
\begin{align}
\label{eq:averages}
    \KLbar(\theta) = \frac{1}{n}\sum_{i=1}^n KL(P_i^{\circ}\|P_{i,\theta}),
    \qquad
    \Dbar(\theta) = \frac{1}{n}\sum_{i=1}^n \mathcal{D}_{\alpha}(P_{i,\theta}\|P_i^{\circ}),
\end{align}
and, for the second moment of the log-likelihood ratio,
\begin{align}
\label{eq:varaverage}
    \Vbar(\theta) = \frac{1}{n}\sum_{i=1}^n \mathbb{E}_{X\sim P_i^{\circ}}\left[ \left( \log\frac{p_i^{\circ}(X)}{p_{i,\theta}(X)} \right)^2 \right].
\end{align}
Both the Kullback-Leibler and the R\'enyi divergences are additive over product measures. It follows that $n\KLbar(\theta)$ and $n\Dbar(\theta)$ are the divergences between the joint law $\bigotimes_{i=1}^n P_i^{\circ}$ of the sample and its model counterpart $\bigotimes_{i=1}^n P_{i,\theta}$. In the identically distributed case,
\begin{align*}
    \KLbar(\theta)=KL(P^{\circ}\|P_{\theta}),
    \qquad
    \Dbar(\theta)=\mathcal{D}_{\alpha}(P_{\theta}\|P^{\circ}),
    \qquad
    \Vbar(\theta)=\mathbb{E}_{X\sim P^{\circ}}\Bigl[\Bigl(\log\frac{p^{\circ}(X)}{p_\theta(X)}\Bigr)^2\Bigr],
\end{align*}
and the sample size disappears from the notation.
\vspace{0.15cm}

\begin{dfn}
    Let $U$ and $V$ be two random variables, $P_{U,V}$ the joint distribution of $U$ and $V$, and $P_U$ the marginal distribution of $U$ and $P_V$ the marginal distribution of $V$. The mutual information (MI) between $U$ and $V$ is defined as:
    \begin{align*}
        \mathcal{I}(U,V):=KL(P_{U,V}\|P_U\otimes P_V).
    \end{align*}
\end{dfn}
We now present our main result. It states a mutual information bound for any randomized procedure, that is, for any probability distribution $\hat{\rho}$ learned from the data, where $\hat{\rho}=\hat{\rho}(\mathcal{S})$. Examples of such procedures include the posterior, fractional posteriors, and variational approximations. By analogy with the Bayesian framework, we will sometimes call $\hat{\rho}$ a (generalized) posterior.
\vspace{0.15cm}

In the notation of the definition above, $U=\theta$ is sampled conditionally on the data,
\begin{align*}
    V=\mathcal{S}=(X_1,\dots,X_n)\sim\bigotimes_{i=1}^n P_i^{\circ},
    \qquad
    U|V\sim\hat{\rho}.
\end{align*} This defines a joint distribution $P_{U,V}$ on the data and the parameter, and hence a mutual information $\mathcal{I}(U,V)=\mathcal{I}(\theta,\mathcal{S})$.
\begin{thm} \label{thm:main-thm}
Given a generalized posterior $\hat{\rho}$, we have, for any $\alpha\in(0,1)$:
\begin{align*}
   \mathbb{E}_{\mathcal{S}}\left[\mathbb{E}_{\theta\sim\hat{\rho}}[\Dbar(\theta)]-\frac{\alpha}{n(1-\alpha)}\mathbb{E}_{\theta\sim\hat{\rho}}[r_n(\theta)]\right] \le \frac{\mathcal{I}(\theta,\mathcal{S})}{n(1-\alpha)}.
\end{align*}
\end{thm}
The proof of Theorem~\ref{thm:main-thm} is provided in Section~\ref{sec:proof}.
\vspace{0.15cm}

Intuitively, it is natural to sample from generalized posteriors $\hat{\rho}$ concentrated on $\theta$ with a large likelihood, that is, for which $r_n(\theta)$ is small. We expect this to make the $P_{i,\theta}$'s close to the $P_i^{\circ}$'s (more precisely, to the best in-model approximation of the truth). Theorem~\ref{thm:main-thm} states that this is provably the case. The averaged R\'enyi divergence $\Dbar(\theta)$ will be small, provided that $\mathcal{I}(\theta,\mathcal{S})/n$ also remains small.
\vspace{0.15cm}

We will see in Section~\ref{section:main} that generalized Bayesian methods such as fractional posteriors and their variational approximations can be tuned to ensure that $\mathcal{I}(\theta,\mathcal{S})/n$ is small. Theorem~\ref{thm:main-thm} then yields tight rates of convergence. In this form the theorem is completely general. There are no regularity assumptions on the densities $p_{i,\theta}$, even though such assumptions can be useful to make the bound tight. The observations need only be independent, and there is no restriction on $\hat{\rho}$. While the bound will give tight results when applied to Bayesian methods, it can also be used to prove convergence of the MLE. This is not entirely straightforward. Taking $\hat{\rho}$ to be the Dirac mass on the MLE will often make $\mathcal{I}(\theta,\mathcal{S})$ infinite. However, taking $\hat{\rho}$ with support on a small neighborhood of the MLE instead lets us analyze the MLE. This is a technique developed by~\citet[Section 5.9]{cat2004} for PAC-Bayes bounds.
\vspace{0.15cm}

For completeness, we recall the definition of the MLE:
\begin{align*}
   \hat{\theta}_n = \argmax_{\theta \in \Theta} \mathcal{L}_n(\theta) = \argmin_{\theta \in \Theta} r_n(\theta)
\end{align*}
when this $\argmin$ exists and is unique. Let $\pi$ be a prior distribution on $(\Theta,\mathcal{B})$. The fractional posterior for $\alpha\in(0,1)$ is defined as:
\begin{align*}
   \pi_{n,\alpha}(\theta):=\frac{e^{-\alpha r_{n}(\theta)}\pi(\theta)}{\mathbb{E}_{\pi}[e^{-\alpha r_{n}(\theta)}]}\propto \mathcal{L}_{n}^{\alpha}(\theta)\pi(\theta).
\end{align*}
Here $\pi$ and $\pi_{n,\alpha}$ denote densities with respect to a common dominating measure on $\Theta$. For $\mathcal{F}\subset \mathcal{M}_1^+$, we define the variational approximation:
\begin{align*}
   \tilde{\rho}_{n,\alpha}= \argmin_{\rho\in \mathcal{F}}\left[\frac{\alpha}{n(1-\alpha)}\mathbb{E}_{\theta \sim \rho}\left[r_n(\theta)\right] +\frac{KL(\rho\|\pi)}{n(1-\alpha)}\right].
\end{align*}
Note that when $\mathcal{F}=\mathcal{M}_1^+$, we have $\tilde{\rho}_{n,\alpha}=\pi_{n,\alpha}$. In practice, variational approximations are useful when $\mathcal{F}$ leads to feasible algorithms; see~\citet{bishop2006pattern} for an introduction.
\vspace{0.15cm}

Many other methods can be analyzed with this theorem. The original MI bound was applied by~\citet{pmlr-v134-neu21a} to a perturbed version of the output of finitely many steps of stochastic gradient descent on the empirical risk. We use our bound in the same way, on the log-likelihood, in Subsection~\ref{subsec:mle}.
\vspace{0.15cm}

In the next section, we formulate assumptions on the model (and the prior $\pi$) that will lead to explicit rates of convergence for the fractional posteriors and their variational approximations.  The MLE is treated in Section~\ref{sec:examples}.

\section{Main Results}
\label{section:main}
Our results rely on two sets of assumptions. Assumptions~\ref{asm:asm2},~\ref{asm:asm3} and~\ref{asm:asm4} bear on the model complexity. Similar conditions appear in~\citet{Catoni2007,bhattacharya2016bayesian,alquier2016properties}. On their own, they would yield the same rates as in these works. An additional condition on the model, Assumption~\ref{asm:asm1}, allows us to obtain improved rates from Theorem~\ref{thm:main-thm}.
\vspace{0.15cm}

Throughout this section, $\theta_0$ no longer plays a special role. The natural benchmark is the best Kullback-Leibler approximation of the truth within the model, that is,
\begin{align*}
   \inf_{\tilde\theta\in\Theta} \KLbar(\tilde\theta).
\end{align*}
This quantity is zero if and only if the model is well-specified, and is otherwise the irreducible approximation error. When the infimum is attained, we write
\begin{align*}
    \theta^{\star}\in\argmin_{\tilde\theta\in\Theta}\KLbar(\tilde\theta)
\end{align*}
for a Kullback-Leibler projection of the truth onto the model. Following the misspecified analysis of~\citet[Section 2.4]{alquier2019concentration}, the complexity assumptions below are stated relative to this benchmark.
\begin{asm}
\label{asm:asm1}
    For any $\alpha\in(0,1)$, there is a $c(\alpha)>0$ such that
    \begin{align*}
        \forall \theta \in \Theta, \quad  \KLbar(\theta)\le  c(\alpha)\Dbar(\theta).
    \end{align*}
\end{asm}

Assumption~\ref{asm:asm1} is stated on the averages~\eqref{eq:averages}, which is the form used in the proofs. It is implied by the corresponding condition observation by observation,
\begin{align}
\label{eq:asm1:peri}
    \forall i\le n,\ \forall \theta\in\Theta,\quad KL(P_i^{\circ}\|P_{i,\theta})\le c(\alpha)\,\mathcal{D}_{\alpha}(P_{i,\theta}\|P_i^{\circ}),
\end{align}
which is the form we verify in the examples of Section~\ref{sec:examples}. The two coincide in the identically distributed case. Note that conditioning on the design matters here. In Gaussian linear regression with a design $Z\sim\mathcal{N}(0,\Sigma_Z)$, the joint law of $(Z,Y)$ gives
\begin{align*}
    KL(P_{\theta_0}\|P_{\theta})&=\frac{s^2}{2v^2},
    \qquad
    s^2=(\theta-\theta_0)^{\top}\Sigma_Z(\theta-\theta_0),
    \\
    \mathcal{D}_{\alpha}(P_{\theta}\|P_{\theta_0})&=\frac{1}{2(1-\alpha)}\log\Bigl(1+\frac{\alpha(1-\alpha)s^2}{v^2}\Bigr),
\end{align*}
so that Assumption~\ref{asm:asm1} fails on $\mathbb{R}^d$. Conditionally on the design, however, $P_{i,\theta}=\mathcal{N}(z_i^{\top}\theta,v^2)$, and the computation of Lemma~\ref{lem_gaussian_asm1} gives~\eqref{eq:asm1:peri} with $c(\alpha)=1/\alpha$.
\vspace{0.15cm}

To understand this property, we recall the following facts~\citep{van2014renyi}. For any $\alpha<1$ and any pair of distributions $p,q$ on the same space,
\[
\mathcal{D}_{\alpha}(q\|p) = \frac{\alpha}{1-\alpha}\mathcal{D}_{1-\alpha}(p\|q) \leq \frac{\alpha}{1-\alpha} KL(p\|q).
\]
The inequality follows from the monotonicity of the R\'enyi divergence in its order, a consequence of Jensen's inequality. Applying this identity with $p=P_i^{\circ}$ and $q=P_{i,\theta}$, condition~\eqref{eq:asm1:peri} is equivalent to
\begin{align*}
    KL(P_i^{\circ}\|P_{i,\theta})\le c(\alpha)\,\frac{\alpha}{1-\alpha}\,\mathcal{D}_{1-\alpha}(P_i^{\circ}\|P_{i,\theta}).
\end{align*}
In this form, both divergences compare $P_i^{\circ}$ and $P_{i,\theta}$ in the same order.
Moreover,
\[\mathcal{D}_{\alpha}(p\|q) \xrightarrow[\alpha\rightarrow 1^{-}]{} KL(p\|q).\]

Thus, Assumption~\ref{asm:asm1} states that the R\'enyi divergence and the KL divergence are equivalent on our statistical model. Conditions of this kind have appeared in various forms. \citet[Lemma 13]{grunwald2018fast} bound the KL divergence by the Hellinger distance under a bounded density ratio, and the earlier~\citet[Theorem 5]{wong1995probability} do so under a moment condition, at the price of a logarithmic factor. In particular,
\begin{equation*}
KL(P_i^{\circ}\|P_{i,\theta})\le \left(2+\log V\right) H^2(P_i^{\circ},P_{i,\theta}),
\end{equation*}
where $V$ is a bound on the density ratio $p_i^{\circ}/p_{i,\theta}$.
It is worth noting that the Hellinger distance $H(P,Q)$ is also related to the R\'enyi divergence, specifically at $\alpha=1/2$:
\begin{equation*}
H^2(P,Q) = 2\left(1 - \exp\left(-\tfrac{1}{2}\mathcal{D}_{1/2}(P\|Q)\right)\right) \leq \mathcal{D}_{1/2}(P\|Q),
\end{equation*}
for the squared Hellinger distance $H^2$. Indeed, writing $\rho(P,Q)$ for the affinity,
\begin{align*}
    H^2(P,Q)&=\int\bigl(\sqrt{{\rm d}P}-\sqrt{{\rm d}Q}\bigr)^2=2\bigl(1-\rho(P,Q)\bigr),
    \\
    \mathcal{D}_{1/2}(P\|Q)&=-2\log\rho(P,Q),
    \qquad
    \rho(P,Q)=\int\sqrt{{\rm d}P\,{\rm d}Q}.
\end{align*}
Models in which $H^2(P,Q)$ is locally equivalent to $KL$ have been studied extensively in asymptotic statistics. We examine several models satisfying Assumption~\ref{asm:asm1} in Section~\ref{sec:examples}.
\vspace{0.15cm}

In the well-specified case, Assumption~\ref{asm:asm1} is a property of the model alone. Under misspecification it becomes a joint condition on the model and the truth, requiring the Kullback-Leibler and R\'enyi divergences to the truth to be equivalent. It remains satisfiable in genuinely misspecified models. For instance, take $P_i^{\circ}=\mathcal{N}(\mu_0,s^2)$ and $P_{i,\theta}=\mathcal{N}(\theta,v^2)$ with $s^2\neq v^2$. Then both $KL(P_i^{\circ}\|P_{i,\theta})$ and $\mathcal{D}_{\alpha}(P_{i,\theta}\|P_i^{\circ})$ are quadratic in $\theta$, with strictly positive leading coefficients. It follows that the inequality holds for a suitable $c(\alpha)$. Under this equivalence, the Kullback-Leibler divergence appears on the left-hand side of the oracle inequalities below.

\begin{asm}
\label{asm:asm2}
For $\beta>0$, we define the following localized prior:
\begin{align*}
    \pi_{-\beta}({\rm d}\theta)=\frac{e^{-\beta \KLbar(\theta)}\pi({\rm d}\theta)}{Z_\beta},
\end{align*}
where
\begin{align*}
 Z_\beta = \int e^{-\beta \KLbar(\theta)}\pi({\rm d}\theta).
\end{align*}
We assume that there is a $\kappa_\pi\in(0,1]$ and a $d_\pi>0$, both independent of $n$, such that
    \begin{align*}
        \sup_{\beta\ge 1}\beta^{\kappa_\pi} \left(\mathbb{E}_{\theta\sim \pi_{-\beta}}\left[ \KLbar(\theta)\right] - \inf_{\tilde\theta\in\Theta} \KLbar(\tilde\theta)\right)\le d_{\pi}.
    \end{align*}
\end{asm}

The bracketed term is the average over $\pi_{-\beta}$ of the excess Kullback-Leibler risk of $P_{i,\theta}$ over the best in-model approximation. This is the misspecified complexity measure used by~\citet[Section 2.4]{alquier2019concentration}. Here it is required to decay at rate $\beta^{-\kappa_\pi}$ as the localization parameter $\beta$ grows. In the well-specified case $\inf_{\tilde\theta\in\Theta} \KLbar(\tilde\theta)=0$ and Assumption~\ref{asm:asm2} reduces to the dimension condition of~\citet{Catoni2007}.
\vspace{0.15cm}

Outside the identically distributed case, the localized prior depends on $n$, through the average $\KLbar$. This is not an issue. Note that $\pi_{-\beta}$ is only used in the proofs, and never has to be computed. However, we require the pair $(\kappa_\pi,d_\pi)$ not to depend on $n$, which is the case in the examples of Section~\ref{sec:examples}.
\vspace{0.15cm}

When $\kappa_\pi=1$, Assumption~\ref{asm:asm2} is the standard condition for parametric rates of convergence. Here $d_{\pi}$ is a notion of model complexity or dimension~\citep{Catoni2007}. Related conditions are discussed by~\cite{baraud2024robust,alquier2024user}. The case $\kappa_\pi<1$ leads to nonparametric rates. We give numerous examples where the assumption holds in Section~\ref{sec:examples}. Moreover, Lemma~\ref{lemmadim} in Section~\ref{sec:discussion} justifies the interpretation of $d_\pi$ as a dimension.
\vspace{0.15cm}

Assumption~\ref{asm:asm2} is related to the usual prior mass condition as follows. Let us write
\begin{align*}
    K(\theta)=\KLbar(\theta)-\inf_{\tilde\theta\in\Theta}\KLbar(\tilde\theta),
    \qquad
    f(\beta)=-\log\int e^{-\beta K(\theta)}\pi({\rm d}\theta).
\end{align*}
The function $f$ is concave and nondecreasing on $(0,\infty)$, with
\begin{align*}
    f(0^+)\ge0,
    \qquad
    f'(\beta)=\mathbb{E}_{\theta\sim\pi_{-\beta}}\bigl[K(\theta)\bigr].
\end{align*}
Thus, for every $\beta>0$ and every $\varepsilon>0$, we have that:
\begin{align}
\label{eq:asm2-priormass}
    \mathbb{E}_{\theta\sim\pi_{-\beta}}\bigl[K(\theta)\bigr]\ \le\ \frac{f(\beta)}{\beta}\ \le\ \varepsilon^2-\frac{1}{\beta}\log\pi\bigl(\{\theta:K(\theta)\le\varepsilon^2\}\bigr).
\end{align}
In particular, if, for every $\beta\ge1$,
\begin{align*}
    \pi\bigl(\{K\le\varepsilon_\beta^2\}\bigr)\ge e^{-\beta\varepsilon_\beta^2}
    \qquad\text{with}\qquad
    \varepsilon_\beta^2\le C\beta^{-\kappa_\pi},
\end{align*}
then Assumption~\ref{asm:asm2} holds with $d_\pi=2C$. At $\beta=n$, this prior mass condition is the Kullback-Leibler part of condition (2.4) of~\citet{ghosal2000convergence}. The converse fails by a logarithmic factor at the parametric endpoint. In the Gaussian mean model of Lemma~\ref{lem_gaussian_asm2_a},
\begin{align*}
    f(\beta)=\frac d2\log\Bigl(1+\frac{\beta\sigma^2}{v^2}\Bigr)+\frac{\beta\|\theta_0\|^2}{2(v^2+\beta\sigma^2)},
\end{align*}
so that $f(\beta)/\beta$ is of order $d\log(\beta)/\beta$, whereas
\begin{align*}
    f'(\beta)\le\frac1\beta\Bigl(\frac d2+\frac{\|\theta_0\|^2}{8\sigma^2}\Bigr).
\end{align*}
In other words, Assumption~\ref{asm:asm2} bounds the slope of $f$ at $\beta$, and the prior mass condition bounds its average slope on $(0,\beta]$. This difference is the logarithmic factor that the localization removes.
\vspace{0.15cm}

Assumptions~\ref{asm:asm1} and~\ref{asm:asm2} suffice to prove the convergence of the fractional posterior in expectation. Contraction with large probability requires the stronger form of Assumption~\ref{asm:asm2} that we state as Assumption~\ref{asm:asm3}.
\begin{asm}
\label{asm:asm3}
We assume that the infimum $\inf_{\tilde\theta\in\Theta} \KLbar(\tilde\theta)$ is attained at some projection $\theta^{\star}$, and that there are $\kappa_\pi\in(0,1]$, $d_\pi>0$ and $d_\pi'>0$, all independent of $n$, such that
    \begin{align*}
        \sup_{\beta\ge 1}\beta^{\kappa_\pi} \left(\mathbb{E}_{\theta\sim \pi_{-\beta}}\left[ \KLbar(\theta)\right]-\inf_{\tilde\theta\in\Theta} \KLbar(\tilde\theta)\right)&\le d_{\pi}, \\
        \sup_{\beta\ge 1}\beta^{\kappa_\pi} \left(\mathbb{E}_{\theta\sim \pi_{-\beta}}\left[ \Vbar(\theta)\right]-\Vbar(\theta^{\star})\right) &\leq d'_{\pi}.
    \end{align*}
\end{asm}
\vspace{0.15cm}

Here $\Vbar(\theta^{\star})$, defined in~\eqref{eq:varaverage}, is the residual fluctuation of the log-likelihood ratio at the projection $\theta^{\star}$. It vanishes in the well-specified case, where $\theta^{\star}=\theta_0$ and the log-ratio is identically zero. It is generally positive under misspecification.
\vspace{0.15cm}

The variational approximation $\tilde{\rho}_{n,\alpha}$ calls for a slightly different assumption, following~\citet{alquier2019concentration}. Assumption~\ref{asm:asm4} extends Assumption~\ref{asm:asm2} to this setting.
\begin{asm}
\label{asm:asm4}
We assume that there exist $\kappa_\pi\in(0,1]$ and $d_\pi>0$, both independent of $n$, such that, for any $\beta\in[1,\alpha n]$, there exists a $\rho\in \mathcal{F}$ such that
$$ \mathbb{E}_{\theta\sim\rho} [\KLbar(\theta)] - \inf_{\tilde\theta\in\Theta} \KLbar(\tilde\theta) + \frac{KL(\rho\|\pi_{-\beta}) }{n} \leq \frac{d_\pi}{\beta^{\kappa_\pi}}.   $$
\end{asm}

We only require this condition for $\beta\in[1,\alpha n]$, which is the range of localizations used by Theorem~\ref{thm:thm1}. Similarly, Assumptions~\ref{asm:asm2} and~\ref{asm:asm3} are only asked for $\beta\ge1$, since the localizations used in the proofs grow with $n$.
\vspace{0.15cm}

Assumption~\ref{asm:asm2} is the special case of Assumption~\ref{asm:asm4} in which $\pi_{-\beta}$ itself belongs to $\mathcal{F}$, so that the competitor $\rho=\pi_{-\beta}$ makes the term $KL(\rho\|\pi_{-\beta})/n$ vanish. This is the case in the Gaussian examples of Subsection~\ref{subsec:gaussian}, where the localized prior remains Gaussian. However, this is not true in general. Appendix~\ref{app:compact} verifies Assumption~\ref{asm:asm4} in the logistic model of Subsection~\ref{subsec:logistic}, restricted to a ball, where $\pi_{-\beta}$ is not in the variational family, and where an explicit competitor has to be produced.
\vspace{0.15cm}

Assumptions~\ref{asm:asm1}, \ref{asm:asm2} and~\ref{asm:asm4} now yield explicit rates in expectation, for the fractional posterior and for its variational approximation.
\begin{thm}
\label{thm:thm1}
Let $\alpha\in(0,1)$ and $1\le\beta < n(1-\alpha)/c(\alpha)$ with $\beta\le\alpha n$. Under Assumptions~\ref{asm:asm1} and~\ref{asm:asm2}, the fractional posterior satisfies
\begin{align*}
   \mathbb{E}_{\mathcal{S}}\mathbb{E}_{\theta\sim\pi_{n,\alpha}}[\KLbar(\theta)]& \le \frac{c(\alpha) \alpha n - \beta c(\alpha) }{n(1-\alpha) - \beta c(\alpha)} \left(\inf_{\tilde\theta\in\Theta} \KLbar(\tilde\theta)+\frac{d_\pi}{\beta^{\kappa_\pi}}\right),
\end{align*}
and, under Assumptions~\ref{asm:asm1} and~\ref{asm:asm4}, its variational approximation satisfies
\begin{align*}
   \mathbb{E}_{\mathcal{S}}\mathbb{E}_{\theta\sim\tilde{\rho}_{n,\alpha}}[\KLbar(\theta)]& \le  \frac{c(\alpha)  n }{n(1-\alpha) - \beta c(\alpha)} \left(\inf_{\tilde\theta\in\Theta} \KLbar(\tilde\theta)+\frac{d_\pi}{\beta^{\kappa_\pi}}\right).
\end{align*}
In particular, if $n(1-\alpha)\ge2c(\alpha)$ and if the model contains a $\tilde\theta$ with $0<\KLbar(\tilde\theta)<\infty$, then $\beta = n(1-\alpha)/(2c(\alpha))$ is admissible, and, at this $\beta$,
\begin{align*}
   \mathbb{E}_{\mathcal{S}}\mathbb{E}_{\theta\sim\pi_{n,\alpha}}[\KLbar(\theta)]
   & \le \alpha\,\frac{2c(\alpha)}{1-\alpha}\,\inf_{\tilde\theta\in\Theta} \KLbar(\tilde\theta) + \alpha\left( \frac{2 c(\alpha)}{1-\alpha} \right)^{1+\kappa_\pi} \frac{ d_\pi}{n^{\kappa_\pi} },
   \\
   \mathbb{E}_{\mathcal{S}}\mathbb{E}_{\theta\sim\tilde{\rho}_{n,\alpha}}[\KLbar(\theta)]
   & \le \frac{2c(\alpha)}{1-\alpha}\,\inf_{\tilde\theta\in\Theta} \KLbar(\tilde\theta) + \left( \frac{2 c(\alpha)}{1-\alpha} \right)^{1+\kappa_\pi} \frac{ d_\pi}{n^{\kappa_\pi} }.
\end{align*}
\end{thm}

At $\beta=n(1-\alpha)/(2c(\alpha))$, the two bounds differ by the factor $\alpha$ alone, and this difference comes from the two assumptions. Assumption~\ref{asm:asm2} controls the localized risk $\mathbb{E}_{\theta\sim\pi_{-\beta}}[\KLbar(\theta)]$, which the proof multiplies by $\alpha-\beta/n$. Assumption~\ref{asm:asm4} controls the sum of a competitor's risk and of its complexity $KL(\rho\|\pi_{-\beta})/n$. Only the first of these two terms may be multiplied by $\alpha-\beta/n$, and thus the factor $\alpha$ does not appear in the variational bound. The exponent $\kappa_\pi$, and hence the rate, is the same.

The condition $\beta\le\alpha n$ is not an extra restriction in any model of interest. Indeed, it is implied by $\beta<n(1-\alpha)/c(\alpha)$ as soon as $c(\alpha)\ge(1-\alpha)/\alpha$. By the floor~\eqref{eq:ess-floor} of Subsection~\ref{subsec:ess}, this last inequality holds whenever the model contains a $\tilde\theta$ with $0<\KLbar(\tilde\theta)<\infty$.
\begin{rem}
Theorem~\ref{thm:thm1} is an oracle inequality of the same form as the
misspecified bound of~\citet[Theorem 2.7]{alquier2019concentration}. The
Kullback-Leibler risk of the fractional posterior is bounded by the best
achievable in-model risk $\inf_{\tilde\theta\in\Theta}\KLbar(\tilde\theta)$,
plus a complexity term of order $d_\pi/n^{\kappa_\pi}$. When $\kappa_\pi=1$,
which holds in several of the examples below, this complexity term is of order
$1/n$. The fixed-prior bounds
of~\citet{bhattacharya2016bayesian,alquier2019concentration} have $\log(n)/n$
instead. Indeed, they measure the complexity through $KL(\rho\|\pi)$ at a
competitor of scale $1/n$, whereas Assumption~\ref{asm:asm2} measures it
through $\mathbb{E}_{\theta\sim\pi_{-\beta}}[\KLbar(\theta)]$ at
$\beta\asymp n$. When $\kappa_\pi<1$, both approaches lead to the same rate. In the well-specified case the oracle term vanishes,
and the bound is a contraction rate. Note that the bound
of~\citet{alquier2019concentration} controls a R\'enyi divergence by a
Kullback-Leibler divergence. Under Assumption~\ref{asm:asm1}, both sides here
are Kullback-Leibler divergences.
\end{rem}

\begin{rem}
The proof of Theorem~\ref{thm:thm1} appears in Section~\ref{sec:proof}. While it can be derived as a corollary of Theorem~\ref{thm:main-thm}, the route we follow in Section~\ref{sec:proof} is slightly different. We begin with PAC-Bayes bounds in expectation based on~\citet{bhattacharya2016bayesian,alquier2019concentration}, and we show that optimization with respect to the prior leads to Theorem~\ref{thm:main-thm}. The ``optimal'' prior, however, does not yield very explicit bounds. Indeed, for any prior $\mu$ that does not depend on the sample, the proof of Theorem~\ref{thm:main-thm} gives
\begin{align*}
    \mathbb{E}_{\mathcal{S}}\bigl[KL(\hat\rho\|\mu)\bigr]=\mathcal{I}(\theta,\mathcal{S})+KL\bigl(\mathbb{E}_{\mathcal{S}}[\hat\rho]\,\|\,\mu\bigr),
\end{align*}
so that Theorem~\ref{thm:main-thm} is the smallest of these PAC-Bayes bounds, attained at the marginal $\mu=\mathbb{E}_{\mathcal{S}}[\hat\rho]$, which cannot be computed. For a fixed prior, the logarithmic factor at the parametric endpoint comes from the second term. For instance, in the Gaussian mean model of Lemma~\ref{lem_gaussian_asm2_a}, with $\hat\rho=\pi_{n,\alpha}$, we have that:
\begin{align*}
    \mathcal{I}(\theta,\mathcal{S})=\frac d2\log\Bigl(1+\frac{n\alpha^2}{n\alpha+v^2/\sigma^2}\Bigr)\ \le\ \frac{d\alpha}{2},
\end{align*}
whereas $\mathbb{E}_{\mathcal{S}}[KL(\pi_{n,\alpha}\|\pi)]$ is equal to $(d/2)\log(1+n\alpha\sigma^2/v^2)$ up to a term bounded in $n$. The localized priors of~\citet{Catoni2007} were developed as approximations of the optimal priors. They provide explicit rates, and Theorem~\ref{thm:thm1} shows that, at $\kappa_\pi=1$, they attain the order $1/n$ of the mutual information bound.
\end{rem}
\vspace{0.15cm}

We now state a variant that holds with large probability rather than in expectation.
\begin{thm}
\label{thm:thm1-p}
Under Assumptions~\ref{asm:asm1}, \ref{asm:asm2} and~\ref{asm:asm3}, we have, for any $\alpha,\delta,\delta'\in(0,1)$ and $\beta = n(1-\alpha)/(2c(\alpha))\ge1$, with probability at least $1-\delta-\delta'$ on the sample $\mathcal{S}$,
\begin{align*}
 \mathbb{E}_{\theta\sim\pi_{n,\alpha}}[\KLbar(\theta)]& \le \alpha\,\frac{2c(\alpha)}{1-\alpha}\,\inf_{\tilde\theta\in\Theta} \KLbar(\tilde\theta) + \alpha\left( \frac{2 c(\alpha)}{1-\alpha} \right)^{1+\kappa_\pi} \frac{ d_\pi}{n^{\kappa_\pi} } \\
   & \quad + \alpha\,\frac{2c(\alpha)}{1-\alpha}\sqrt{\frac{\Vbar(\theta^{\star})+d_\pi'/\beta^{\kappa_\pi}}{n\delta'}} +  \frac{2c(\alpha)\log\frac{1}{\delta}}{n(1-\alpha)}.
\end{align*}
\end{thm}
\vspace{0.15cm}

Theorem~\ref{thm:thm1-p} is stated for the fractional posterior, but nothing in the argument is particular to it. It is derived from Theorem~\ref{thm:thm5} of Section~\ref{sec:proof}, which is proved for the variational approximation $\tilde{\rho}_{n,\alpha}$, at any fixed competitor $\rho\in\mathcal{F}$. Let us take there, in place of $\pi_{-\beta}$, the competitor $\rho$ given by Assumption~\ref{asm:asm4} at the same $\beta$, which lies in $[1,\alpha n]$ when the model contains a $\tilde\theta$ with $0<\KLbar(\tilde\theta)<\infty$. We obtain, with probability at least $1-\delta-\delta'$ on the sample $\mathcal{S}$,
\begin{align*}
 \mathbb{E}_{\theta\sim\tilde{\rho}_{n,\alpha}}[\KLbar(\theta)]& \le \frac{2c(\alpha)}{1-\alpha}\,\inf_{\tilde\theta\in\Theta} \KLbar(\tilde\theta) + \left( \frac{2 c(\alpha)}{1-\alpha} \right)^{1+\kappa_\pi} \frac{ d_\pi}{n^{\kappa_\pi} } \\
   & \quad + \alpha\,\frac{2c(\alpha)}{1-\alpha}\sqrt{\frac{\mathbb{E}_{\theta\sim\rho}[\Vbar(\theta)]}{n\delta'}} +  \frac{2c(\alpha)\log\frac{1}{\delta}}{n(1-\alpha)}.
\end{align*}
\vspace{0.15cm}

The square-root term is an irreducible $O(1/\sqrt{n})$ fluctuation governed by the residual variance $\Vbar(\theta^{\star})$ at the projection. In the well-specified case $\Vbar(\theta^{\star})=0$ and this term is of order $n^{-(\kappa_\pi+1)/2}$, of the same nature as the fluctuation term in~\citet{alquier2019concentration}. The misspecified analysis of~\citet{alquier2019concentration} is carried out in expectation. The high-probability statement here uses the Chebyshev argument of their well-specified analysis, and retains the residual fluctuation at $\theta^{\star}$.

\begin{rem}
\label{rem:log-delta}
The factor $1/\sqrt{\delta'}$ in Theorem~\ref{thm:thm1-p} comes from Chebyshev's inequality in the proof of Theorem~\ref{thm:thm5}. A factor $\log(1/\delta')$ requires exponential moments of the log-likelihood ratio, which are R\'enyi divergences of order larger than one. Indeed, for $\lambda>0$,
\begin{align*}
    \mathbb{E}_{\mathcal{S}}\bigl[e^{\lambda r_n(\theta)}\bigr]=e^{n\lambda\overline{\mathcal{D}}_{1+\lambda}(\theta)},
    \qquad
    \overline{\mathcal{D}}_{1+\lambda}(\theta)=\frac1n\sum_{i=1}^n\mathcal{D}_{1+\lambda}(P_i^{\circ}\|P_{i,\theta}),
\end{align*}
so that, by Jensen's and Markov's inequalities, with probability at least $1-\delta'$,
\begin{align}
\label{eq:chernoff}
    \mathbb{E}_{\theta\sim\rho}[r_n(\theta)]\ \le\ \frac1\lambda\log\mathbb{E}_{\theta\sim\rho}\bigl[e^{n\lambda\overline{\mathcal{D}}_{1+\lambda}(\theta)}\bigr]+\frac{1}{\lambda}\log\frac{1}{\delta'} .
\end{align}
With a fixed prior, and $\rho$ the restriction of $\pi$ to $\{\overline{\mathcal{D}}_{1+\lambda}\le\varepsilon^2\}$, the first term of~\eqref{eq:chernoff} is at most $n\varepsilon^2$, and the prior mass of this neighborhood is condition (C3) of~\citet{zhang2020convergence}. \citet[Lemma B.1]{lhuillier2023semiparametric} and~\citet[Theorem 3.1]{bhattacharya2016bayesian} use Chebyshev's inequality instead, with probability $1-\mathcal{O}(1/(n\varepsilon_n^2))$.
With the localized prior, the same gain follows from a comparison of order $2$. In the well-specified case, assume that $\overline{\mathcal{D}}_{2}\le C\,\KLbar$, which holds with $C=2$ in the Gaussian models and with $C=2\kappa$ in Example~\ref{exm:bernoulli} on a compact set. We take $\rho=\pi_{-\beta}$ and $\lambda=\beta/(2nC)$. Since $\overline{\mathcal{D}}_{1+\lambda}\le\overline{\mathcal{D}}_{2}$ and the function $f$ of~\eqref{eq:asm2-priormass} is concave, the first term of~\eqref{eq:chernoff} is at most
\begin{align*}
    \frac{f(\beta)-f(\beta/2)}{\lambda}\ \le\ \frac{\beta}{2\lambda}\,f'\Bigl(\frac\beta2\Bigr)\ \le\ nC\,d_\pi\Bigl(\frac2\beta\Bigr)^{\kappa_\pi}.
\end{align*}
Using~\eqref{eq:chernoff} in place of Chebyshev's inequality in the proof of Theorem~\ref{thm:thm5}, at $\beta=n(1-\alpha)/(2c(\alpha))\ge2$, we obtain, with probability at least $1-\delta-\delta'$,
\begin{align*}
    \mathbb{E}_{\theta\sim\pi_{n,\alpha}}\bigl[\KLbar(\theta)\bigr]
    \ \le\ \alpha C2^{\kappa_\pi}\Bigl(\frac{2c(\alpha)}{1-\alpha}\Bigr)^{1+\kappa_\pi}\frac{d_\pi}{n^{\kappa_\pi}}
    +\frac{2c(\alpha)}{1-\alpha}\cdot\frac{1}{n}\Bigl(2\alpha C\,\frac{2c(\alpha)}{1-\alpha}\log\frac{1}{\delta'}+\log\frac1\delta\Bigr).
\end{align*}
The term in $1/\sqrt{\delta'}$ of Theorem~\ref{thm:thm1-p} becomes a term in $\log(1/\delta')/n$, and Assumption~\ref{asm:asm3} is not needed.
\end{rem}
\subsection{The Independent and Identically Distributed Case}
\label{subsec:iid}

The proofs of Theorems~\ref{thm:main-thm},~\ref{thm:thm1} and~\ref{thm:thm1-p} use only the independence of the observations. Identical distribution is never needed, but it does simplify the statements considerably. When $P_i^{\circ}=P^{\circ}$ and $P_{i,\theta}=P_\theta$ for every $i$, the averages of~\eqref{eq:averages} reduce to a single pair of divergences. Moreover, the localized prior no longer depends on $n$, and Assumptions~\ref{asm:asm1} to~\ref{asm:asm4} take the form in which they are usually met. We now record the mutual information bound and the in-expectation rate in this case.
\begin{cor}
\label{cor:iid-mi}
Let $X_1,\dots,X_n$ be i.i.d. with distribution $P^{\circ}$. For any generalized posterior $\hat{\rho}$ and any $\alpha\in(0,1)$,
\begin{align*}
   \mathbb{E}_{\mathcal{S}}\left[\mathbb{E}_{\theta\sim\hat{\rho}}[\mathcal{D}_{\alpha}(P_{\theta}\|P^{\circ})]-\frac{\alpha}{n(1-\alpha)}\mathbb{E}_{\theta\sim\hat{\rho}}[r_n(\theta)]\right] \le \frac{\mathcal{I}(\theta,\mathcal{S})}{n(1-\alpha)}.
\end{align*}
\end{cor}

\begin{cor}
\label{cor:iid-rate}
Let $X_1,\dots,X_n$ be i.i.d. with distribution $P^{\circ}$, let $\alpha\in(0,1)$ and let $n(1-\alpha)\ge2c(\alpha)$. Assume that the model contains a $\tilde\theta$ with $0<KL(P^{\circ}\|P_{\tilde\theta})<\infty$, that the comparison below holds for every $\theta\in\Theta$, and that Assumption~\ref{asm:asm2} holds with the localized prior below,
\begin{align*}
    KL(P^{\circ}\|P_{\theta})\le c(\alpha)\,\mathcal{D}_{\alpha}(P_{\theta}\|P^{\circ}),
    \qquad
    \pi_{-\beta}({\rm d}\theta)\propto e^{-\beta KL(P^{\circ}\|P_{\theta})}\pi({\rm d}\theta).
\end{align*}
Then
\begin{align*}
   \mathbb{E}_{\mathcal{S}}\mathbb{E}_{\theta\sim\pi_{n,\alpha}}[KL(P^{\circ}\|P_{\theta})]& \le \alpha\,\frac{2c(\alpha)}{1-\alpha}\,\inf_{\tilde\theta\in\Theta} KL(P^{\circ}\|P_{\tilde\theta}) + \alpha\left( \frac{2 c(\alpha)}{1-\alpha} \right)^{1+\kappa_\pi} \frac{ d_\pi}{n^{\kappa_\pi} }.
\end{align*}
If, moreover, Assumption~\ref{asm:asm4} holds with the same localized prior, then the variational approximation satisfies
\begin{align*}
   \mathbb{E}_{\mathcal{S}}\mathbb{E}_{\theta\sim\tilde{\rho}_{n,\alpha}}[KL(P^{\circ}\|P_{\theta})]& \le \frac{2c(\alpha)}{1-\alpha}\,\inf_{\tilde\theta\in\Theta} KL(P^{\circ}\|P_{\tilde\theta}) + \left( \frac{2 c(\alpha)}{1-\alpha} \right)^{1+\kappa_\pi} \frac{ d_\pi}{n^{\kappa_\pi} }.
\end{align*}
\end{cor}

In the well-specified case $P^{\circ}=P_{\theta_0}$, the oracle term vanishes and Corollary~\ref{cor:iid-rate} is a contraction rate. In other words, the fractional posterior, and its variational approximation, concentrate around $\theta_0$ at the rate $d_\pi n^{-\kappa_\pi}$, without a logarithmic factor. This is the form under which the bound is used in Subsections~\ref{subsec:gaussian} to~\ref{subsec:smooth}.
\section{Applications}

\label{sec:examples}

In this section, we consider various statistical models. We discuss the validity of Assumption~\ref{asm:asm1} and Assumption~\ref{asm:asm2} (or its variants), and apply Theorem~\ref{thm:thm1}.
\vspace{0.15cm}

Subsections~\ref{subsec:gaussian} to~\ref{subsec:smooth} concern independent and identically distributed observations from a well-specified model, that is, $P_i^{\circ}=P_{\theta_0}$ for some $\theta_0\in\Theta$. Then $\inf_{\tilde\theta\in\Theta}\KLbar(\tilde\theta)=0$ and $\Vbar(\theta^{\star})=0$. The oracle terms in Theorems~\ref{thm:thm1} and~\ref{thm:thm1-p} therefore vanish, and these results reduce to the contraction bounds of Corollary~\ref{cor:iid-rate}. In these subsections we write $\theta_0$ for the true parameter and $r_n(\theta,\theta_0)$ for the log-likelihood ratio. Subsection~\ref{subsec:logistic} uses the general setting. There the observations are independent but not identically distributed, and the variational approximation is computed by a stochastic gradient method.

\subsection{Gaussian Model}
\label{subsec:gaussian}
\subsubsection*{Gaussian Mean Estimation}

Consider the case where the model is a $d$-dimensional Gaussian $\mathcal{N}(\theta,v^2 I)$ with known covariance matrix $v^2 I$ and unknown mean $\theta$, and a Gaussian prior on $\Theta=\mathbb{R}^d$:
\begin{align}
    \label{equa:exm:gauss:1}
    \pi(\theta)
    & =\frac{1}{(2\pi)^{d/2}|\Sigma|^{1/2}}\exp\left(-\frac{\theta^{\top} \Sigma^{-1}\theta}{2}\right),
    \\
    \label{equa:exm:gauss:2}
    p_{\theta}(x)
    & =\frac{1}{(2\pi)^{d/2}v^{d}}\exp\left(-\frac{(x-\theta)^{\top}(x-\theta)}{2v^2}\right).
\end{align}
Between two Gaussians with a common covariance, the Kullback-Leibler and the R\'enyi divergences are proportional, with a ratio depending on $\alpha$ alone.
\begin{lem}
\label{lem_gaussian_asm1}
When the model is given by~\eqref{equa:exm:gauss:2},
$$
  KL(P_{\theta_0}\|P_{\theta})=\frac{1}{2v^2}\left\|\theta-\theta_0\right\|^2= \frac{1}{\alpha}\cdot\frac{\alpha}{2v^2}\left\|\theta-\theta_0\right\|^2=\frac{1}{\alpha} \mathcal{D}_{\alpha}(P_{\theta}\|P_{\theta_0}).$$
In other words, Assumption~\ref{asm:asm1} is satisfied with $c(\alpha)=1/\alpha$.
\end{lem}
The constant is explicit and does not depend on the dimension. Note that tilting a Gaussian prior by this quadratic risk again gives a Gaussian. The localized complexity of Assumption~\ref{asm:asm2} can therefore be computed in closed form. It is finite at the parametric exponent $\kappa_\pi=1$.
  \begin{lem}
\label{lem_gaussian_asm2_a}
When the model and the prior are given by~\eqref{equa:exm:gauss:1} and~\eqref{equa:exm:gauss:2},
$$  \beta\mathbb{E}_{\pi_{-\beta}}\left[KL(P_{\theta_0}\|P_{\theta})\right] = \frac{\beta}{2 v^2}\left\| \left(\Sigma^{-1} + \frac{\beta}{v^2}I_d\right)^{-1} \Sigma^{-1}\theta_0 \right\|^2 +\frac{\beta }{2 v^2} {\rm  Tr}\left[ \left(\Sigma^{-1} + \frac{\beta}{v^2} I_d\right)^{-1} \right]. $$
 In particular, when $\Sigma = \sigma^2 I_d$,
$$
    \beta\mathbb{E}_{\pi_{-\beta }}\left[KL(P_{\theta_0}\|P_{\theta})\right]=\frac{ \frac{\beta}{v^2}\left\|\theta_0\right\|^2}{2\sigma^4\left( \frac{\beta}{v^2}+\frac{1}{\sigma^2}\right)^2}+\frac{ \frac{\beta}{v^2} d}{2\left( \frac{\beta}{v^2}+\frac{1}{\sigma^2}\right)} \leq \frac{d}{2}+\frac{\left\|\theta_0\right\|^2}{8 \sigma^2}
$$
and thus Assumption~\ref{asm:asm2} is satisfied with $\kappa_\pi=1$ and $d_\pi=d/2+\left\|\theta_0\right\|^2/(8\sigma^2)$. Moreover, Assumption~\ref{asm:asm3} holds with the same $(\kappa_\pi,d_\pi)$ and a finite constant $d_\pi'$.
 \end{lem}
Lemma~\ref{lem_gaussian_asm1} gives the constant $c(\alpha)$ of Assumption~\ref{asm:asm1}, and Lemma~\ref{lem_gaussian_asm2_a} gives the pair $(\kappa_\pi,d_\pi)$ of Assumption~\ref{asm:asm2}.  Plugging Lemmas~\ref{lem_gaussian_asm1} and~\ref{lem_gaussian_asm2_a} into Theorem~\ref{thm:thm1}, we obtain a parametric rate.
 \begin{cor}
    In the Gaussian model given by~\eqref{equa:exm:gauss:1} and~\eqref{equa:exm:gauss:2} and $\Sigma=\sigma^2 I$, for every $\alpha\in(0,1)$ with $n\alpha(1-\alpha)\ge2$,
  \begin{align*}
    \mathbb{E}_{\mathcal{S}}\mathbb{E}_{\theta\sim\pi_{n,\alpha}}\left[KL(P_{\theta_0}\|P_{\theta})\right]& \le \frac{2d + \frac{\left\|\theta_0\right\|^2}{2\sigma^2}}{\alpha (1-\alpha)^2 n}.
\end{align*}
 \end{cor}
This is Theorem~\ref{thm:thm1} at $\beta=n\alpha(1-\alpha)/2$, in its fractional-posterior branch.
The condition $n\alpha(1-\alpha)\ge2$ ensures that this $\beta$ is at least $1$. By Lemma~\ref{lem_gaussian_asm1}, the squared-$\ell_2$ risk of the fractional posterior is $2v^2$ times the left-hand side.
\vspace{0.15cm}

The rate is $1/n$. The dimension $d$ and the prior scale $\sigma^2$ only appear in the constant. Of course, in this case the fractional posterior has an explicit form, and the expected risk can be derived directly. However, we believe that this simple example illustrates what Theorem~\ref{thm:thm1} gives. Indeed, it recovers that risk with the correct dependence on $d$ and $n$, whereas the results of~\citet{bhattacharya2016bayesian,alquier2016properties} would yield rates in $\log(n)/n$ rather than $1/n$. The term in $\|\theta_0\|^2$ of this bound is an artifact of the localized bound. In this model, the prior bias of the fractional posterior is of order $n^{-2}$.

By the second part of Lemma~\ref{lem_gaussian_asm2_a}, Theorem~\ref{thm:thm1-p} also applies in this model, with the same $(\kappa_\pi,d_\pi)$ as Theorem~\ref{thm:thm1}.

\subsubsection*{Gaussian Sequence Model}
We now examine the Gaussian sequence model. While this model can be viewed as a variant of the previous one and remains simple, it usefully illustrates how Theorem~\ref{thm:thm1} can lead to nonparametric rates. It also shows that the rate obtained is minimax in two parameters at once: the sample size, and the radius of the Sobolev ball on which the truth lies.
\vspace{0.15cm}

Following the setting of~\citet[Section 1.7]{castillo2024bayesian}, let $X_{i,j} = \theta_{j} + \varepsilon_{i,j}$ where $i=1,2,\dots,n$ and $j=1,2,\dots$, the $\varepsilon_{i,j}$ are i.i.d. $\mathcal{N}(0,1)$, and
$$ \theta_0 \in S_{b}(L) = \left\{\theta: \sum_{\ell=1}^{\infty} \ell^{2b} \theta_\ell^2 \leq L \right\}  $$
for some $L,b>0$. The model is infinite-dimensional, and we discard the coordinates $j>n$. That is, we consider $p_\theta$ as the density of $\mathcal{N}((\theta_1,\dots,\theta_n), v^2I_n)$ with $v=1$. By Lemma~\ref{lem_gaussian_asm1}, Assumption~\ref{asm:asm1} holds with $c(\alpha)=1/\alpha$, since
\begin{align}
\label{eq:seq:kl-gauss}
    KL(P_{\theta_0}\|P_{\theta}) = \frac{1}{2}\left\|\theta-\theta_0\right\|^2 = \frac{1}{\alpha}\mathcal{D}_{\alpha}(P_{\theta}\|P_{\theta_0}).
\end{align}
\vspace{0.15cm}

Note that the radius $L$ only appears in the bound of Theorem~\ref{thm:thm1} through the complexity constant $d_\pi$ of Assumption~\ref{asm:asm2}. We control $d_\pi$ through the prior, and take the Gaussian product prior whose variances decay at the rate given by the smoothness $b$. We also keep a free scale $\tau$ on these variances:
\begin{align}
\label{eq:seq:prior}
    \pi = \bigotimes_{j\ge 1} \mathcal{N}(0,\sigma_j^2), \qquad \sigma_j^2 = \tau^2 j^{-1-2b}, \qquad \tau>0.
\end{align}
The scale $\tau$ is the only ingredient we add to the construction of~\citet{castillo2024bayesian}. Let us fix it at the value $\tau_{\star}(L)$ chosen below. The prior~\eqref{eq:seq:prior} then depends on the radius $L$ and on the smoothness $b$, but on no observation. In other words, it is an oracle prior, and Theorem~\ref{thm:thm1} applies to it unchanged.
\vspace{0.15cm}

We first compute the localized complexity of~\eqref{eq:seq:prior} for an arbitrary scale.

\begin{lem}
\label{lem_gaussian_seq_scaled}
In the Gaussian sequence model, with the prior~\eqref{eq:seq:prior} and $\theta_0\in S_b(L)$, we have, for every $\beta>0$ and every $\tau>0$,
\begin{align}
\label{eq:seq:dpi-tau}
    \mathbb{E}_{\theta\sim\pi_{-\beta}}\left[KL(P_{\theta_0}\|P_{\theta})\right]
    \le d_{\pi}(\tau,L)\,\beta^{-\kappa} + \frac{1}{2\beta},
    \qquad \kappa = \frac{2b}{2b+1},
\end{align}
where
\begin{align}
\label{eq:seq:dpi-def}
    d_{\pi}(\tau,L) = \frac{3L}{2}\,\tau^{-\frac{4b}{1+2b}} + \left(\frac{1}{2}+\frac{1}{4b}\right)\tau^{\frac{2}{1+2b}}.
\end{align}
In particular,
\begin{align*}
    \sup_{\beta\ge 1}\ \beta^{\kappa}\, \mathbb{E}_{\theta\sim\pi_{-\beta}}\left[KL(P_{\theta_0}\|P_{\theta})\right] \le d_{\pi}(\tau,L) + \frac{1}{2},
\end{align*}
so that Assumption~\ref{asm:asm2} is satisfied with
\begin{align}
\label{eq:seq:asm2}
    \kappa_\pi = \kappa = \frac{2b}{2b+1}, \qquad d_\pi = d_{\pi}(\tau,L) + \frac{1}{2},
\end{align}
at every localization $\beta\ge 1$.
Moreover, the constant~\eqref{eq:seq:dpi-def} is the sum of a decreasing and of an increasing power of $\tau$. With $A = 3L/2$ and $B = 1/2 + 1/(4b) = (2b+1)/(4b)$, the map $\tau\mapsto d_{\pi}(\tau,L)$ is minimized at
\begin{align}
\label{eq:seq:tau-star}
    \tau_{\star}^2 = \frac{2bA}{B} = \frac{12\,b^2 L}{2b+1},
\end{align}
with minimal value
\begin{align}
\label{eq:seq:C-b}
    d_{\pi}^{\star}(L) = (2b+1)\,A^{\frac{1}{2b+1}}\left(\frac{B}{2b}\right)^{\frac{2b}{2b+1}}
    = C_b\, L^{\frac{1}{2b+1}},
    \qquad
    C_b = (2b+1)\left(\frac{3}{2}\right)^{\frac{1}{2b+1}}\left(\frac{2b+1}{8b^2}\right)^{\frac{2b}{2b+1}}.
\end{align}
\end{lem}

Plugging these constants into Theorem~\ref{thm:thm1} at the localization $\beta^{\star}=n\alpha(1-\alpha)/2$, we obtain the following result.

\begin{cor}
\label{cor:seq:combined}
In the Gaussian sequence model with the prior~\eqref{eq:seq:prior} taken at the scale $\tau=\tau_{\star}$ of~\eqref{eq:seq:tau-star}, for every $\alpha\in(0,1)$ with $n\alpha(1-\alpha)\ge 2$,
\begin{align}
\label{eq:seq:combined}
    \mathbb{E}_{\mathcal{S}}\mathbb{E}_{\theta\sim\pi_{n,\alpha}}\left[KL(P_{\theta_0}\|P_{\theta})\right]
    \le \frac{2^{\kappa}\,(1+\alpha)\left(C_b\, L^{\frac{1}{2b+1}} + \frac{1}{2}\right)}{(1-\alpha)^{1+\kappa}\,(n\alpha)^{\kappa}},
    \qquad \kappa = \frac{2b}{2b+1}.
\end{align}
The norm in~\eqref{eq:seq:kl-gauss} runs over the retained coordinates $j\le n$. The discarded coordinates stay at the prior, and their contribution to the full-sequence risk satisfies
\begin{align*}
    \sum_{j>n}\bigl(\sigma_j^2+\theta_{0,j}^2\bigr)\le\Bigl(\frac{\tau^2}{2b}+L\Bigr)n^{-2b},
\end{align*}
which is of smaller order than $n^{-\kappa}$ since $2b>\kappa$. By~\eqref{eq:seq:kl-gauss}, the squared-$\ell_2$ risk over the retained coordinates is twice the left-hand side of~\eqref{eq:seq:combined}.
\end{cor}

The same argument applies to the variational approximation $\tilde{\rho}_{n,\alpha}$. Here $\mathcal{F}$ is the mean-field family of Gaussian products $\bigotimes_{j\le n}\mathcal{N}(m_j,s_j^2)$ on the retained coordinates. Since $KL(P_{\theta_0}\|P_{\theta})=\|\theta-\theta_0\|^2/2$ separates over coordinates and the prior~\eqref{eq:seq:prior} is a Gaussian product, the localized prior $\pi_{-\beta}$ is again a Gaussian product, and belongs to $\mathcal{F}$. Taking the competitor $\rho=\pi_{-\beta}$ gives $KL(\rho\|\pi_{-\beta})=0$, so Assumption~\ref{asm:asm4} reduces to Assumption~\ref{asm:asm2} with the constants~\eqref{eq:seq:asm2}. It follows from Theorem~\ref{thm:thm1}, at the same localization, that, for every $\alpha\in(0,1)$ with $n\alpha(1-\alpha)\ge 2$,
\begin{align*}
    \mathbb{E}_{\mathcal{S}}\mathbb{E}_{\theta\sim\tilde{\rho}_{n,\alpha}}\left[KL(P_{\theta_0}\|P_{\theta})\right]
    \le \frac{2^{1+\kappa}\left(C_b\, L^{\frac{1}{2b+1}} + \frac{1}{2}\right)}{\alpha\,(1-\alpha)^{1+\kappa}\,(n\alpha)^{\kappa}}.
\end{align*}
\vspace{0.15cm}

Both exponents in Corollary~\ref{cor:seq:combined} are minimax over $S_b(L)$. Indeed, the ball $S_b(L)$ is the Sobolev ball of radius $\sqrt{L}$ in $H^{b}$. Over this ball, the minimax squared-$\ell_2$ risk is
\begin{align}
\label{eq:seq:minimax}
    \inf_{\hat\theta}\ \sup_{\theta_0\in S_b(L)}\ \mathbb{E}_{\theta_0}\bigl[\|\hat\theta-\theta_0\|^2\bigr]
    \asymp \frac{\left(\sqrt{L}\right)^{\frac{2}{2b+1}}}{n^{\kappa}}
    = \frac{L^{\frac{1}{2b+1}}}{n^{\kappa}},
    \qquad \kappa = \frac{2b}{2b+1},
\end{align}
by~\citet[Theorem 6.3.8 and Remark 6.3.6]{gine2016mathematical}, stated there for the unsquared risk, in the regime $nL\gtrsim1$ where the right-hand side is the binding one. Twice the right-hand side of~\eqref{eq:seq:combined}, which bounds the squared-$\ell_2$ risk, is of order $L^{1/(2b+1)}(n\alpha)^{-\kappa}$ for $\alpha$ bounded away from $1$, as long as $C_bL^{1/(2b+1)}\gtrsim1$. It therefore attains both of these exponents, in the effective sample size $n\alpha$. The additive $1/2$ in its constant does not vanish with $L$, so for a radius much smaller than one the comparison in $L$ is no longer informative.
\vspace{0.15cm}

The bound of Corollary~\ref{cor:seq:combined} is sharp for the fractional posterior itself. Indeed, $\pi_{n,\alpha}$ is a Gaussian product whose $j$-th coordinate has variance $(\sigma_j^{-2}+n\alpha)^{-1}$, and $\sigma_j^{-2}\le n\alpha$ for $j\le(n\alpha\tau^2)^{1/(2b+1)}$. Thus, for every $\theta_0$, we have that:
\begin{align}
\label{eq:seq:lower}
    \mathbb{E}_{\mathcal{S}}\mathbb{E}_{\theta\sim\pi_{n,\alpha}}\left[KL(P_{\theta_0}\|P_{\theta})\right]
    \ \ge\ \frac12\sum_{j\le n}\frac{1}{\sigma_j^{-2}+n\alpha}
    \ \ge\ \frac{J}{4n\alpha},
    \qquad
    J=\bigl\lfloor\min\bigl\{n,(n\alpha\tau^2)^{1/(2b+1)}\bigr\}\bigr\rfloor.
\end{align}
When $1\le(n\alpha\tau_\star^2)^{1/(2b+1)}\le n$, the right-hand side is at least
\begin{align*}
    \frac{(\tau_\star^2)^{1/(2b+1)}}{8\,(n\alpha)^{\kappa}},
\end{align*}
and $(\tau_\star^2)^{1/(2b+1)}$ is a constant depending on $b$ times $L^{1/(2b+1)}$. In other words, for $\alpha$ bounded away from $1$ and in the regime $C_bL^{1/(2b+1)}\gtrsim1$, Corollary~\ref{cor:seq:combined} is sharp for this prior in $n\alpha$ and in $L$ jointly, up to a constant depending on $b$.
\vspace{0.15cm}

It is worth noting that the optimality is in the exponents alone. The constant $C_b$ of~\eqref{eq:seq:C-b} is not the sharp (Pinsker) constant of~\eqref{eq:seq:minimax}. Note also that in~\eqref{eq:seq:dpi-def} the bias term decreases in $\tau$ while the variance term increases, and~\eqref{eq:seq:tau-star} is the scale at which the two balance.
\vspace{0.15cm}

Note that the sample size only appears in~\eqref{eq:seq:combined} through the product $n\alpha$. We discuss this effective sample size in Subsection~\ref{subsec:ess}, together with its form in a general model.

\subsection{Exponential Family}
\label{subsec:expfam}
We now turn to more general statistical models. Assumption~\ref{asm:asm1} can be verified in a wide class of models in the exponential family. From here on, and in the corresponding proofs of Section~\ref{sec:proof}, $L$ denotes a Lipschitz or smoothness constant, and no longer the radius of the Sobolev ball of Subsection~\ref{subsec:gaussian}.

\begin{dfn}
   We say that $(P_\theta,\theta\in\Theta)$ is an exponential family with respect to a measure $\mu$ on $(\mathbf{X},\mathcal{X})$ if there exist measurable functions $h:\mathbf{X}\rightarrow \mathbb{R}$ and $s:\mathbf{X}\rightarrow \mathbb{R}^d$ such that:
   \begin{align*}
       p_\theta(x) =\frac{{\rm d} P_\theta}{{\rm d}\mu}(x)=\exp\left(h(x)+\langle s(x),\theta \rangle-\psi(\theta)\right),
   \end{align*}
   where
   \begin{align*}
       \psi(\theta)=\log\int_{\mathbf{X}}\exp\bigl(h(x)+\langle s(x),\theta\rangle\bigr)\,\mu({\rm d}x)
   \end{align*}
   is the log-partition function.
\end{dfn}

Strong convexity of the log-partition function bounds the R\'enyi divergence from below, and a Lipschitz gradient bounds the Kullback-Leibler divergence from above. The ratio of the two constants then gives the constant $c(\alpha)$ of Assumption~\ref{asm:asm1}.

\begin{lem}
\label{lem_exp_1}
   Let $(P_\theta,\theta\in\Theta)$ be an exponential family with respect to a dominating measure $\mu$ on $(\mathbf{X},\mathcal{X})$, and assume that $\Theta$ is a convex set on which $\psi$ is $m$-strongly convex and has a gradient Lipschitz with constant $L$. Define the condition number of $\psi$ as $\kappa=L/m$ (a constant of the model, not to be confused with the exponent $\kappa_\pi$ of Assumption~\ref{asm:asm2}, nor with the Sobolev exponent $\kappa=2b/(2b+1)$ of Subsection~\ref{subsec:gaussian}). For any $\alpha\in(0,1)$ and any $\theta,\theta_0\in\Theta$,
   \begin{align*}
       KL(P_{\theta_0}\|P_{\theta})&\le \frac{L}{2} \|\theta_0-\theta\|^2 \leq  \frac{\kappa}{\alpha} \mathcal{D}_{\alpha}(P_{\theta}\|P_{\theta_0})
   \end{align*}
   and thus Assumption~\ref{asm:asm1} is satisfied with $c(\alpha)=\kappa/\alpha$.
\end{lem}

We now consider two standard discrete families. On a compact parameter set, the log-partition function has curvature bounded above and below in both cases. The condition number is then finite, and Assumption~\ref{asm:asm1} follows from Lemma~\ref{lem_exp_1}.
\begin{example}[Poisson Distribution]
\label{exm:poisson}
Consider the Poisson distribution with probability mass function $\lambda^{x}e^{-\lambda}/x!$.
The canonical decomposition yields $\theta=\log\lambda$ and $s(x)=x$. The log-partition function is given by:
\begin{align*}
   \psi(\theta)=\lambda=e^{\theta}.
\end{align*}
For $\Theta \subseteq [a,b]$, we have $m = e^a$ (strong convexity constant) and $L = e^b$ (Lipschitz constant), giving condition number $\kappa = L/m = e^{b-a}$. Thus, Assumption~\ref{asm:asm1} is satisfied for the Poisson distribution when the parameter set $\Theta$ is compact.
\end{example}

\begin{example}[Bernoulli Distribution]
\label{exm:bernoulli}
Consider the Bernoulli distribution with probability mass function $p^x(1-p)^{1-x}$.
The canonical decomposition yields $\theta=\log[p/(1-p)]$ and $s(x)=x$. The log-partition function is given by:
\begin{align*}
   \psi(\theta)=\log\left(1+e^{\theta}\right).
\end{align*}
The second derivative is:
\begin{align*}
   \psi''(\theta)=\frac{e^{\theta}}{\left(1+e^{\theta}\right)^2}.
\end{align*}
For $\Theta \subseteq [a,b]$, we have $m = \min\{\psi''(a), \psi''(b)\}$ (strong convexity constant), $L = 1/4$ (Lipschitz constant), and condition number $\kappa = L/m = 1/(4m)$. Thus, Assumption~\ref{asm:asm1} is satisfied for the Bernoulli distribution on compact sets.
\end{example}
A discussion of Assumption~\ref{asm:asm2} in these two examples is given in Subsection~\ref{subsec:catodim}.

\subsection{Smooth Models}
\label{subsec:smooth}
The relationships between Fisher information and various divergences are deeply connected to Le Cam's quadratic mean differentiability (QMD)~\citep{le2012asymptotic}, a fundamental concept in asymptotic statistics. Following \citet[Section 7.11]{polyanskiy2025}, we present calculations that verify Assumption~\ref{asm:asm1} holds in a local neighborhood of $\theta_0$. These calculations build upon previous results on the local behavior of the Hellinger distance \citep[Equation (121)]{baraud2017new} and KL divergence \citep[Section 7.2]{van2000asymptotic}, unified under the $f$-divergence framework of \citet{csiszar1967information}. The expansions below require three derivatives of the log-density near $\theta_0$. These derivatives must be dominated, so that their moments exist, and we must be allowed to differentiate under the integral sign. We therefore make the following assumption.
 \begin{asm}
    \label{asm:differentiability}
    The parameter is one-dimensional, $\Theta\subseteq\mathbb{R}$. For all $x \in \mathbf{X}$, the function $\theta \mapsto \log p_\theta(x)$ is three times differentiable on a neighborhood $U$ of $\theta_0$, and we can exchange differentiation and integration with respect to $p_{\theta_0}$ up to the third derivative. Moreover, the three derivatives are dominated on $U$: there are measurable functions $M_1,M_2,M_3$ such that, for all $\theta\in U$ and all $x\in\mathbf{X}$,
    \begin{align*}
        \left|\frac{\partial}{\partial\theta}\log p_\theta(x)\right|\le M_1(x),
        \qquad
        \left|\frac{\partial^2}{\partial\theta^2}\log p_\theta(x)\right|\le M_2(x),
        \qquad
        \left|\frac{\partial^3}{\partial\theta^3}\log p_\theta(x)\right|\le M_3(x),
    \end{align*}
    with
    \begin{align*}
        \bar{\mu} := \sup_{t\in U}\ \mathbb{E}_{p_t}\left[M_1(X)^3+M_1(X)M_2(X)+M_3(X)\right]<\infty .
    \end{align*}
\end{asm}
Under this assumption, we can define the Fisher information.
\begin{dfn}
    The Fisher information at $\theta$ is defined as:
    \[I(\theta) = \mathbb{E}_{p_{\theta}}\left[\left(\frac{\partial}{\partial\theta} \log p_\theta(X)\right)^2\right].\]
\end{dfn}

To second order at $\theta_0$, the squared Hellinger distance, the Kullback-Leibler divergence and the R\'enyi divergence of order $1/2$ are all given by the same Fisher quadratic form, up to the factors $1/4$, $1/2$ and $1/4$. Moreover, their remainders are cubic in $\theta-\theta_0$, with explicit constants.
  \begin{thm}
 \label{thm:fisher_divergences_remainder}
Let Assumption~\ref{asm:differentiability} hold, so that $\Theta\subseteq\mathbb{R}$, and write $\Delta=\theta-\theta_0$. Let $U_0=\{\theta:|\Delta|\le u_0\}$ for any $u_0\le1$ such that $[\theta_0-u_0,\theta_0+u_0]\subseteq U$ and $u_0^2\,(I(\theta_0)+\bar{\mu})\le4$. Then, for all $\theta\in U_0$, we have:
\begin{align*}
 H^2(P_\theta, P_{\theta_0}) &= \frac{1}{4}I(\theta_0)\Delta^2 + r_H(\theta) \\
 KL(P_{\theta_0} \| P_\theta) &= \frac{1}{2}I(\theta_0)\Delta^2 + r_K(\theta) \\
 \mathcal{D}_{1/2}(P_\theta \| P_{\theta_0}) &= \frac{1}{4}I(\theta_0)\Delta^2 + r_R(\theta)
 \end{align*}
and the remainder terms are dominated at the cubic order:
 \begin{align*}
 |r_H(\theta)| &\le \frac{\bar{\mu}}{4}\,|\Delta|^3, \\
 |r_K(\theta)| &\le \frac{\bar{\mu}}{6}\,|\Delta|^3, \\
 |r_R(\theta)| &\le \left(\frac{\bar{\mu}}{4}+\frac{\bigl(I(\theta_0)+\bar{\mu}\bigr)^2}{32}\right)|\Delta|^3 .
 \end{align*}
\end{thm}
Near $\theta_0$, any two of the divergences are therefore proportional up to remainders of cubic order. A two-sided bound on the Fisher information then yields the two-sided inequalities required by Assumption~\ref{asm:asm1}.
\begin{cor}
 \label{cor:fisher_divergences_remainder}
Let Assumption~\ref{asm:differentiability} hold and assume there exist constants $m$ and $M$ such that $0 < m \leq I(\theta) \leq M < \infty$ for all $\theta\in U$. Then, with $C=\bar{\mu}/4+(M+\bar{\mu})^2/32$ and $\Delta=\theta-\theta_0$, for all $\theta\in U_0$:
\begin{align*}
\frac{m}{4}\Delta^2 - C|\Delta|^3 &\leq H^2(P_\theta, P_{\theta_0}) \leq \frac{M}{4}\Delta^2 + C|\Delta|^3 \\
\frac{m}{2}\Delta^2 - C|\Delta|^3 &\leq KL(P_{\theta_0} \| P_\theta) \leq \frac{M}{2}\Delta^2 + C|\Delta|^3 \\
\frac{m}{4}\Delta^2 - C|\Delta|^3 &\leq \mathcal{D}_{1/2}(P_\theta \| P_{\theta_0}) \leq \frac{M}{4}\Delta^2 + C|\Delta|^3 .
\end{align*}
Consequently, for every $\epsilon\in(0,1)$, for all $\theta\in U_0$ with $|\Delta|\le \epsilon m/(4C)$ and every $\alpha\in(0,1)$,
\begin{align}
\label{eq:fisher:calpha}
    KL(P_{\theta_0}\|P_{\theta}) \le \frac{2(1+\epsilon)}{\alpha(1-\epsilon)}\,\mathcal{D}_{\alpha}(P_{\theta}\|P_{\theta_0}) .
\end{align}
\end{cor}
Thus, if the model is smooth enough, then Assumption~\ref{asm:asm1} is satisfied, at least in neighborhoods of $\theta_0$, with the explicit constant $c(\alpha)=2(1+\epsilon)/[\alpha(1-\epsilon)]$ of~\eqref{eq:fisher:calpha}. Note that this expansion is stated for a one-dimensional parameter. For multivariate product models, the comparison holds coordinate by coordinate, since both divergences are additive over product measures.

\subsection{Logistic Regression}
\label{subsec:logistic}

We now turn to logistic regression. This model is not conjugate, and it uses the generality of Section~\ref{section:prelim}. Indeed, the observations are independent and not identically distributed. The localized prior leaves the variational family, so a competitor has to be exhibited. Moreover, the variational approximation is not available in closed form, and is computed by a stochastic gradient method. The estimator we bound is thus the output of an algorithm run for finitely many steps. Our statement is made in the R\'enyi risk $\Dbar$, on the whole parameter space, with an unrestricted Gaussian prior and an unrestricted Gaussian variational family. We use the risk, the prior and the family of~\citet[Section 3]{alquier2019concentration}, so that our result can be compared with their analysis of the same model.
\vspace{0.15cm}

The observations are the pairs $X_i=(Y_i,Z_i)$, $i=1,\dots,n$, independent, with $Y_i\in\{-1,+1\}$ and $Z_i\in\mathbb{R}^d$, and the conditional law of the response is the binary logistic model in the $\pm1$ sign convention,
\begin{align}
\label{eq:logistic:model}
    \mathbb{P}_{\theta}(Y_i=y\mid Z_i=z)=\sigma\bigl(y z^{\top}\theta\bigr),
    \qquad
    \sigma(u)=\frac{e^{u}}{1+e^{u}},
    \qquad
    \theta\in\Theta=\mathbb{R}^d .
\end{align}
The example is well-specified: $P_i^{\circ}=P_{i,\theta_0}$ for a $\theta_0\in\mathbb{R}^d$. The oracle terms of Theorems~\ref{thm:thm1} and~\ref{thm:thm1-p} therefore vanish. Here $\KLbar$ and $\Dbar$ of~\eqref{eq:averages} are the per-observation averages against $P_{i,\theta_0}$, and we write $r_n(\theta,\theta_0)$ for the log-likelihood ratio. The joint density $p_{i,\theta}(y,z)=\sigma(yz^{\top}\theta)q_i(z)$ has a covariate marginal that does not depend on $\theta$. This marginal cancels in every likelihood ratio, and it need not be the same for all $i$.
\vspace{0.15cm}

Every statement of this subsection is made conditionally on the design. We fix $z_1,\dots,z_n$, we write $\mathbb{E}_{\mathcal{S}}$ for the expectation over the responses $Y_1,\dots,Y_n$ at that design, and we assume
\begin{align}
\label{eq:logistic:gram}
    G_n=\frac{1}{n}\sum_{i=1}^n z_i z_i^{\top}\succ0,
    \qquad
    \|z_i\|\le B\ \ (i\le n),
    \qquad
    \|\theta_0\|\le R_0 .
\end{align}
Note that $G_n\succ0$ requires $n\ge d$. Here $G_n$ is the \emph{empirical} Gram matrix. The random-design case, where the expected Gram matrix appears, is discussed in Subsection~\ref{subsec:logistic-discussion}. Conditionally on the design, $P_{i,\theta}$ is the law of $Y_i$ alone. It is the Bernoulli law of Example~\ref{exm:bernoulli} at the natural parameter $u_i=z_i^{\top}\theta$, with $\psi(u)=\log(1+e^{u})$, $\psi'=\sigma$ and $\psi''=\sigma(1-\sigma)$. It follows that $\KLbar$ is the average over $i$ of the Bregman divergence of $\psi$ at $z_i^{\top}\theta_0$. This average is written out at~\eqref{eq:pflog:bregman}.
\vspace{0.15cm}

The temperature $\alpha\in(0,1)$ and the prior scale $\tau>0$ are fixed throughout. The prior is the centered Gaussian, and the variational family is the full Gaussian family,
\begin{align*}
    \pi=\mathcal{N}(0,\tau^2I_d),
    \qquad
    \mathcal{F}=\bigl\{\mathcal{N}(m,\Sigma):m\in\mathbb{R}^d,\ \Sigma\succ0\bigr\}.
\end{align*} Contrary to the Gaussian models of Subsection~\ref{subsec:gaussian}, $\psi$ is not a quadratic. Thus, $\KLbar$ is not a quadratic form, the tilted measure $e^{-\beta\Dbar(\theta)}\pi({\rm d}\theta)$ is not Gaussian, and it lies outside $\mathcal{F}$ for every $\beta>0$.  The proof uses the Gaussian centered at the truth whose covariance is aligned with the Gram matrix,
\begin{align}
\label{eq:logistic:Sigma}
    \rho_{\star}=\mathcal{N}\bigl(\theta_0,\Sigma_n^{\star}\bigr),
    \qquad
    \Sigma_n^{\star}=\Bigl[\tau^{-2}I_d+\frac{\alpha n}{4}G_n\Bigr]^{-1}.
\end{align}
\vspace{0.15cm}

Our result only depends on the estimator through an expected objective gap, defined as follows. Let us write
\begin{align*}
    J_{\mathcal{S}}(\rho)=\alpha\,\mathbb{E}_{\theta\sim\rho}\bigl[r_n(\theta,\theta_0)\bigr]+KL(\rho\|\pi)
\end{align*}
for the unnormalized variational objective of Section~\ref{section:prelim}. Note that $J_{\mathcal{S}}$ is $n(1-\alpha)$ times the objective minimized by $\tilde{\rho}_{n,\alpha}$, and has the same minimizer. Let $\hat\rho$ be an element of $\mathcal{F}$, possibly randomized, data-dependent, and with a randomization independent of the responses, and let $\rho_0\in\mathcal{F}$ be deterministic with
\begin{align}
\label{eq:logistic:Delta}
    \mathbb{E}\bigl[J_{\mathcal{S}}(\hat\rho)\bigr]\ \le\ \mathbb{E}\bigl[J_{\mathcal{S}}(\rho_0)\bigr]+\Delta,
\end{align}
where the expectation is over the responses at the fixed design and over the randomization of the method. Condition~\eqref{eq:logistic:Delta} holds at every $\rho_0\in\mathcal{F}$ as soon as
\begin{align*}
    \mathbb{E}\bigl[J_{\mathcal{S}}(\hat\rho)\bigr]\le\inf_{\rho\in\mathcal{F}}\mathbb{E}\bigl[J_{\mathcal{S}}(\rho)\bigr]+\Delta,
\end{align*}
and in particular at $\Delta=0$ for the exact minimizer $\tilde{\rho}_{n,\alpha}$. For an algorithmic output, $\Delta$ is its expected objective gap, and the corollary below gives such a $\Delta$ at the competitor $\rho_0=\rho_{\star}$ of~\eqref{eq:logistic:Sigma}.
\vspace{0.15cm}

Two constants of the design remain, a local curvature and a separation. For $r>0$, put
\begin{align}
\label{eq:logistic:constants}
    \eta_r&=\sigma\bigl(B(R_0+r)\bigr)\bigl(1-\sigma\bigl(B(R_0+r)\bigr)\bigr)\in(0,1/4],
    \nonumber\\
    m_{r,\alpha}&=\inf_{\|w-\theta_0\|=r}\Dbar(w)\ \ge\ \frac{\alpha\eta_r\lambda_{\min}(G_n)r^2}{2}\ >\ 0  .
\end{align}
Here $\eta_r$ is the curvature floor of the log-partition function over the range of the natural parameter on the ball of radius $r$ around $\theta_0$. The lower bound on $m_{r,\alpha}$ comes from the local quadratic lower bound of Lemma~\ref{lem:logistic:quad}. That lemma also shows that $\Dbar$ is nondecreasing along every ray from $\theta_0$, so that $m_{r,\alpha}$ is a lower bound on $\Dbar$ on the whole region outside the ball.  The choice $r=R_0$ gives $\eta_r=\sigma(2BR_0)(1-\sigma(2BR_0))$. The exterior of the ball contributes the remainder
\begin{align}
\label{eq:logistic:Rnr}
    \mathcal{R}_{n,r}=e^{-n(1-\alpha)m_{r,\alpha}/2}\,\det\bigl(I_d+b_n G_n\bigr)^{1/2},
    \qquad
    b_n:=\frac{\alpha\,\eta_r\,(1-\alpha)\,n\,\tau^2}{2} .
\end{align}

\begin{thm}
\label{thm:logistic}
Let $\alpha\in(0,1)$, let the design satisfy~\eqref{eq:logistic:gram}, let $\pi=\mathcal{N}(0,\tau^2I_d)$ on $\mathbb{R}^d$ with no restriction, and let $\hat\rho$ satisfy~\eqref{eq:logistic:Delta} at $\rho_0=\rho_{\star}$ of~\eqref{eq:logistic:Sigma}. Then, for every $r>0$, with $\eta_r$, $m_{r,\alpha}$ of~\eqref{eq:logistic:constants} and $\mathcal{R}_{n,r}$ of~\eqref{eq:logistic:Rnr},
\begin{align}
\label{eq:logistic:rate}
    \mathbb{E}\,\mathbb{E}_{\theta\sim\hat\rho}\bigl[\Dbar(\theta)\bigr]
    &\ \le\ C_{\alpha,r}\,\frac dn
    +\frac{R_0^2}{n(1-\alpha)\tau^2}
    +\frac{2\mathcal{R}_{n,r}}{n(1-\alpha)}
    +\frac{2\Delta}{n(1-\alpha)},
    \nonumber\\
    C_{\alpha,r}&=\frac{1}{1-\alpha}\log\frac{1}{2\eta_r(1-\alpha)} .
\end{align}
Moreover, $\mathcal{R}_{n,r}\le1/n$ if and only if
\begin{align}
\label{eq:logistic:tail}
    \frac12\log\det\bigl(I_d+b_nG_n\bigr)+\log n\ \le\ \frac{n(1-\alpha)\,m_{r,\alpha}}{2}.
\end{align}
\end{thm}

The constant $C_{\alpha,r}$ does not depend on $n$, on $\tau$ or on the spectrum of $G_n$; it is at least $\log2/(1-\alpha)$ since $\eta_r\le1/4$, and of order one at a fixed $\alpha$ and a curvature bounded away from zero. The first two terms of~\eqref{eq:logistic:rate} are then of order $d/n$ as soon as $R_0^2/\tau^2=\mathcal{O}(d)$, the third is at most $2/[n^2(1-\alpha)]$ under~\eqref{eq:logistic:tail}, and the fourth is the optimization error. Condition~\eqref{eq:logistic:tail} bounds the exterior remainder and does not affect the rate. Along a sequence of designs with $\lambda_{\max}(G_n)=\mathcal{O}(1)$ and $\liminf_n\lambda_{\min}(G_n)>0$, its left-hand side is of order $d\log n$ against a right-hand side of order $n$, so it holds for all large $n$ whenever $d\log n=o(n)$.
\vspace{0.15cm}

The proof relies on four facts. First, the curvature of the log-partition function is bounded above on the whole line, which allows us to evaluate the Gaussian competitor. Second, it is bounded below on the ball of radius $r$, which controls the localized prior on this ball. Third, since $G_n\succ0$, $\Dbar$ is bounded away from zero outside the ball. Finally, the covariance of~\eqref{eq:logistic:Sigma} is aligned with $G_n$, so that the two log-determinants cancel eigenvalue by eigenvalue.
\vspace{0.15cm}

This last cancellation is the reason for the absence of a logarithm in~\eqref{eq:logistic:rate}. The covariance of~\eqref{eq:logistic:Sigma} minimizes a bound on $\alpha n\mathbb{E}_{\theta\sim\rho}[\KLbar(\theta)]+KL(\rho\|\pi)$ over the Gaussian competitors $\rho$ centered at $\theta_0$. This bound, and the logarithm of the partition function of the localized prior at the localization $\beta=n(1-\alpha)/2$ chosen in the proof, are at most, respectively,
\begin{align*}
    \frac12\log\det\Bigl(I_d+\frac{\alpha n\tau^2}{4}G_n\Bigr)+\frac{R_0^2}{2\tau^2}
    \qquad\text{and}\qquad
    -\frac12\log\det\bigl(I_d+b_nG_n\bigr)+\mathcal{R}_{n,r}.
\end{align*}
Both are of order $d\log n$ when the spectrum of $G_n$ is bounded. In their sum, the two log-determinants give half the logarithm of a determinant ratio, bounded through $(1+au)/(1+bu)\le\max(1,a/b)$ for $u\ge0$, hence by
\begin{align*}
    \frac d2\log\frac{1}{2\eta_r(1-\alpha)},
\end{align*}
which is of order $d$. Division by the coefficient $n(1-\alpha)/2$ gives $d/n$. With a fixed prior, only the first log-determinant appears.
\vspace{0.15cm}

The variational approximation is computed by a projected stochastic gradient method, as in~\citet[Section 3.1]{alquier2019concentration}. The method is run on the reparametrization $\theta=m+C\xi$, $\xi\sim\mathcal{N}(0,I_d)$, $\Sigma=CC^{\top}$, over the search set
\begin{align}
\label{eq:logistic:ball}
    \mathbb{B}&=\bigl\{(m,C):\ C=C^{\top}\succ0,\ \|(m,C)\|_2\le R_{\mathrm{sgd}},\ C\succeq\sqrt{\varsigma}\,I_d\bigr\},
    \nonumber\\
    \|(m,C)\|_2^2&=\|m\|^2+\|C\|_F^2.
\end{align}
Here $\varsigma>0$ is a covariance floor, the constraint $C\succeq\sqrt{\varsigma}I_d$ giving $\Sigma=CC^{\top}\succeq\varsigma I_d$. The set~\eqref{eq:logistic:ball} is the search set of that reference, symmetrized in $C$ and with the floor written on $C$ rather than on $CC^{\top}$. This makes it convex. The gradient of the Kullback-Leibler term of the objective is computed in closed form, and only the likelihood term is estimated, from one Gaussian draw. Let us write
\begin{align*}
    u_{\mathcal{S}}(\theta)=-\alpha\sum_{i=1}^n\log\sigma\bigl(Y_iz_i^{\top}\theta\bigr),
    \qquad
    v=\nabla u_{\mathcal{S}}(m+C\xi),
    \qquad
    {\rm sym}(A)=\frac{A+A^{\top}}{2}.
\end{align*}
At a point $(m,C)$ of $\mathbb{B}$, the stochastic gradient is
\begin{align}
\label{eq:logistic:grad}
    g_m=v+\tau^{-2}m,
    \qquad
    g_C={\rm sym}\bigl(v\xi^{\top}\bigr)+\tau^{-2}C-C^{-1}.
\end{align}
With $B$ the bound on the covariates in~\eqref{eq:logistic:gram}, we put
\begin{align}
\label{eq:logistic:sgdconst}
    G_{\mathrm{sgd}}^2=2\alpha^2n^2B^2(1+d)+\frac{4R_{\mathrm{sgd}}^2}{\tau^4}+\frac{4d}{\varsigma},
    \qquad
    \eta_T=\frac{2R_{\mathrm{sgd}}}{G_{\mathrm{sgd}}\sqrt T},
    \qquad
    \Delta_T=\frac{2R_{\mathrm{sgd}}G_{\mathrm{sgd}}}{\sqrt T}.
\end{align}
Starting from a point $x_0\in\mathbb{B}$, the method computes, for $t=0,\dots,T-1$,
\begin{align*}
    x_{t+1}=\Pi_{\mathbb{B}}\bigl(x_t-\eta_Tg_t\bigr),
    \qquad
    \bar x_T=(\bar m_T,\bar C_T)=\frac1T\sum_{t=0}^{T-1}x_t,
    \qquad
    \rho_T=\mathcal{N}\bigl(\bar m_T,\bar C_T\bar C_T^{\top}\bigr).
\end{align*}
Here $\Pi_{\mathbb{B}}$ is the Euclidean projection onto $\mathbb{B}$, and $g_t$ is~\eqref{eq:logistic:grad} at $x_t$, computed from a draw $\xi_t$ independent of $\xi_0,\dots,\xi_{t-1}$ and of the responses.  We show in Section~\ref{sec:proof} that $G_{\mathrm{sgd}}^2$ bounds the second moment of~\eqref{eq:logistic:grad} on $\mathbb{B}$, and that $\Delta_T$ is a $\Delta$ in the sense of~\eqref{eq:logistic:Delta}.

\begin{cor}
\label{cor:logistic-sgd}
Under the hypotheses of Theorem~\ref{thm:logistic} and~\eqref{eq:logistic:tail}, assume that the competitor~\eqref{eq:logistic:Sigma} is feasible for the method,
\begin{align}
\label{eq:logistic:feasible}
    \varsigma\ \le\ \lambda_{\min}\bigl(\Sigma_n^{\star}\bigr)=\Bigl[\tau^{-2}+\frac{\alpha n}{4}\lambda_{\max}(G_n)\Bigr]^{-1},
    \qquad
    R_0^2+{\rm Tr}\bigl(\Sigma_n^{\star}\bigr)\ \le\ R_{\mathrm{sgd}}^2 .
\end{align}
Then the Gaussian $\rho_T$ defined by the average of the pre-update iterates $x_0,\dots,x_{T-1}$ satisfies, with $\Delta_T$ of~\eqref{eq:logistic:sgdconst},
\begin{align}
\label{eq:logistic:sgd}
    \mathbb{E}\,\mathbb{E}_{\theta\sim\rho_T}\bigl[\Dbar(\theta)\bigr]
    \ \le\ C_{\alpha,r}\,\frac dn+\frac{R_0^2}{n(1-\alpha)\tau^2}+\frac{2}{n^2(1-\alpha)}
    +\frac{2\Delta_T}{n(1-\alpha)} .
\end{align}
\end{cor}

Conditions~\eqref{eq:logistic:feasible} are the covariance floor and the radius of the search set, evaluated at the competitor. Since
\begin{align*}
    {\rm Tr}\bigl(\Sigma_n^{\star}\bigr)\le\frac{4d}{\alpha n\lambda_{\min}(G_n)},
\end{align*}
both are met for all large $n$ at the floor $\varsigma=1/(n\sqrt d)$ of~\citet[Section 3.2]{alquier2019concentration}, along a sequence of designs with $\limsup_n\alpha\lambda_{\max}(G_n)<4\sqrt d$ and with $R_{\mathrm{sgd}}>R_0$. The last term of~\eqref{eq:logistic:sgd} is the only algorithmic one. It is at most $C_{\alpha,r}d/n$ as soon as
\begin{align*}
    T\ \ge\ \Bigl(\frac{4R_{\mathrm{sgd}}G_{\mathrm{sgd}}}{(1-\alpha)\,C_{\alpha,r}\,d}\Bigr)^2 .
\end{align*}
At the floor $\varsigma=1/(n\sqrt d)$, the last term of $G_{\mathrm{sgd}}^2$ in~\eqref{eq:logistic:sgdconst} is $4nd^{3/2}$, and thus this number of iterations is polynomial in $n$ and $d$ when the other constants are fixed.
\vspace{0.15cm}

We compare this result with~\citet{alquier2019concentration}, and discuss the random design and the Kullback-Leibler risk, in Subsection~\ref{subsec:logistic-discussion}.

\subsection{Stochastic Maximization of the Likelihood}
\label{subsec:mle}

We now apply Theorem~\ref{thm:main-thm} to an approximation of the MLE obtained by stochastic optimization, rather than
 to a posterior or a variational approximation. Information
 bounds for noisy iterative methods were studied by
 \citet{pensia2018generalization}; see also \citet{pmlr-v134-neu21a} for perturbed SGD.
In principle, our mutual information bound could be applied to any stochastic optimization algorithm. For the sake of simplicity, we only analyze
a projected gradient
 descent with independent Gaussian noise added to each gradient, in the case of an exponential family.

Let $X_1,\ldots,X_n$ be independent with common distribution $P_{\theta_0}$,
 where $\theta_0\in\Theta$ and $\Theta\subset\mathbb R^d$ is compact and convex.
 Consider a regular exponential family
\[
 p_\theta(x)=\exp\{h(x)+\langle s(x),\theta\rangle-\psi(\theta)\},
 \qquad \theta\in\Theta,
\]
with common support $\mathbf{X}$ and with $\Theta$ contained in the open natural parameter
 domain. Assume that $\psi$ is $m$-strongly convex on $\Theta$, for some $m>0$,
 and that, for some $G>0$, for any $x\in\mathbf{X}$ and $\theta\in\Theta$
\begin{equation}
 \|\nabla\psi(\theta)-s(x)\|\leq G.
 \label{eq:noisy-score}
\end{equation}
Thus each negative log-likelihood is convex and $G$-Lipschitz. Write
\begin{align*}
    F_{\mathcal{S}}(\theta)=-\frac{1}{n}\sum_{i=1}^n\log p_\theta(X_i),
    \qquad
    \bar s_n=\frac1n\sum_{i=1}^n s(X_i),
\end{align*}
so that
 $\nabla F_{\mathcal{S}}(\theta)=\nabla\psi(\theta)-\bar s_n$.
 Starting from a fixed $\theta_1\in\Theta$, take independent
 $Z_t\sim\mathcal N(0,I_d)$, independent also of the sample, and set
\begin{equation}
 \theta_{t+1}=\Pi_\Theta\!\left[\theta_t-
 \frac{2}{m(t+1)}\{\nabla F_{\mathcal{S}}(\theta_t)+\nu Z_t\}\right],
 \qquad t=1,\ldots,T,
 \label{eq:noisy-update}
\end{equation}
where $\Pi_\Theta$ is the Euclidean projection and $\nu>0$.
 The output is the weighted average of the pre-update iterates,
\begin{align*}
    \bar\theta_T=\frac{2}{T(T+1)}\sum_{t=1}^T t\,\theta_t .
\end{align*}
 This is a noisy full-gradient method. No subsampling is used, and $\bar s_n$
 is computed only once.

\begin{cor}
\label{cor:MLE:stoch}
Under these assumptions, for every $\alpha\in(0,1)$,
\begin{equation}
 \mathbb{E}\bigl[\mathcal{D}_{\alpha}(P_{\bar\theta_T}\|P_{\theta_0})\bigr]
 \leq \frac{1}{1-\alpha}\left\{
 \frac{2\alpha(G^2+d\nu^2)}{m(T+1)}
 +\frac{TG^2}{2n^2\nu^2}\right\},
 \label{eq:noisy-general}
\end{equation}
where the expectation includes the sample and the algorithmic randomness.
 In particular, choosing $\nu^2=TG^2/(dn)$ gives
\begin{equation}
 \mathbb{E}\bigl[\mathcal{D}_{\alpha}(P_{\bar\theta_T}\|P_{\theta_0})\bigr]
 \leq \frac{1}{1-\alpha}\left\{
 \frac{2\alpha G^2}{m(T+1)}
 +\frac{2\alpha G^2/m+d/2}{n}\right\}.
 \label{eq:noisy-tuned}
\end{equation}
For $T\geq n$, the right-hand side is at most
 \begin{align*}
    \frac{4\alpha G^2/m+d/2}{(1-\alpha)n}.
\end{align*}
 Moreover, $\mathbb{E}[\|\bar\theta_T-\theta_0\|^2]$ is bounded by
 $2/(\alpha m)$ times the right-hand side of \eqref{eq:noisy-general}.
\end{cor}

At a fixed $\alpha$, this is of order $d/n$ when $G^2/m=\mathcal{O}(d)$.
 For example, in the product Bernoulli model on $\Theta=[-A,A]^d$,
 we may take $G^2=d$ and $m=e^A/(1+e^A)^2$. The bound is therefore of
 parametric order for fixed $A$. It concerns the randomized output of the method rather than the exact MLE.
\section{PAC-Bayes, Mutual Information and Proofs}

\label{sec:proof}

In this section, we provide the proof of all the theorems. We try to organize the proofs by introducing the theory step by step. We first give a brief overview of the PAC-Bayes bounds of~\cite{bhattacharya2016bayesian,alquier2019concentration}, and we provide variants in expectation, in Subsection~\ref{subsec:pacbayes}.
We then use these results to prove Theorem~\ref{thm:main-thm}, in Subsection~\ref{subsec:proof:mi}.
Next, we prove all the results of Section~\ref{section:main}, including Theorem~\ref{thm:thm1}, in Subsection~\ref{subsec:proof:main}. All these results could be obtained as consequences of Theorem~\ref{thm:main-thm}. We rather derive them from a localized version of the PAC-Bayes bound, following a technique developed by~\cite{Catoni2007}.
\vspace{0.15cm}

Finally, in Subsection~\ref{subsec:proof:other}, we gather all the proofs of the results in Section~\ref{sec:examples}.

\subsection{PAC-Bayes Bounds}
\label{subsec:pacbayes}

We first state a PAC-Bayes bound in expectation for any generalized posterior. In this form, the bound is due to~\cite{alquier2019concentration}, who obtained it as a variant of a result of~\cite{bhattacharya2016bayesian} using the same proof technique. Their Theorem 2.6 is stated for the variational approximation over a family $\mathcal{F}$, under identically distributed observations. The version below, for an arbitrary data-dependent $\hat\rho$ and independent observations that need not be identically distributed, is the one their proof establishes. Similar results can be found in~\cite{zhang2006klentropy}. We recall the proof for completeness.
\begin{thm}[Variant of Theorem $2.6$ in~\citet{alquier2019concentration}]
\label{thm:thm2}
 Given a data-dependent posterior $\hat{\rho}$, a prior $\pi$, and independent observations $X_i$ drawn from the data-generating distributions $P_i^{\circ}$ with densities $p_i^{\circ}$ with respect to $\mu_i$, we have the following bound, for any $\alpha\in(0,1)$:
\begin{align*}
 \mathbb{E}_{\mathcal{S}}\mathbb{E}_{\theta\sim\hat{\rho}}[\Dbar(\theta)]\le \mathbb{E}_{\mathcal{S}}\left[\frac{\alpha}{n(1-\alpha)}\mathbb{E}_{\theta\sim\hat{\rho}}[r_n(\theta)] \right]+ \frac{1}{n(1-\alpha)}\mathbb{E}_{\mathcal{S}}[KL(\hat{\rho}\|\pi)].
 \end{align*}
\end{thm}
\begin{proof}
 Fix $\alpha\in(0,1)$ and $\theta\in\Theta$. Since $P_i^{\circ}\ll\mu_i$ for every $i$, each R\'enyi divergence $\mathcal{D}_{\alpha}(P_{i,\theta}\|P_i^{\circ})$ to the truth is well defined. Their average $\Dbar(\theta)$ is then well defined as well. Note that the $X_i$ are independent, and thus the expectation of the product below factorizes:
\begin{align*}
 \mathbb{E}_{\mathcal{S}}\left[\exp\left(-\alpha r_n(\theta)\right)\right]
 &=\mathbb{E}_{\mathcal{S}}\left[\prod_{i=1}^n\left(\frac{p_{i,\theta}(X_i)}{p_i^{\circ}(X_i)}\right)^{\alpha}\right]
 =\prod_{i=1}^n\mathbb{E}_{X_i}\left[\left(\frac{p_{i,\theta}(X_i)}{p_i^{\circ}(X_i)}\right)^{\alpha}\right] \\
 &=\prod_{i=1}^n\exp\left(-(1-\alpha)\mathcal{D}_{\alpha}(P_{i,\theta}\|P_i^{\circ})\right)
 =\exp\left(-n(1-\alpha)\Dbar(\theta)\right).
\end{align*}
Here the second equality is the only place where the joint law of the sample is used. It requires independence alone, and the observations need not be identically distributed.
This can be rewritten as follows:
\begin{align*}
\mathbb{E}_{\mathcal{S}}\left[\exp\left(-\alpha r_n(\theta)+n(1-\alpha)\Dbar(\theta)\right)\right]=1.
\end{align*}
Integrating with respect to $\pi$ and using Fubini's theorem,
 \begin{align*}
\mathbb{E}_{\mathcal{S}}\left[\mathbb{E}_{\theta\sim\pi}\left[\exp\left(-\alpha r_n(\theta)+n(1-\alpha)\Dbar(\theta)\right)\right]\right]=1.
\end{align*}
The main step is the following lemma.
\begin{lem}[\citet{donsker1976asymptotic} variational formula]
For any probability $\pi$ on $\Theta$ and any measurable function $h$ such that $\int e^h \pi <\infty $:
 \begin{align*}
        \log\left(\displaystyle\int e^h \pi \right) =\sup_{\rho \in \mathcal{M}_1^+}\left[\int h\rho - KL(\rho\|\pi)\right],
    \end{align*}
    with the convention $\infty-\infty=\infty$. Moreover, when $h$ is bounded from above on the support of $\pi$, the supremum is attained at:
    \begin{align*}
        \pi_{h}({\rm d}\theta)=\frac{e^{h(\theta)}\pi({\rm d}\theta)}{\int e^{h(\theta)}\pi({\rm d}\theta)}.
    \end{align*}
    \end{lem}

    We use the above lemma for the function $h(\theta)=-\alpha r_n(\theta)+n(1-\alpha)\Dbar(\theta)$. The identity above then reads
    \begin{align*}
        \mathbb{E}_{\mathcal{S}}\Bigl[\exp\Bigl(\sup_{\rho\in \mathcal{M}_1^+}\Bigl[\int h\rho - KL(\rho\|\pi)\Bigr]\Bigr)\Bigr]=1,
    \end{align*}
    and Jensen's inequality gives that:
    \begin{align*}
        \mathbb{E}_{\mathcal{S}}\left[\sup_{\rho\in \mathcal{M}_1^+}\left[\int h\rho - KL(\rho\|\pi)\right]\right]\le 0.
    \end{align*}

    In particular, the same inequality holds at $\rho=\hat{\rho}$ for any data-dependent posterior $\hat{\rho}$. Using the definition of $h$, we have that:
 \begin{align*}
\mathbb{E}_{\mathcal{S}}\Bigl[\int \bigl(-\alpha r_n(\theta)+n(1-\alpha)\Dbar(\theta)\bigr)\hat{\rho}({\rm d}\theta) - KL(\hat{\rho}\|\pi)\Bigr]\le 0.
\end{align*}
 Rearranging the terms, we have that:
 \begin{align*}
n(1-\alpha) \mathbb{E}_{\mathcal{S}}\mathbb{E}_{\theta\sim\hat{\rho}}\left[\Dbar(\theta)\right]\le \alpha \mathbb{E}_{\mathcal{S}}\mathbb{E}_{\theta\sim\hat{\rho}}\left[r_n(\theta)\right]+ \mathbb{E}_{\mathcal{S}}\left[KL(\hat{\rho}\|\pi)\right].
\end{align*}
Dividing by $n(1-\alpha)$, we get the desired result.
\end{proof}
    Let us now consider the special case $\hat{\rho}=\pi_{n,\alpha}$, and let us take the infimum over a family $\mathcal{F}\subset\mathcal{M}_1^+$ on the right-hand side. This gives the following corollary, which is the in-expectation bound of~\citet[Theorem 2.6]{alquier2019concentration}.

\begin{cor}
\label{cor:cor1}
    For any $\mathcal{F}\subset\mathcal{M}_1^+$ and $\alpha\in(0,1)$, we have the following bound:
    \begin{align*}
        \mathbb{E}_{\mathcal{S}}\mathbb{E}_{\theta \sim  \pi_{n,\alpha}}\left[\Dbar(\theta)\right]\le \inf_{\rho\in \mathcal{F}}\left[\frac{\alpha}{1-\alpha}\mathbb{E}_{\theta \sim \rho}\left[ \KLbar(\theta)\right] +\frac{KL(\rho\|\pi)}{n(1-\alpha)}\right].
    \end{align*}
    \end{cor}
At the parametric endpoint $\kappa_\pi=1$, and under a prior-mass condition that bounds $KL(\rho\|\pi)$ at a competitor $\rho$ concentrated at scale $1/n$, the bound of Corollary~\ref{cor:cor1} is of order $d_{\pi}\log(n)/n$, as shown in~\citet{alquier2019concentration}. Assumption~\ref{asm:asm2} alone does not give this, since it constrains $\mathbb{E}_{\theta\sim\pi_{-\beta}}[\KLbar(\theta)]$ rather than the complexity of such a competitor. The bound holds for any prior $\pi$, and can therefore be optimized over $\pi$. We now carry out this optimization using mutual information, and this leads to Theorem~\ref{thm:main-thm}.

We also state the variant in probability.
\begin{thm}[Theorem 2.1 and Corollary 2.3 in~\cite{alquier2019concentration}]
\label{thm:thm2-p}
 Given a data-dependent posterior $\hat{\rho}$ and independent observations $X_i$ drawn from the data-generating distributions $P_i^{\circ}$ with densities $p_i^{\circ}$ with respect to $\mu_i$, we have, for any $\alpha\in(0,1)$, with probability at least $1-\delta$ over the sample $\mathcal{S}$,
\begin{align*}
\mathbb{E}_{\theta\sim\hat{\rho}}[\Dbar(\theta)]\le \frac{\alpha}{n(1-\alpha)}\mathbb{E}_{\theta\sim\hat{\rho}}[r_n(\theta)] + \frac{KL(\hat{\rho}\|\pi)+\log\frac{1}{\delta}}{n(1-\alpha)}.
 \end{align*}
 In particular,
 \begin{align*}
\mathbb{E}_{\theta\sim\tilde{\rho}_{n,\alpha}}[\Dbar(\theta)]\le
\inf_{\rho\in\mathcal{F}}\left[
\frac{\alpha}{n(1-\alpha)}\mathbb{E}_{\theta\sim\rho}[r_n(\theta)] + \frac{KL(\rho\|\pi)+\log\frac{1}{\delta}}{n(1-\alpha)}
\right].
 \end{align*}
\end{thm}

\subsection{Mutual Information and Proof of Theorem~\ref{thm:main-thm}}
\label{subsec:proof:mi}
We are now ready to prove Theorem~\ref{thm:main-thm}, based on ideas from pages 14 and 51 in~\cite{Catoni2007}.

\begin{proof}[Proof of Theorem~\ref{thm:main-thm}]
By Theorem~\ref{thm:thm2}, we have that:
\begin{align*}
    \mathbb{E}_{\mathcal{S}}\mathbb{E}_{\theta\sim\hat{\rho}}[\Dbar(\theta)]-\frac{\alpha}{n(1-\alpha)}\mathbb{E}_{\mathcal{S}}\mathbb{E}_{\theta\sim\hat{\rho}}[r_n(\theta)] \le \frac{1}{n(1-\alpha)}\mathbb{E}_{\mathcal{S}}[KL(\hat{\rho}\|\pi)].
\end{align*}
Let $\hat{\rho}$ be any data-dependent posterior that is absolutely continuous with respect to $\pi$. \cite{Catoni2007} defines $\mathbb{E}_{\mathcal{S}}[\hat{\rho}]$ as the probability measure given by:
\begin{align*}
    \frac{{\rm d}\mathbb{E}_{\mathcal{S}}[\hat{\rho}]}{{\rm d}\pi}(\theta)=\mathbb{E}_{\mathcal{S}}\left[\frac{d\hat{\rho}}{d\pi}(\theta)\right].
\end{align*}
Direct calculations show that:
\begin{align*}
    \mathbb{E}_{\mathcal{S}}[KL(\hat{\rho}\|\pi)]=\mathbb{E}_{\mathcal{S}}\left[KL(\hat{\rho}\|\mathbb{E}_{\mathcal{S}}\left[\hat{\rho}\right])\right]+KL(\mathbb{E}_{\mathcal{S}}[\hat{\rho}]\|\pi)=\mathcal{I}(\theta,\mathcal{S})+KL(\mathbb{E}_{\mathcal{S}}[\hat{\rho}]\|\pi).
\end{align*}
Using the above equality, we have that:
\begin{align*}
    \mathbb{E}_{\mathcal{S}}\mathbb{E}_{\theta\sim\hat{\rho}}[\Dbar(\theta)]-\frac{\alpha}{n(1-\alpha)}\mathbb{E}_{\mathcal{S}}\mathbb{E}_{\theta\sim\hat{\rho}}[r_n(\theta)] \le \frac{1}{n(1-\alpha)}\left(\mathcal{I}(\theta,\mathcal{S})+KL(\mathbb{E}_{\mathcal{S}}[\hat{\rho}]\|\pi)\right).
\end{align*}
Note that the last term is nonnegative. Choosing $\pi=\mathbb{E}_{\mathcal{S}}[\hat{\rho}]$ gives $KL(\mathbb{E}_{\mathcal{S}}[\hat{\rho}]\|\pi)=0$, which is its smallest possible value. This gives the bound of Theorem~\ref{thm:main-thm}.
\end{proof}
In other words, MI bounds can be viewed as PAC-Bayes bounds that have been optimized with respect to the choice of prior. This reflects the well-established relationship between relative entropy and mutual information (see, e.g., \citet[Proposition II.3]{melbourne2022differential}).
\vspace{0.15cm}

Note that the mutual information, and thus the above bound, cannot be computed in practice. However, one can settle for a suboptimal prior $\pi$, said to be localized~\citep{Catoni2007,alquier2024user}. We select the localized prior $\pi_{-\beta}$ of Assumption~\ref{asm:asm2}, with $\beta$ of order $n$, as an approximation of $\mathbb{E}_{\mathcal{S}}[\pi_{n,\alpha}]$.

\subsection{Localized PAC-Bayes Bounds and Proofs of Theorems~\ref{thm:thm1} and~\ref{thm:thm1-p}}
\label{subsec:proof:main}

In our context, the localization approach of Catoni consists in using the localized prior, and then in upper bounding $KL(\hat{\rho}\|\pi_{-\beta})$ via empirical bounds. The following theorem gives the resulting bound under Assumption~\ref{asm:asm1}.
\begin{thm}
\label{thm:thm4}
    Let $\alpha\in(0,1)$ and $\mathcal{F}\subset\mathcal{M}_1^+$, let $\tilde{\rho}_{n,\alpha}$ be the variational approximation of Section~\ref{section:prelim}, and let $\pi_{-\beta}$ be the localized prior of Assumption~\ref{asm:asm2}.
    If Assumption~\ref{asm:asm1} is satisfied, then we have the following bound, for any $\beta<n(1-\alpha)/c(\alpha) $:
    \begin{equation*}
      \mathbb{E}_{\mathcal{S}}\mathbb{E}_{\theta\sim\tilde{\rho}_{n,\alpha}}[\KLbar(\theta)]\le
      \frac{c(\alpha) n}{n(1-\alpha) - \beta c(\alpha)}
      \inf_{\rho\in \mathcal{F}}\left[\left(\alpha -\frac{\beta}{n}\right)\mathbb{E}_{\theta\sim\rho}\left[ \KLbar(\theta)\right] +\frac{KL(\rho\|\pi_{-\beta})}{n}\right].
    \end{equation*}
    \end{thm}
    \begin{proof}
    We first apply Theorem~\ref{thm:thm2} at the choice $\pi=\pi_{-\beta}$. This gives
    \begin{align*}
        \mathbb{E}_{\mathcal{S}}\mathbb{E}_{\theta\sim\hat{\rho}}[\Dbar(\theta)]
        \le \frac{\alpha}{n(1-\alpha)}\mathbb{E}_{\mathcal{S}}\mathbb{E}_{\theta\sim\hat{\rho}}[r_n(\theta)]
        + \frac{\mathbb{E}_{\mathcal{S}}[KL(\hat{\rho}\|\pi_{-\beta})]}{n(1-\alpha)} .
    \end{align*}
    We now expand the last term. Splitting the logarithm at $\pi$ and using the definition of $\pi_{-\beta}$, we have the decomposition
    \begin{align}
        KL(\hat{\rho}\|\pi_{-\beta})&= \mathbb{E}_{\theta\sim\hat{\rho}}\left[\log\frac{\hat{\rho}(\theta)}{\pi(\theta)}\right]+\mathbb{E}_{\theta\sim\hat{\rho}}\left[\log\frac{\pi(\theta)}{\pi_{-\beta}(\theta)}\right]\nonumber\\
        &= KL(\hat{\rho}\|\pi)+\mathbb{E}_{\theta\sim\hat{\rho}}\left[\log\left(\frac{\mathbb{E}_{\pi}[e^{-\beta \KLbar(\theta) }]}{e^{-\beta \KLbar(\theta) }}\right)\right]\nonumber\\
        &=KL(\hat{\rho}\|\pi)+\beta\mathbb{E}_{\theta\sim\hat{\rho}}\left[\KLbar(\theta) \right]+\log\mathbb{E}_{\pi}[e^{-\beta \KLbar(\theta)}]\label{eq:decomp}.
    \end{align}
    Using the above decomposition, and using Assumption~\ref{asm:asm1}, we have
    \begin{align*}
       &  \mathbb{E}_{\mathcal{S}}\mathbb{E}_{\theta\sim\hat{\rho}}[\Dbar(\theta)]
       \\
       & \le \mathbb{E}_{\mathcal{S}}\left[\frac{\alpha}{n(1-\alpha)}\mathbb{E}_{\theta\sim\hat{\rho}}[r_n(\theta)] \right]
       \\
       & \qquad + \frac{1}{n(1-\alpha)}\mathbb{E}_{\mathcal{S}}\left[KL(\hat{\rho}\|\pi)+\beta\mathbb{E}_{\theta\sim\hat{\rho}}\left[\KLbar(\theta)\right]+\log\mathbb{E}_{\pi}[e^{-\beta \KLbar(\theta)}]\right]
       \\
       & \le \mathbb{E}_{\mathcal{S}}\left[\frac{\alpha}{n(1-\alpha)}\mathbb{E}_{\theta\sim\hat{\rho}}[r_n(\theta)] \right]
       \\
       & \qquad + \frac{1}{n(1-\alpha)}\mathbb{E}_{\mathcal{S}}\left[KL(\hat{\rho}\|\pi)+\beta c(\alpha) \mathbb{E}_{\theta\sim\hat{\rho}}\left[\Dbar(\theta)\right]+\log\mathbb{E}_{\pi}[e^{-\beta \KLbar(\theta)}]\right].
    \end{align*}
    Rearranging the terms, which is licit because the left-hand side of Theorem~\ref{thm:thm2} is finite at $\pi$ itself, we have that:
    \begin{align*}
       & \left(1-\frac{\beta c(\alpha)}{n(1-\alpha)}\right)\mathbb{E}_{\mathcal{S}}\mathbb{E}_{\theta\sim\hat{\rho}}[\Dbar(\theta)]
       \\
       &\le
       \frac{\alpha}{n(1-\alpha)} \mathbb{E}_{\mathcal{S}}\left[\mathbb{E}_{\theta\sim\hat{\rho}}[r_n(\theta)] \right]+ \frac{1}{n(1-\alpha)}\mathbb{E}_{\mathcal{S}}\left[KL(\hat{\rho}\|\pi)+\log\mathbb{E}_{\pi}\left[e^{-\beta \KLbar(\theta)}\right]\right].
    \end{align*}
    Next, we use the definition of $\tilde{\rho}_{n,\alpha}$, which bounds the right-hand side by the same expression evaluated at any $\rho\in\mathcal{F}$. We then use the two facts below. The first follows from the monotonicity of the expectation at each fixed $\rho$, and the second holds for a $\rho$ that does not depend on the sample:
    \begin{align*}
        \mathbb{E}_{\mathcal{S}}\Bigl[\inf_{\rho\in\mathcal{F}}(\cdot)\Bigr]\le\inf_{\rho\in\mathcal{F}}\mathbb{E}_{\mathcal{S}}[\,\cdot\,],
        \qquad
        \mathbb{E}_{\mathcal{S}}\mathbb{E}_{\theta\sim\rho}\bigl[r_n(\theta)\bigr]=n\,\mathbb{E}_{\theta\sim\rho}\bigl[\KLbar(\theta)\bigr].
    \end{align*}
    We have that:
    \begin{align*}
      &  \left(1-\frac{\beta c(\alpha)}{n(1-\alpha)}\right)\mathbb{E}_{\mathcal{S}}\mathbb{E}_{\theta\sim\tilde{\rho}_{n,\alpha}}[\Dbar(\theta)]
    \\
    & \le \inf_{\rho\in \mathcal{F}}\left[\frac{\alpha}{(1-\alpha)}\mathbb{E}_{\theta\sim\rho}\left[\KLbar(\theta)\right] + \frac{KL(\rho\|\pi)+\log\mathbb{E}_{\pi}\left[e^{-\beta \KLbar(\theta)}\right]}{n(1-\alpha)}\right].
    \end{align*}
    The decomposition~\eqref{eq:decomp}, used in the other direction, gives
    \begin{align*}
        KL(\rho\|\pi)+\log\mathbb{E}_{\pi}\left[e^{-\beta \KLbar(\theta)}\right]
        = KL(\rho\|\pi_{-\beta})-\beta\mathbb{E}_{\theta\sim\rho}\left[\KLbar(\theta)\right].
    \end{align*}
    Substituting it and rearranging the terms, we have that:
    \begin{align*}
       \left(1-\frac{\beta c(\alpha)}{n(1-\alpha)}\right)\mathbb{E}_{\mathcal{S}}\mathbb{E}_{\theta\sim\tilde{\rho}_{n,\alpha}}[\Dbar(\theta)]&\le\inf_{\rho\in \mathcal{F}}\left[\left(\frac{\alpha -\frac{\beta}{n}}{1-\alpha}\right)\mathbb{E}_{\theta\sim\rho}\left[ \KLbar(\theta)\right] +\frac{KL(\rho\|\pi_{-\beta})}{n(1-\alpha)}\right].
    \end{align*}
    Using Assumption~\ref{asm:asm1} again,
        \begin{align*}
       \left(\frac{1}{c(\alpha)}-\frac{\beta }{n(1-\alpha)}\right)\mathbb{E}_{\mathcal{S}}\mathbb{E}_{\theta\sim\tilde{\rho}_{n,\alpha}}[\KLbar(\theta)]&\le\inf_{\rho\in \mathcal{F}}\left[\left(\frac{\alpha -\frac{\beta}{n}}{1-\alpha}\right)\mathbb{E}_{\theta\sim\rho}\left[ \KLbar(\theta)\right] +\frac{KL(\rho\|\pi_{-\beta})}{n(1-\alpha)}\right].
    \end{align*}
    We can divide both sides by $1/c(\alpha) -\beta/[n(1-\alpha)]$ on the condition that this quantity is $>0$.
    \end{proof}

    Note that, at $\beta=0$, Theorem~\ref{thm:thm4} reduces to Corollary~\ref{cor:cor1} combined with Assumption~\ref{asm:asm1}.

    We are now ready to prove Theorem~\ref{thm:thm1}.

    \begin{proof}[Proof of Theorem~\ref{thm:thm1}]
    We apply Theorem~\ref{thm:thm4}. If Assumption~\ref{asm:asm2} is satisfied and $\mathcal{F}=\mathcal{M}_1^+$ (so that $\tilde{\rho}_{n,\alpha}=\pi_{n,\alpha}$), we take $\rho=\pi_{-\beta}\in\mathcal{F}$, for which $KL(\rho\|\pi_{-\beta})=0$. As $\beta$ grows, the localized prior $\pi_{-\beta}$ concentrates on the Kullback-Leibler projection of the truth onto the model. Its average risk $\mathbb{E}_{\pi_{-\beta}}[\KLbar(\theta)]$ cannot be smaller than $\inf_{\tilde\theta\in\Theta}\KLbar(\tilde\theta)$. Assumption~\ref{asm:asm2} bounds the excess above this value, namely
    $$ \mathbb{E}_{\pi_{-\beta}}[\KLbar(\theta)]\le \inf_{\tilde\theta\in\Theta} \KLbar(\tilde\theta)+\frac{d_\pi}{\beta^{\kappa_\pi}}. $$
    The substitution below multiplies this upper bound by the factor $\alpha-\beta/n$. It is therefore valid on the range $\beta\le\alpha n$, where that factor is nonnegative. Note that this condition is part of the admissible range of the theorem. Moreover, in every model containing a $\tilde\theta$ with $0<\KLbar(\tilde\theta)<\infty$, the floor~\eqref{eq:ess-floor} of Subsection~\ref{subsec:ess} gives
    \begin{align*}
        c(\alpha)\ \ge\ \frac{1-\alpha}{\alpha},
        \qquad\text{hence}\qquad
        \frac{n(1-\alpha)}{c(\alpha)}\ \le\ \alpha n
        \quad\text{and}\quad
        \frac{n(1-\alpha)}{2c(\alpha)}\ \le\ \frac{\alpha n}{2}.
    \end{align*}
    Thus the condition is implied by $\beta<n(1-\alpha)/c(\alpha)$, and it holds at the choice $\beta=n(1-\alpha)/(2c(\alpha))$ made below. Substituting into Theorem~\ref{thm:thm4} gives
    \begin{align*}
        \mathbb{E}_{\mathcal{S}}\mathbb{E}_{\theta\sim\pi_{n,\alpha}}[\KLbar(\theta)]& \le \frac{c(\alpha) \alpha n - \beta c(\alpha) }{n(1-\alpha) - \beta c(\alpha)} \left(\inf_{\tilde\theta\in\Theta} \KLbar(\tilde\theta)+\frac{d_\pi}{\beta^{\kappa_\pi}}\right).
    \end{align*}
    For $\beta=n(1-\alpha)/(2c(\alpha))$, we have that:
    \begin{align*}
        \frac{c(\alpha)\alpha n-\beta c(\alpha)}{n(1-\alpha)-\beta c(\alpha)}
        =\frac{2c(\alpha)\alpha}{1-\alpha}-1\ \le\ \alpha\,\frac{2c(\alpha)}{1-\alpha},
        \qquad
        \frac{d_\pi}{\beta^{\kappa_\pi}}=d_\pi\Bigl(\frac{2c(\alpha)}{1-\alpha}\Bigr)^{\kappa_\pi}n^{-\kappa_\pi}.
    \end{align*}
    This gives the first closed-form bound of the theorem. If Assumption~\ref{asm:asm4} holds instead, we take the $\rho\in\mathcal{F}$ that it provides, which satisfies
    \begin{align*}
        \mathbb{E}_{\rho}[\KLbar(\theta)]-\inf_{\tilde\theta\in\Theta}\KLbar(\tilde\theta)+\frac{KL(\rho\|\pi_{-\beta})}{n}\ \le\ \frac{d_\pi}{\beta^{\kappa_\pi}}.
    \end{align*}
    Since $\alpha-\beta/n\le1$, Theorem~\ref{thm:thm4} at this $\rho$ gives
    \begin{align*}
        \mathbb{E}_{\mathcal{S}}\mathbb{E}_{\theta\sim\tilde{\rho}_{n,\alpha}}[\KLbar(\theta)]
        &\le \frac{c(\alpha)n}{n(1-\alpha)-\beta c(\alpha)}\left(\mathbb{E}_{\rho}[\KLbar(\theta)]+\frac{KL(\rho\|\pi_{-\beta})}{n}\right)
        \\
        &\le \frac{c(\alpha)n}{n(1-\alpha)-\beta c(\alpha)}\left(\inf_{\tilde\theta\in\Theta}\KLbar(\tilde\theta)+\frac{d_\pi}{\beta^{\kappa_\pi}}\right).
    \end{align*}
    At $\beta=n(1-\alpha)/(2c(\alpha))$, the prefactor equals $2c(\alpha)/(1-\alpha)$, and this gives the second closed-form bound of the theorem.
    \end{proof}

    \begin{rem}
    Theorem~\ref{thm:main-thm} requires neither a well-specified model nor identically distributed observations, since its proof only uses $P_i^{\circ}\ll\mu_i$ and the factorization step in the proof of Theorem~\ref{thm:thm2}. The oracle term of Theorem~\ref{thm:thm1} comes from two ingredients that are specific to it. First, the localized prior $\pi_{-\beta}$ can concentrate only on the best in-model approximation, so its risk floor is $\inf_{\tilde\theta\in\Theta}\KLbar(\tilde\theta)$ rather than zero. Second, Assumption~\ref{asm:asm1} replaces the R\'enyi risk by the Kullback-Leibler risk on both sides. In the well-specified case, the infimum is zero and it is attained at $\theta_0$. Theorem~\ref{thm:thm1} then reduces to a contraction bound of the type proved by~\citet{bhattacharya2016bayesian,alquier2019concentration}, without their extra $\log(n)$ factor at the parametric endpoint.
    \end{rem}

    Let us now state a version of Theorem~\ref{thm:thm4} that holds with large probability. It is stated at a single competitor chosen before seeing the sample, since the Chebyshev step below produces an event of its own for each competitor.

\begin{thm}
     \label{thm:thm5}
 Let $\alpha\in(0,1)$, $\delta,\delta'\in(0,1)$, $\mathcal{F}\subset \mathcal{M}_1^+$, and let $\rho\in\mathcal{F}$ be fixed, not depending on the sample.
    If Assumption~\ref{asm:asm1} is satisfied, then we have, for any $\beta<n(1-\alpha)/c(\alpha) $, with probability at least $1-\delta-\delta'$ on the sample $\mathcal{S}$,
    \begin{multline*}
      \mathbb{E}_{\theta\sim\tilde{\rho}_{n,\alpha}}[\KLbar(\theta)]
      \\
      \le
      \frac{c(\alpha) n}{n(1-\alpha) - \beta c(\alpha)}
      \biggl[\left(\alpha -\frac{\beta}{n}\right)\mathbb{E}_{\theta\sim\rho}\left[ \KLbar(\theta)\right]
      +\alpha \sqrt{\frac{\mathbb{E}_{\theta\sim\rho}\left[ \Vbar(\theta)\right]}{n\delta'}}
      \\
      +\frac{KL(\rho\|\pi_{-\beta})+\log\frac{1}{\delta}}{n}\biggr].
    \end{multline*}
\end{thm}

    \begin{proof}
    We start from Theorem~\ref{thm:thm2-p} at $\pi=\pi_{-\beta}$, which holds with probability at least $1-\delta$ over $\mathcal{S}$. From there, we use the decomposition~\eqref{eq:decomp}, Assumption~\ref{asm:asm1}, the rearrangement of the tilted term and the definition of $\tilde{\rho}_{n,\alpha}$, as in the proof of Theorem~\ref{thm:thm4}. The only difference is the additive $\log(1/\delta)$ of Theorem~\ref{thm:thm2-p}, which is kept unchanged through these steps. We then bound the infimum over $\mathcal{F}$ by its value at the fixed $\rho$ of the statement. This gives, with probability at least $1-\delta$ on $\mathcal{S}$,
    \begin{align*}
      & \left(1-\frac{\beta c(\alpha)}{n(1-\alpha)}\right)\mathbb{E}_{\theta\sim\tilde{\rho}_{n,\alpha}}[\Dbar(\theta)]
    \\
    & \le \frac{\alpha}{n(1-\alpha)} \mathbb{E}_{\theta\sim\rho}\left[r_n(\theta)\right] + \frac{KL(\rho\|\pi)+\log\mathbb{E}_{\pi}\left[e^{-\beta \KLbar(\theta)}\right]+\log\frac{1}{\delta}}{n(1-\alpha)}.
    \end{align*}
    In expectation, $\mathbb{E}_{\theta\sim\rho}[r_n(\theta)]$ would now be replaced by $n\mathbb{E}_{\theta\sim\rho}[\KLbar(\theta)]$. Here this replacement requires an additional fluctuation term, which we control by the second moment of the log-likelihood ratio. By Chebyshev's inequality and Fubini's theorem,
    $$ \mathbb{P}_{\mathcal{S}} \left[\mathbb{E}_{\theta\sim\rho}\left[r_n(\theta) - n\, \KLbar(\theta)\right] \geq t \right] \leq \frac{ \mathrm{Var}_{\mathcal{S}}\!\left(\mathbb{E}_{\theta\sim\rho}[r_n(\theta)]\right)}{t^2}
    \leq \frac{n\, \mathbb{E}_{\theta\sim\rho}[\Vbar(\theta)]}{t^2},
    $$
    where the last step uses, by Jensen's inequality and the bound of the variance by the second moment,
    \begin{align*}
    \mathrm{Var}_{\mathcal{S}}\!\left(\mathbb{E}_{\theta\sim\rho}[r_n(\theta)]\right)
    \leq \mathbb{E}_{\theta\sim\rho}\!\left[\mathrm{Var}_{\mathcal{S}}(r_n(\theta))\right]
    = \mathbb{E}_{\theta\sim\rho}\!\left[\sum_{i=1}^n\mathrm{Var}_{P_i^{\circ}}\!\left(\log\frac{p_i^{\circ}}{p_{i,\theta}}\right)\right]
    \leq n\,\mathbb{E}_{\theta\sim\rho}[\Vbar(\theta)].
    \end{align*}
    The middle equality holds because the $X_i$ are independent, and thus the variance of the sum $r_n(\theta)$ is the sum of the variances.
    Therefore, with probability at least $1-\delta'$ on the sample $\mathcal{S}$,
    $$
    \mathbb{E}_{\theta\sim\rho}\left[r_n(\theta) - n  \KLbar(\theta) \right] \leq \sqrt{n \frac{\mathbb{E}_{\theta\sim\rho}[\Vbar(\theta)] }{ \delta' }}.
    $$
     Intersecting it with the event of the first display, we have, with probability at least $1-\delta-\delta'$ on $\mathcal{S}$,
    \begin{align*}
       \left(1-\frac{\beta c(\alpha)}{n(1-\alpha)}\right)
       &  \mathbb{E}_{\theta\sim\tilde{\rho}_{n,\alpha}}[\Dbar(\theta)]
    \\
    & \le \frac{\alpha}{1-\alpha} \left[ \mathbb{E}_{\theta\sim\rho}\left[ \KLbar(\theta)\right]
    +  \sqrt{\frac{\mathbb{E}_{\theta\sim\rho}[\Vbar(\theta)] }{ n \delta' }}
    \right]
    \\
    & \qquad + \frac{KL(\rho\|\pi)+\log\mathbb{E}_{\pi}\left[e^{-\beta \KLbar(\theta)}\right]+\log\frac{1}{\delta}}{n(1-\alpha)}.
    \end{align*}
    We now use the decomposition~\eqref{eq:decomp}, we rearrange the terms and we use Assumption~\ref{asm:asm1} again. These are the same three steps as at the end of the proof of Theorem~\ref{thm:thm4}. We have that:
        \begin{multline*}
       \left(\frac{1}{c(\alpha)}-\frac{\beta }{n(1-\alpha)}\right)\mathbb{E}_{\theta\sim\tilde{\rho}_{n,\alpha}}[\KLbar(\theta)]
      \\
      \le\left(\frac{\alpha -\frac{\beta}{n}}{1-\alpha}\right)\mathbb{E}_{\theta\sim\rho}\left[ \KLbar(\theta)\right]
       + \frac{\alpha}{1-\alpha} \sqrt{\frac{\mathbb{E}_{\theta\sim\rho}[\Vbar(\theta)] }{ n \delta' }}
       +\frac{KL(\rho\|\pi_{-\beta})+ \log\frac{1}{\delta} }{n(1-\alpha)}.
    \end{multline*}
    We can divide both sides by $1/c(\alpha) -\beta/[n(1-\alpha)]$ on the condition that this quantity is $>0$.
    \end{proof}

    We can now prove Theorem~\ref{thm:thm1-p}.

    \begin{proof}[Proof of Theorem~\ref{thm:thm1-p}]
Apply Theorem~\ref{thm:thm5} at the fixed competitor $\rho=\pi_{-\beta}$, which does not depend on the sample, so that $KL(\rho\|\pi_{-\beta})=0$. By Assumption~\ref{asm:asm3},
$$ \mathbb{E}_{\pi_{-\beta}}[\KLbar(\theta)]\le \inf_{\tilde\theta\in\Theta} \KLbar(\tilde\theta)+\frac{d_\pi}{\beta^{\kappa_\pi}} \qquad\text{and}\qquad \mathbb{E}_{\pi_{-\beta}}[\Vbar(\theta)]\le \Vbar(\theta^{\star})+\frac{d_\pi'}{\beta^{\kappa_\pi}}. $$
We bound $(\alpha-\beta/n)/(1-\alpha)\le \alpha/(1-\alpha)$ and we take $\beta=n(1-\alpha)/(2c(\alpha))$, so that the prefactor $1/c(\alpha)-\beta/[n(1-\alpha)]$ equals $1/[2c(\alpha)]$. This gives the announced bound. Note that we keep the square-root term as it is. Indeed, the linearization
\begin{align*}
    \sqrt{\frac{\mathbb{E}_{\pi_{-\beta}}[\Vbar(\theta)]}{n\delta'}}\le 2\,\mathbb{E}_{\pi_{-\beta}}[\Vbar(\theta)]+\frac{1}{8 n\delta'},
\end{align*}
used in the well-specified case, would replace the fluctuation by a term proportional to $\mathbb{E}_{\pi_{-\beta}}[\Vbar(\theta)]$. Under misspecification, this term does not vanish, and it tends to $\Vbar(\theta^{\star})>0$.
    \end{proof}

\subsection{Other Proofs}
\label{subsec:proof:other}

The meaning of the symbols $L$ and $\kappa$ depends on the subsection. In the sequence-model proofs below, $L$ is the radius of the Sobolev ball $S_b(L)$ and $\kappa=2b/(2b+1)$ is its exponent. From the proof of Lemma~\ref{lem_exp_1} onwards, $L$ is a Lipschitz or smoothness constant and $\kappa=L/m$ is the condition number of the log-partition. Note that neither $\kappa$ is the exponent $\kappa_\pi$ of Assumptions~\ref{asm:asm2} to~\ref{asm:asm4}, which keeps its subscript throughout. The symbol $\Delta$ depends on the subsection in the same way: it is the displacement $\theta-\theta_0$ in the proofs for smooth models, and the expected objective gap of~\eqref{eq:logistic:Delta} in the proofs for logistic regression.
\begin{proof}[Proof of Lemma~\ref{lem_gaussian_asm1}]
In order to calculate the R\'enyi divergence, we use the formulas in~\citet[p. 45]{liese1987convex}. Since the covariance matrices are equal, they give
\begin{align*}
    \mathcal{D}_{\alpha}(P_{\theta}\|P_{\theta_0})=\frac{\alpha}{2 v^2}\left\|\theta-\theta_0\right\|^2.
\end{align*}
Moreover, $KL(P_{\theta_0}\|P_{\theta})=\|\theta-\theta_0\|^2/(2v^2)$. This proves the lemma.
\end{proof}

\begin{proof}[Proof of Lemma~\ref{lem_gaussian_asm2_a}]
From the proof of Lemma~\ref{lem_gaussian_asm1},
\begin{align*}
    KL(P_{\theta_0}\|P_{\theta})=\frac{1}{2 v^2}\left\|\theta-\theta_0\right\|^2.
\end{align*}
Our prior is $\pi = \mathcal{N}(0,\Sigma)$ for some positive definite matrix $\Sigma$. For any $\beta>0$, we have that:
\begin{align*}
    \pi_{-\beta}(\theta) \propto \exp\left(-\frac{\theta^{\top} \Sigma^{-1} \theta}{2}-\frac{\beta  }{2 v^2}\left\|\theta-\theta_0\right\|^2\right)
    \propto \exp\left(-
    \frac{\theta^{\top}\left(\Sigma^{-1} + \frac{\beta }{v^2} I_d\right)\theta
    }{2}+\frac{\beta }{v^2}\langle\theta,\theta_0\rangle\right).
 \end{align*}
 This is a Gaussian exponent with precision $\Sigma^{-1}+\beta v^{-2}I_d$ and linear term $\beta v^{-2}\theta_0$. Completing the square, we have that:
 \begin{align*}
    \pi_{-\beta}(\theta)\sim \mathcal{N}\left(\frac{\beta }{v^2}\left(\Sigma^{-1} + \frac{\beta }{v^2} I_d\right)^{-1} \theta_0 ,\left(\Sigma^{-1} + \frac{\beta }{v^2}I_d\right)^{-1} \right).
 \end{align*}
Thus,
 \begin{align*}
    \beta\mathbb{E}_{\pi_{-\beta}}\left[KL(P_{\theta_0}\|P_{\theta})\right]&=\frac{\beta}{2 v^2}\mathbb{E}_{\pi_{-\beta}}\left[\left\|\theta-\theta_0\right\|^2\right]\\
     &=\frac{\beta}{2 v^2}\left\| \left(\Sigma^{-1} + \frac{\beta}{v^2}I_d\right)^{-1} \Sigma^{-1}\theta_0 \right\|^2 +\frac{\beta }{2 v^2} {\rm  Tr}\left[ \left(\Sigma^{-1} + \frac{\beta}{v^2} I_d\right)^{-1} \right].
 \end{align*}
 In particular, when $\Sigma=\sigma^2 I$,
 $$
    \beta\mathbb{E}_{\pi_{-\beta }}\left[KL(P_{\theta_0}\|P_{\theta})\right]=\frac{ \frac{\beta}{v^2}\left\|\theta_0\right\|^2}{2\sigma^4\left( \frac{\beta}{v^2}+\frac{1}{\sigma^2}\right)^2}+\frac{ \frac{\beta}{v^2} d}{2\left( \frac{\beta}{v^2}+\frac{1}{\sigma^2}\right)} \leq \frac{d}{2}+\frac{\left\|\theta_0\right\|^2}{8 \sigma^2}.
$$
The last inequality holds because the function
\begin{align*}
    g(\beta)=\frac{ \beta/v^2}{2\sigma^4\left( \beta/v^2+1/\sigma^2\right)^2}
\end{align*}
is maximized at $\beta_0=v^2/\sigma^2$, where it equals $1/(8\sigma^2)$.
  Thus, Assumption~\ref{asm:asm2} is satisfied with $\kappa_\pi=1$ and $d_\pi=d/2+\left\|\theta_0\right\|^2/(8\sigma^2)$.
\vspace{0.15cm}

We now turn to Assumption~\ref{asm:asm3}. Write $u=\theta-\theta_0$ and $X=\theta_0+v\varepsilon$ with $\varepsilon\sim\mathcal{N}(0,I)$, so that
\begin{align*}
\log \frac{p_{\theta_0}(X) }{ p_\theta(X) }
= \frac{\|X-\theta\|^2-\|X-\theta_0\|^2}{2 v^2}
= \frac{\langle u,\,u-2v\varepsilon\rangle}{2 v^2}
= \frac{\|u\|^2-2v\langle u,\varepsilon\rangle}{2 v^2}.
\end{align*}
Using $\mathbb{E}[\langle u,\varepsilon\rangle]=0$ and $\mathbb{E}[\langle u,\varepsilon\rangle^2]=\|u\|^2$,
$$
\Vbar(\theta) = \mathbb{E}\!\left[\log^2 \frac{p_{\theta_0}(X)}{p_\theta(X)}\right]
= \frac{\|u\|^4+4v^2\|u\|^2}{4 v^4}
= \frac{\|\theta-\theta_0\|^2\bigl(\|\theta-\theta_0\|^2+4v^2\bigr)}{4 v^4},
$$
so that $\Vbar(\theta^{\star})=\Vbar(\theta_0)=0$. Hence
$$
\beta\, \mathbb{E}_{\pi_{-\beta}}[\Vbar(\theta)]
= \frac{\beta\,\mathbb{E}_{\pi_{-\beta}}\bigl[\|u\|^4\bigr]}{4 v^4} + \frac{\beta\,\mathbb{E}_{\pi_{-\beta}}\bigl[\|u\|^2\bigr]}{v^2}.
$$
By the first part of the proof, the second term equals $2\beta\,\mathbb{E}_{\pi_{-\beta}}[KL(P_{\theta_0}\|P_\theta)]$, which is at most $2d_\pi$. For the first term, $\pi_{-\beta}=\mathcal{N}(m_\beta,s_\beta^2 I)$ with
$$
s_\beta^2=\Bigl(\frac{1}{\sigma^{2}}+\frac{\beta}{v^2}\Bigr)^{-1}
\qquad\text{and}\qquad
m_\beta\ \longrightarrow\ \theta_0 .
$$
Thus $\mathbb{E}_{\pi_{-\beta}}[\|u\|^4]=\mathcal{O}(\beta^{-2})$ as $\beta\to\infty$, and it stays bounded for $\beta$ small, so that $\beta\,\mathbb{E}_{\pi_{-\beta}}[\|u\|^4]$ is bounded uniformly in $\beta$. Therefore, for a finite constant $d'_\pi$,
$$
\sup_{\beta\ge0}\ \beta\,\mathbb{E}_{\pi_{-\beta}}\bigl[\Vbar(\theta)\bigr]\ \le\ d'_\pi ,
$$
and Assumption~\ref{asm:asm3} holds with $\kappa_\pi=1$.
\end{proof}

\begin{proof}[Proof of Lemma~\ref{lem_gaussian_seq_scaled}]
Here $v=1$, and the prior~\eqref{eq:seq:prior} is $\pi=\mathcal{N}(0,\Sigma)$ with $\Sigma$ diagonal with entries $\sigma_j^2=\tau^2 j^{-1-2b}$. Note that $KL(P_{\theta_0}\|P_{\theta})=\|\theta-\theta_0\|^2/2$ separates over the coordinates. Thus the localized prior $\pi_{-\beta}$ is again a Gaussian product. Its $j$-th coordinate has precision $1/\sigma_j^2+\beta$ and mean $m_j = \beta\theta_{0,j}/(1/\sigma_j^2+\beta)$, so that
\begin{align*}
    (m_j-\theta_{0,j})^2 = \frac{\theta_{0,j}^2}{\sigma_j^4\left(1/\sigma_j^2+\beta\right)^2},
    \qquad
    \mathrm{Var}_{\pi_{-\beta}}(\theta_j) = \frac{1}{1/\sigma_j^2+\beta}.
\end{align*}
Summing the squared bias and the variance over the coordinates and multiplying by $\beta$, Lemma~\ref{lem_gaussian_asm2_a} gives that:
\begin{align}
\label{eq:seq:betaKL}
    \beta\,\mathbb{E}_{\theta\sim\pi_{-\beta}}\left[KL(P_{\theta_0}\|P_{\theta})\right]
    = \sum_{j}\left( \frac{\beta\,\theta_{0,j}^2}{2\,\sigma_j^4\left(1/\sigma_j^2+\beta\right)^2}
    + \frac{\beta}{2\left(1/\sigma_j^2+\beta\right)} \right),
\end{align}
where the sum runs over the retained coordinates $j\le n$. All the terms are nonnegative, and thus we bound~\eqref{eq:seq:betaKL} by the corresponding sum over every $j\ge 1$.
\vspace{0.15cm}

We split each of the two terms at the coordinate where the prior precision becomes larger than the localization parameter. Since $1/\sigma_j^2 = j^{1+2b}/\tau^2$, the condition $1/\sigma_j^2\le \beta$ is equivalent to $j\le x^{q}$, where $x := \beta\tau^2$, $q := 1/(1+2b)$ and $\kappa := 1-q = 2b/(1+2b)$. Below the split, we have $1/\sigma_j^2+\beta\ge \beta$, which bounds the two terms of~\eqref{eq:seq:betaKL} by $\theta_{0,j}^2/(\sigma_j^4\beta)$ and $1/2$. The first of these two bounds drops a factor $1/2$. Keeping it would give $A=L$ instead of $A=3L/2$ below, and a smaller $C_b$. Above the split, we have $1/\sigma_j^2+\beta\ge 1/\sigma_j^2$, which bounds them by $\beta\theta_{0,j}^2/2$ and $\beta\sigma_j^2/2$. We now treat the four sums in turn.
\vspace{0.15cm}

On the low-frequency block, using $1/\sigma_j^2 = j^{1+2b}/\tau^2$ and $\theta_0\in S_b(L)$, we have that:
\begin{align}
\label{eq:seq:lowbias}
    \sum_{j\le x^{q}} \frac{\theta_{0,j}^2}{\sigma_j^4\beta}
    &= \frac{1}{\beta\tau^4}\sum_{j\le x^{q}}\left(\theta_{0,j}^2\, j^{2b}\right) j^{2+2b}
    \le \frac{x^{(2+2b)q}}{\beta\tau^4}\sum_{j\ge 1}\theta_{0,j}^2\, j^{2b}
    \le \frac{L\,x^{(2+2b)q}}{\beta\tau^4}
    = L\,\beta^{q}\,\tau^{-\frac{4b}{1+2b}},
\end{align}
where the last equality follows from the exponent identities $(2+2b)q-1=q$ in $\beta$ and $2(2+2b)q-4=-4b/(1+2b)$ in $\tau$, at $x=\beta\tau^2$. The corresponding variance sum has at most $x^{q}$ terms, each at most $1/2$, so that:
\begin{align}
\label{eq:seq:lowvar}
    \sum_{j\le x^{q}} \frac{\beta}{2\left(1/\sigma_j^2+\beta\right)}
    \le \frac{x^{q}}{2} = \frac{1}{2}\,\beta^{q}\,\tau^{\frac{2}{1+2b}}.
\end{align}
Above the split we have $j^{2b}>x^{2bq}$, hence $\theta_{0,j}^2\le \theta_{0,j}^2\, j^{2b}\, x^{-2bq}$, and the high-frequency bias satisfies
\begin{align}
\label{eq:seq:highbias}
    \sum_{j>x^{q}} \frac{\beta\,\theta_{0,j}^2}{2}
    \le \frac{\beta}{2}\,x^{-2bq}\sum_{j\ge 1}\theta_{0,j}^2\, j^{2b}
    \le \frac{L\beta}{2}\,x^{-2bq}
    = \frac{L}{2}\,\beta^{q}\,\tau^{-\frac{4b}{1+2b}},
\end{align}
by $\beta x^{-2bq} = \beta^{q}\tau^{-4b/(1+2b)}$. The high-frequency variance is a tail of $\sum_j j^{-1-2b}$. We bound it by the comparison
\begin{align*}
    \sum_{j>a} f(j)\le f(a)+\int_a^{\infty} f(t)\,{\rm d}t,
\end{align*}
valid for a decreasing $f>0$, taken at $a=x^{q}$:
\begin{align}
\label{eq:seq:highvar}
    \sum_{j>x^{q}} \frac{\beta\,\sigma_j^2}{2}
    = \frac{\beta\tau^2}{2}\sum_{j>x^{q}} j^{-1-2b}
    \le \frac{\beta\tau^2}{2}\left( \left(x^{q}\right)^{-1-2b} + \int_{x^{q}}^{\infty} t^{-1-2b}\,{\rm d}t \right)
    = \frac{1}{2} + \frac{1}{4b}\,\beta^{q}\,\tau^{\frac{2}{1+2b}}.
\end{align}
Here the constant $1/2$ comes from the boundary term $(x^{q})^{-1-2b}=x^{-1}$ together with $\beta\tau^2=x$, and the remainder comes from
\begin{align*}
    \beta\tau^2 x^{-2bq} = x^{1-2bq} = x^{q} = \beta^{q}\tau^{2/(1+2b)}.
\end{align*}
\vspace{0.15cm}

Summing~\eqref{eq:seq:lowbias}--\eqref{eq:seq:highvar} and dividing by $\beta$, we have that:
\begin{align*}
    \mathbb{E}_{\theta\sim\pi_{-\beta}}\left[KL(P_{\theta_0}\|P_{\theta})\right]
    &\le \left( \left(L+\frac{L}{2}\right)\tau^{-\frac{4b}{1+2b}}
    + \left(\frac{1}{2}+\frac{1}{4b}\right)\tau^{\frac{2}{1+2b}} \right)\beta^{q-1} + \frac{1}{2\beta}
    \\
    &= d_{\pi}(\tau,L)\,\beta^{-\kappa} + \frac{1}{2\beta},
\end{align*}
since $q-1=-\kappa$, which is~\eqref{eq:seq:dpi-tau} with the constant~\eqref{eq:seq:dpi-def}.
\vspace{0.15cm}

The additive $1/(2\beta)$ is the scale-free boundary term of~\eqref{eq:seq:highvar}. It contains no power of $\tau$. It is therefore not involved in the optimization over $\tau$ that is carried out in the second part of Lemma~\ref{lem_gaussian_seq_scaled}, and it is added to the final constant afterwards.
\vspace{0.15cm}

It remains to convert~\eqref{eq:seq:dpi-tau} into Assumption~\ref{asm:asm2}. For $\beta\ge 1$ we have $\beta^{-1}\le \beta^{-\kappa}$, since $\kappa\le 1$. Thus
\begin{align*}
    \sup_{\beta\ge 1}\beta^{\kappa}\,\mathbb{E}_{\theta\sim\pi_{-\beta}}\bigl[KL(P_{\theta_0}\|P_{\theta})\bigr]\le d_{\pi}(\tau,L)+\frac12,
\end{align*}
which is~\eqref{eq:seq:asm2}.
\vspace{0.15cm}

It remains to minimize $d_{\pi}(\tau,L)$ in $\tau$. Let us put $y=\tau^{2/(1+2b)}>0$, so that $\tau^{-4b/(1+2b)}=y^{-2b}$ and $\tau^{2/(1+2b)}=y$. We have that:
\begin{align*}
    d_{\pi}(\tau,L) = A\,y^{-2b} + B\,y =: g(y),
    \qquad A = \frac{3L}{2},\quad B = \frac{1}{2}+\frac{1}{4b} = \frac{2b+1}{4b}.
\end{align*}
Then $g'(y) = -2bA\,y^{-2b-1} + B$ vanishes at $y_{\star}^{2b+1}=2bA/B$, and it changes sign from negative to positive there. Thus $y_{\star}$ is the unique minimizer, and $g''>0$ confirms convexity. From $\tau^2=y^{1+2b}$ we get $\tau_{\star}^2 = y_{\star}^{1+2b} = 2bA/B$, which is~\eqref{eq:seq:tau-star} once we insert $A$ and $B$. Using the stationarity condition $B\,y_{\star} = 2bA\,y_{\star}^{-2b}$, we have that:
\begin{align*}
    g(y_{\star}) &= A\,y_{\star}^{-2b} + B\,y_{\star}
    = (1+2b)\,A\,y_{\star}^{-2b}
    \\
    &= (1+2b)\,A\left(\frac{2bA}{B}\right)^{-\frac{2b}{2b+1}}
    = (1+2b)\,A^{\frac{1}{2b+1}}\left(\frac{B}{2b}\right)^{\frac{2b}{2b+1}},
\end{align*}
which is~\eqref{eq:seq:C-b} once we substitute $A=3L/2$ and $B/(2b)=(2b+1)/(8b^2)$.
\end{proof}

\begin{proof}[Proof of Corollary~\ref{cor:seq:combined}]
The model is well-specified, so $\inf_{\tilde\theta\in\Theta} KL(P_{\theta_0}\|P_{\tilde\theta})=0$ and the oracle term of Theorem~\ref{thm:thm1} vanishes. By Lemma~\ref{lem_gaussian_asm1}, Assumption~\ref{asm:asm1} holds with $c(\alpha)=1/\alpha$. By Lemma~\ref{lem_gaussian_seq_scaled}, the scale $\tau=\tau_{\star}$ gives $d_{\pi}(\tau_{\star},L)=C_b L^{1/(2b+1)}$. Thus Lemma~\ref{lem_gaussian_seq_scaled} gives Assumption~\ref{asm:asm2} with
\begin{align*}
    \kappa_\pi = \kappa = \frac{2b}{2b+1},
    \qquad
    d_\pi = C_b\, L^{\frac{1}{2b+1}} + \frac{1}{2},
\end{align*}
at every localization $\beta\ge 1$.
\vspace{0.15cm}

We apply Theorem~\ref{thm:thm1} at $\beta^{\star} = n(1-\alpha)/[2c(\alpha)]$, which equals $n\alpha(1-\alpha)/2$. This is the midpoint of the admissible range $\beta<n\alpha(1-\alpha)$. The hypothesis $n\alpha(1-\alpha)\ge 2$ gives $\beta^{\star}\ge 1$, which is the range on which the constants above are available.
\vspace{0.15cm}

We compute the two factors in the bound of Theorem~\ref{thm:thm1} separately. With $c(\alpha)=1/\alpha$, the numerator and the denominator of the prefactor are
\begin{align}
\label{eq:seq:num}
    c(\alpha)\,\alpha n - \beta^{\star} c(\alpha)
    &= \frac{1}{\alpha}\left(\alpha n - \frac{n\alpha(1-\alpha)}{2}\right)
    = n - \frac{n(1-\alpha)}{2} = \frac{n(1+\alpha)}{2},
    \\
\label{eq:seq:den}
    n(1-\alpha) - \beta^{\star} c(\alpha)
    &= n(1-\alpha) - \frac{n\alpha(1-\alpha)}{2}\cdot\frac{1}{\alpha}
    = n(1-\alpha) - \frac{n(1-\alpha)}{2} = \frac{n(1-\alpha)}{2},
\end{align}
so the prefactor equals $(1+\alpha)/(1-\alpha)$. For the complexity term, we have that:
\begin{align}
\label{eq:seq:beta-power}
    \frac{d_\pi}{(\beta^{\star})^{\kappa}}
    = d_\pi\left(\frac{n\alpha(1-\alpha)}{2}\right)^{-\kappa}
    = \frac{2^{\kappa}\,d_\pi}{(1-\alpha)^{\kappa}\,(n\alpha)^{\kappa}}.
\end{align}
Multiplying~\eqref{eq:seq:num}--\eqref{eq:seq:beta-power} and inserting the value $d_\pi = C_b L^{1/(2b+1)}+1/2$ gives~\eqref{eq:seq:combined}.
\vspace{0.15cm}

\end{proof}

\begin{proof}[Proof of Lemma~\ref{lem_exp_1}]
From here on $L$ is a Lipschitz constant and $\kappa=L/m$, as in Lemma~\ref{lem_exp_1}. The two identities below are algebraic identities in the log-partition $\psi$ of the definition of Subsection~\ref{subsec:expfam}, and they hold for any dominating measure $\mu$ on $(\mathbf{X},\mathcal{X})$. By~\citet{nielsen2011r}, the following holds:
\begin{align*}
    \mathcal{D}_{\alpha}(P_{\theta}\|P_{\theta_0})&=\frac{1}{1-\alpha}\left(\alpha\psi(\theta)+(1-\alpha)\psi(\theta_0)-\psi(\alpha\theta+(1-\alpha)\theta_0)\right),\\
    KL(P_{\theta_0}\|P_{\theta})&=\psi(\theta)-\psi(\theta_0)-\langle\theta-\theta_0,\nabla \psi(\theta_0)\rangle.
\end{align*}
Note that the convex set $\Theta$ contains the segment between $\theta_0$ and $\theta$, and that the two curvature bounds are available on this segment. As $\psi$ is $m$-strongly convex there, its Jensen gap at the weights $(\alpha,1-\alpha)$ is at least $\tfrac{m}{2}\alpha(1-\alpha)\|\theta_0-\theta\|^2$. Thus, for any $\alpha\in(0,1)$,
\begin{align*}
    \mathcal{D}_{\alpha}(P_{\theta}\|P_{\theta_0})&\ge \frac{m\alpha}{2} \left\|\theta_0-\theta\right\|^2.
\end{align*}
Similarly, as the gradient of $\psi$ is Lipschitz with constant $L$ on that segment, the Bregman divergence of $\psi$ is at most $\tfrac{L}{2}\|\theta_0-\theta\|^2$. Combining the two bounds, we have that:
    \begin{align*}
        KL(P_{\theta_0}\|P_{\theta})&\le \frac{L}{2} \left\|\theta_0-\theta\right\|^2 = \frac{m\kappa}{2} \left\|\theta_0-\theta\right\|^2 = \frac{\kappa}{\alpha}\cdot\frac{m\alpha}{2} \left\|\theta_0-\theta\right\|^2 \leq  \frac{\kappa}{\alpha} \mathcal{D}_{\alpha}(P_{\theta}\|P_{\theta_0}).
    \end{align*}
\end{proof}

\begin{proof}[Proof of Theorem~\ref{thm:fisher_divergences_remainder}]
Write $\ell_\theta(x)=\log p_\theta(x)$, $\Delta=\theta-\theta_0$, and let $\dot\ell=\partial_\theta\ell_\theta|_{\theta_0}$ and $\ddot\ell=\partial^2_\theta\ell_\theta|_{\theta_0}$. Assumption~\ref{asm:differentiability} allows every exchange of differentiation and integration below. It also gives the two standard identities
\begin{align*}
    \int p_{\theta_0}\,\dot\ell = 0,
    \qquad
    \int p_{\theta_0}\,\ddot\ell = -I(\theta_0),
    \qquad
    I(\theta_0)=\int p_{\theta_0}\,\dot\ell^{\,2}.
\end{align*}
In each of the three cases, we expand to second order in $\Delta$, we integrate against $p_{\theta_0}$, and we recognize the Fisher information. We then bound the third-order Taylor remainder, written in integral form, using the domination of Assumption~\ref{asm:differentiability}. Throughout, $\theta\in U_0$, so the segment between $\theta_0$ and $\theta$ lies in $U$.
\vspace{0.15cm}

We treat the Kullback-Leibler divergence first. Taylor's theorem in integral form, applied to $t\mapsto \ell_t(x)$ on the segment between $\theta_0$ and $\theta$, gives
\begin{align*}
    \ell_\theta(x)-\ell_{\theta_0}(x)=\dot\ell(x)\,\Delta+\frac12\,\ddot\ell(x)\,\Delta^2+R_3(x),
    \qquad
    R_3(x)=\frac12\int_{\theta_0}^{\theta}(\theta-t)^2\,\partial_t^3\ell_t(x)\,{\rm d}t.
\end{align*}
The domination $|\partial_t^3\ell_t|\le M_3$ on $U$ then gives $|R_3(x)|\le M_3(x)\,|\Delta|^3/6$. We now integrate $-(\ell_\theta-\ell_{\theta_0})$ against $p_{\theta_0}$. The first-order term vanishes, and by the identities above the second-order term is
\begin{align*}
    -\frac12\Delta^2\int p_{\theta_0}\ddot\ell=\frac12 I(\theta_0)\Delta^2 .
\end{align*}
Thus we have that:
\begin{align*}
    KL(P_{\theta_0} \| P_\theta) &= \frac{1}{2}I(\theta_0)\Delta^2 + r_K(\theta),
    \qquad
    r_K(\theta)=-\int p_{\theta_0}\,R_3,
    \\
    |r_K(\theta)|&\le \frac{\mathbb{E}_{p_{\theta_0}}[M_3]}{6}\,|\Delta|^3\le\frac{\bar{\mu}}{6}\,|\Delta|^3.
\end{align*}
The last step holds because $\theta_0\in U$ and $M_3\le M_1^3+M_1M_2+M_3$.
\vspace{0.15cm}

Both remaining divergences are functions of the affinity $\rho(\theta)=\int\sqrt{p_\theta p_{\theta_0}}$. Let us write
\begin{align*}
    s_t(x)&=\sqrt{\frac{p_t(x)}{p_{\theta_0}(x)}}=e^{(\ell_t(x)-\ell_{\theta_0}(x))/2},
    \\
    \rho(\theta)&=\mathbb{E}_{p_{\theta_0}}[s_\theta],
    \qquad H^2=2(1-\rho),
    \qquad \mathcal{D}_{1/2}=-2\log\rho .
\end{align*}
We therefore expand $\rho$. Differentiating $t\mapsto s_t(x)$ three times, and writing $\dot\ell_t,\ddot\ell_t$ for the derivatives of $\ell$ at $t$, we have that:
\begin{align*}
    \partial_t s_t=\frac{s_t\,\dot\ell_t}{2},
    \qquad
    \partial_t^2 s_t=s_t\left(\frac{\dot\ell_t^{\,2}}{4}+\frac{\ddot\ell_t}{2}\right),
    \qquad
    \partial_t^3 s_t=s_t\left(\frac{\dot\ell_t^{\,3}}{8}+\frac{3\,\dot\ell_t\,\ddot\ell_t}{4}+\frac{\partial_t^3\ell_t}{2}\right).
\end{align*}
Taylor's theorem in integral form applied to $t\mapsto s_t(x)$, integrated against $p_{\theta_0}$, gives, using $s_{\theta_0}=1$ and the two identities above once more,
\begin{align*}
    \rho(\theta)
    = 1 + \frac{\Delta^2}{2}\left[\frac{I(\theta_0)}{4}-\frac{I(\theta_0)}{2}\right] + \mathbb{E}_{p_{\theta_0}}[R_s]
    = 1-\frac{1}{8}I(\theta_0)\Delta^2+\mathbb{E}_{p_{\theta_0}}[R_s],
\end{align*}
where
\begin{align*}
    R_s(x)=\frac12\int_{\theta_0}^{\theta}(\theta-t)^2\,\partial_t^3 s_t(x)\,{\rm d}t .
\end{align*}
To bound the remainder, we bound $\mathbb{E}_{p_{\theta_0}}[|\partial_t^3 s_t|]$ for each $t\in U$. The domination gives $|\partial_t^3 s_t|\le s_t\,G$, with
\begin{align*}
    G=\frac18 M_1^3+\frac34 M_1M_2+\frac12 M_3\le\frac34\bigl(M_1^3+M_1M_2+M_3\bigr).
\end{align*}
Moreover, $p_{\theta_0}s_t=\sqrt{p_{\theta_0}p_t}\le(p_{\theta_0}+p_t)/2$, so that:
\begin{align*}
    \mathbb{E}_{p_{\theta_0}}\bigl[|\partial_t^3 s_t|\bigr]
    \ \le\ \int\sqrt{p_{\theta_0}\,p_t}\;G
    \ \le\ \frac{\mathbb{E}_{p_{\theta_0}}[G]+\mathbb{E}_{p_t}[G]}{2}
    \ \le\ \frac{3\,\bar{\mu}}{4},
\end{align*}
by the definition of $\bar{\mu}$, since $\theta_0$ and $t$ both lie in $U$. By Fubini's theorem, which this bound allows,
\begin{align*}
    \bigl|\mathbb{E}_{p_{\theta_0}}[R_s]\bigr|
    \ \le\ \frac12\int_{\theta_0}^{\theta}(\theta-t)^2\,\mathbb{E}_{p_{\theta_0}}\bigl[|\partial_t^3 s_t|\bigr]\,{\rm d}t
    \ \le\ \frac12\cdot\frac{3\,\bar{\mu}}{4}\cdot\frac{|\Delta|^3}{3}
    = \frac{\bar{\mu}}{8}\,|\Delta|^3 .
\end{align*}
\vspace{0.15cm}

For the squared Hellinger distance, multiplying the expansion of $\rho$ by $-2$ and adding $2$, we obtain
\begin{align*}
    H^2(P_\theta, P_{\theta_0}) = \frac{1}{4}I(\theta_0)\Delta^2 + r_H(\theta),
    \qquad
    r_H(\theta) = -2\,\mathbb{E}_{p_{\theta_0}}[R_s],
    \qquad
    |r_H(\theta)|\le\frac{\bar{\mu}}{4}\,|\Delta|^3 .
\end{align*}
\vspace{0.15cm}

It remains to treat the R\'enyi divergence. Here
\begin{align*}
    \mathcal{D}_{1/2}(P_\theta\|P_{\theta_0})=-2\log\rho(\theta)=-2\log\Bigl(1-\frac{H^2(P_\theta,P_{\theta_0})}{2}\Bigr).
\end{align*}
We use the elementary bounds $x\le-\log(1-x)\le x+x^2$ for $x\in[0,1/2]$, the upper one following from
\begin{align*}
    -\log(1-x)=x+\sum_{k\ge2}\frac{x^k}{k}\le x+\frac{x^2}{2(1-x)}.
\end{align*} Taking $x=H^2/2$, and writing $H^2$ for $H^2(P_\theta,P_{\theta_0})$, we have that:
\begin{align*}
    H^2\ \le\ \mathcal{D}_{1/2}(P_\theta \| P_{\theta_0})\ \le\ H^2 + \frac{H^4}{2}
    \qquad\text{whenever } H^2\le1 .
\end{align*}
On $U_0$ we have $|\Delta|\le1$ and $\Delta^2(I(\theta_0)+\bar{\mu})\le4$. Thus the Hellinger expansion gives
\begin{align*}
    H^2\le\frac14 I(\theta_0)\Delta^2+\frac{\bar{\mu}}{4}\Delta^2 = \frac14\bigl(I(\theta_0)+\bar{\mu}\bigr)\Delta^2\le1 .
\end{align*}
In particular, the display above applies. Moreover,
\begin{align*}
    \frac{H^4}{2}\ \le\ \frac{\bigl(I(\theta_0)+\bar{\mu}\bigr)^2}{32}\,\Delta^4\ \le\ \frac{\bigl(I(\theta_0)+\bar{\mu}\bigr)^2}{32}\,|\Delta|^3 ,
\end{align*}
where the last step uses $|\Delta|\le1$. Therefore
\begin{align*}
    \mathcal{D}_{1/2}(P_\theta\|P_{\theta_0})=\frac14 I(\theta_0)\Delta^2+r_R(\theta),
    \qquad
    r_R=r_H+\bigl(\mathcal{D}_{1/2}-H^2\bigr),
\end{align*}
and
\begin{align*}
    |r_R(\theta)|\ \le\ \frac{\bar{\mu}}{4}\,|\Delta|^3+\frac{\bigl(I(\theta_0)+\bar{\mu}\bigr)^2}{32}\,|\Delta|^3,
\end{align*}
which is the constant of the statement.
\end{proof}

\begin{proof}[Proof of Corollary~\ref{cor:fisher_divergences_remainder}]
The two-sided bounds follow from Theorem~\ref{thm:fisher_divergences_remainder}, from $m\le I(\theta_0)\le M$, and from the fact that $C=\bar{\mu}/4+(M+\bar{\mu})^2/32$ dominates the three remainder constants of the theorem, again by $I(\theta_0)\le M$.
\vspace{0.15cm}

For the last claim, fix $\epsilon\in(0,1)$ and restrict to $|\Delta|\le\epsilon m/(4C)$. On this range, the cubic terms are at most an $\epsilon$-fraction of the leading quadratic ones. Indeed,
\begin{align*}
    C|\Delta|^3\le\frac{\epsilon m}{4}\Delta^2\le\frac{\epsilon}{4} I(\theta_0)\Delta^2 .
\end{align*}
The theorem then gives
\begin{align*}
    KL(P_{\theta_0}\|P_\theta)\ \le\ \frac{(1+\epsilon)\,I(\theta_0)}{2}\,\Delta^2,
    \qquad
    \mathcal{D}_{1/2}(P_\theta\|P_{\theta_0})\ \ge\ \frac{(1-\epsilon)\,I(\theta_0)}{4}\,\Delta^2,
\end{align*}
hence
\begin{align*}
    KL(P_{\theta_0}\|P_\theta)\ \le\ \frac{2(1+\epsilon)}{1-\epsilon}\,\mathcal{D}_{1/2}(P_\theta\|P_{\theta_0}).
\end{align*}
It remains to bound $\mathcal{D}_{1/2}$ by $\mathcal{D}_{\alpha}/\alpha$ for an arbitrary order. Assume first that $\alpha\le1/2$. The skew symmetry recalled after Assumption~\ref{asm:asm1}, then the monotonicity of the R\'enyi divergence in the order, then the symmetry of $\mathcal{D}_{1/2}$ in its two arguments, and finally $1-\alpha\le1$, give
\begin{align*}
    \mathcal{D}_{\alpha}(P_\theta\|P_{\theta_0})
    &=\frac{\alpha}{1-\alpha}\,\mathcal{D}_{1-\alpha}(P_{\theta_0}\|P_\theta)
    \ \ge\ \frac{\alpha}{1-\alpha}\,\mathcal{D}_{1/2}(P_{\theta_0}\|P_\theta)
    \\
    &=\frac{\alpha}{1-\alpha}\,\mathcal{D}_{1/2}(P_\theta\|P_{\theta_0})
    \ \ge\ \alpha\,\mathcal{D}_{1/2}(P_\theta\|P_{\theta_0}).
\end{align*}
For $\alpha\ge1/2$, monotonicity alone gives $\mathcal{D}_{\alpha}\ge\mathcal{D}_{1/2}\ge\alpha\,\mathcal{D}_{1/2}$. In both cases,
\begin{align*}
    \mathcal{D}_{1/2}(P_\theta\|P_{\theta_0})\le\frac{1}{\alpha}\,\mathcal{D}_{\alpha}(P_\theta\|P_{\theta_0}),
\end{align*}
and combining this with the bound on $KL(P_{\theta_0}\|P_\theta)$ above gives~\eqref{eq:fisher:calpha}.
\end{proof}

\begin{proof}[Proof of Lemma~\ref{lem:one_dim}]
Here $L$ and $m$ are the smoothness and strong-convexity constants of~\eqref{hyp-dim} and $\kappa=L/m$ is their ratio, as in Lemma~\ref{lem_exp_1}. For the uniform prior on $\Theta=(-M,M)$, the localized prior is
\begin{align*}
    \pi_{-\beta}({\rm d}\theta) = \frac{e^{-\beta KL(P_{\theta_0}\|P_{\theta})}}{Z_\beta} \cdot \frac{1}{2M}\mathds{1}\left(\theta\in[-M,M]\right){\rm d}\theta.
\end{align*}
Fix $\beta>0$. At $\beta=0$ the quantity to bound vanishes. We use the assumption in~\eqref{hyp-dim} to bound, pointwise, the divergence from above and its exponential weight from above, and to bound the normalizing integral from below. We obtain:
\begin{align*}
    F(\beta):=\mathbb{E}_{\theta\sim \pi_{-\beta}}\left[ KL(P_{\theta_0}\|P_{\theta})\right]
    \leq \frac{L}{2}\,\frac{\int_{-M}^M (\theta-\theta_0)^2 e^{-\frac{\beta m}{2}(\theta-\theta_0)^2}\,{\rm d}\theta}{\int_{-M}^M e^{-\frac{\beta L}{2}(\theta-\theta_0)^2}\,{\rm d}\theta}.
\end{align*}
We evaluate the two integrals after the substitutions $u=\sqrt{\beta m}\,(\theta-\theta_0)$ and $u=\sqrt{\beta L}\,(\theta-\theta_0)$, which map $[-M,M]$ onto the intervals
\begin{align*}
    I_m=\sqrt{\beta m}\,\bigl[-M-\theta_0,\,M-\theta_0\bigr],
    \qquad
    I_L=\sqrt{\beta L}\,\bigl[-M-\theta_0,\,M-\theta_0\bigr];
\end{align*}
both contain $0$, since $\theta_0\in(-M,M)$, and $I_m\subseteq I_L$, since $m\le L$. We have that:
\begin{align*}
    \int_{-M}^M (\theta-\theta_0)^2 e^{-\frac{\beta m}{2}(\theta-\theta_0)^2}\,{\rm d}\theta &= (\beta m)^{-3/2}\int_{I_m} u^2 e^{-u^2/2}\,{\rm d}u,
    \\
    \int_{-M}^M e^{-\frac{\beta L}{2}(\theta-\theta_0)^2}\,{\rm d}\theta &= (\beta L)^{-1/2}\int_{I_L} e^{-u^2/2}\,{\rm d}u .
\end{align*}
On any interval $[a,b]$ with $a\le 0\le b$, integration by parts gives
\begin{align*}
    \int_{a}^{b} u^2 e^{-u^2/2}\,{\rm d}u
    = a\,e^{-a^2/2}-b\,e^{-b^2/2}+\int_{a}^{b} e^{-u^2/2}\,{\rm d}u
    \ \leq\ \int_{a}^{b} e^{-u^2/2}\,{\rm d}u ,
\end{align*}
since the boundary term is nonpositive. Applying this on $I_m$, and then enlarging the domain to $I_L$, we have that:
\begin{align*}
    \int_{I_m} u^2 e^{-u^2/2}\,{\rm d}u \ \le\ \int_{I_m} e^{-u^2/2}\,{\rm d}u \ \le\ \int_{I_L} e^{-u^2/2}\,{\rm d}u .
\end{align*}
Therefore, uniformly in $\beta>0$,
\begin{align*}
    \beta F(\beta) \ \leq\ \beta\cdot\frac{L}{2}\cdot\frac{(\beta m)^{-3/2}}{(\beta L)^{-1/2}}
    = \frac{L}{2}\cdot\frac{L^{1/2}}{m^{3/2}}
    = \frac{\kappa^{3/2}}{2}.
\end{align*}
This verifies Assumption~\ref{asm:asm2} with $\kappa_\pi = 1$ and $d_\pi=\kappa^{3/2}/2$.
\end{proof}

\begin{proof}[Proof of Corollary~\ref{cor:MLE:stoch}]
We first bound the empirical objective. Put $M^2=G^2+d\nu^2$.
 Conditional on $\mathcal{S}$, the noisy gradient in \eqref{eq:noisy-update} is unbiased
 and has second moment at most $M^2$. For any fixed $u\in\Theta$, projection
 and strong convexity give, with $\eta_t=2/[m(t+1)]$ and
 $a_t=\mathbb{E}_{\rm alg}[\|\theta_t-u\|^2]$,
\[
 2\eta_t\,\mathbb{E}_{\rm alg}
       [F_{\mathcal{S}}(\theta_t)-F_{\mathcal{S}}(u)]
 \leq (1-m\eta_t)a_t-a_{t+1}+\eta_t^2M^2.
\]
Dividing by $2\eta_t$, multiplying by $t$ and summing makes the distance
 terms telescope, since their coefficients are $mt(t-1)/4$ and
 $-mt(t+1)/4$. Convexity and the definition of $\bar\theta_T$ then yield
\begin{equation}
 \mathbb{E}_{\rm alg}[F_{\mathcal{S}}(\bar\theta_T)-F_{\mathcal{S}}(u)]
 \leq\frac{2M^2}{m(T+1)}.
 \label{eq:noisy-optimization}
\end{equation}
Conditionally on the sample,~\eqref{eq:noisy-optimization} also holds at $u$ equal to a constrained MLE.
 At $u=\theta_0$, its left-hand side is
 $n^{-1}\mathbb{E}_{\rm alg}[r_n(\bar\theta_T,\theta_0)]$.

For the information bound, let $Q$ be the law of the same projected process
 with $\bar s_n$ replaced by the deterministic vector
 $\mathbb{E}_{\theta_0}[s(X_1)]=\nabla\psi(\theta_0)$.
  At each common path prefix,
 the two unprojected Gaussian transitions have covariance
 $\eta_t^2\nu^2 I_d$ and their means differ by
 $\eta_t(\bar s_n-\nabla\psi(\theta_0))$. Their KL divergence is therefore
 $\|\bar s_n-\nabla\psi(\theta_0)\|^2/(2\nu^2)$; projection cannot increase it.
 The chain rule for relative entropy, followed by data processing, gives
\begin{align}
 \mathcal{I}(\bar\theta_T,\mathcal{S})
 &\leq \mathbb{E}_{\mathcal{S}}\bigl[KL(P_{\theta_{1:T+1}\mid \mathcal{S}}\|Q)\bigr] \notag\\
 &\leq\frac{T}{2\nu^2}\mathbb{E}\bigl[\|\bar s_n-\nabla\psi(\theta_0)\|^2\bigr]
 \leq\frac{TG^2}{2n\nu^2}.
 \label{eq:noisy-information}
\end{align}
The last inequality follows from independence and \eqref{eq:noisy-score}
 at $\theta_0$. In particular, the noise level $\nu^2=TG^2/(dn)$ makes
 $\mathcal{I}(\bar\theta_T,\mathcal{S})\leq d/2$.

Apply Theorem~\ref{thm:main-thm} with the conditional law of $\bar\theta_T$ given $\mathcal{S}$.
 Inequalities~\eqref{eq:noisy-optimization} and~\eqref{eq:noisy-information}
 give \eqref{eq:noisy-general}. Substituting
 $\nu^2=TG^2/(dn)$ and using $T/(T+1)\leq1$ gives
 \eqref{eq:noisy-tuned}. Finally, the exponential-family identity and
 strong convexity, as in the proof of Lemma~\ref{lem_exp_1}, give
\begin{align*}
    \mathcal{D}_{\alpha}(P_\theta\|P_{\theta_0})\geq\frac{\alpha m}{2}\|\theta-\theta_0\|^2 .
\end{align*}
\end{proof}

\subsubsection*{Proofs for the Logistic Regression Model}

Throughout this part the design is fixed and satisfies~\eqref{eq:logistic:gram}, $\mathbb{E}_{\mathcal{S}}$ is the expectation over the responses at that design, and $G_n$ is the empirical Gram matrix. Conditionally on the design, $P_{i,\theta}$ is the law of $Y_i$ alone, a two-point law on $\{-1,+1\}$ with $p_{i,\theta}(y)=\sigma(yu_i)$, $u_i=z_i^{\top}\theta$ and $u_{0,i}=z_i^{\top}\theta_0$. Since $\sigma(u)=e^{u}/(1+e^{u})$ and $\sigma(-u)=1/(1+e^{u})$, both values of $y$ are covered by
\begin{align}
\label{eq:pflog:expfam}
    p_u(y)=\exp\bigl(\mathds{1}\left(y=+1\right)u-\psi(u)\bigr),
    \qquad
    \psi(u)=\log\bigl(1+e^{u}\bigr),
    \qquad
    \psi'(u)=\sigma(u),
\end{align}
with $\psi''(u)=\sigma(u)(1-\sigma(u))$, with respect to the counting measure on $\{-1,+1\}$. This is the Bernoulli exponential family of Example~\ref{exm:bernoulli} in its natural parameter $u$, and we write $P_u$ for its law. The Kullback-Leibler divergence of an exponential family is the Bregman divergence of its log-partition,
\begin{align}
\label{eq:pflog:bregman}
    KL(P_{u_0}\|P_{u})=\psi(u)-\psi(u_0)-\psi'(u_0)(u-u_0),
\end{align}
by $\mathbb{E}_{P_{u_0}}[\mathds{1}(Y=+1)]=\psi'(u_0)$, and summing~\eqref{eq:pflog:expfam} over $y\in\{-1,+1\}$ gives the closed form of the R\'enyi divergence,
\begin{align}
\label{eq:pflog:renyiclosed}
    \mathcal{D}_{\alpha}(P_u\|P_{u_0})=\frac{\alpha\psi(u)+(1-\alpha)\psi(u_0)-\psi\bigl(\alpha u+(1-\alpha)u_0\bigr)}{1-\alpha},
\end{align}
which is the Jensen gap of $\psi$ at the weights $(\alpha,1-\alpha)$, divided by $1-\alpha$. Both divergences are additive over the coordinates of a product measure, and thus the averages~\eqref{eq:averages} are given by
\begin{align}
\label{eq:pflog:add}
    \KLbar(\theta)=\frac{1}{n}\sum_{i=1}^n KL\bigl(P_{u_{0,i}}\big\|P_{u_i}\bigr),
    \qquad
    \Dbar(\theta)=\frac{1}{n}\sum_{i=1}^n \mathcal{D}_{\alpha}\bigl(P_{u_i}\big\|P_{u_{0,i}}\bigr),
\end{align}
and $n\KLbar(\theta)$, $n\Dbar(\theta)$ are the divergences between the joint conditional laws $\bigotimes_{i}P_{u_{0,i}}$ and $\bigotimes_i P_{u_i}$. We use repeatedly the identity
\begin{align}
\label{eq:pflog:gramform}
    \frac1n\sum_{i=1}^n\bigl(z_i^{\top}(\theta-\theta_0)\bigr)^2=(\theta-\theta_0)^{\top}G_n(\theta-\theta_0).
\end{align}
\vspace{0.15cm}

The first lemma collects what the log-partition gives on the two sides. It provides a global quadratic upper bound on $\KLbar$, which is used for the competitor. It also provides a local quadratic lower bound on $\Dbar$, together with a separation outside the local ball, which is used for the partition function.

\begin{lem}
\label{lem:logistic:quad}
Fix the design as in~\eqref{eq:logistic:gram}, fix $r>0$ and let $\eta_r$ be as in~\eqref{eq:logistic:constants}. Then, for every $\theta\in\mathbb{R}^d$,
\begin{align}
\label{eq:pflog:quad}
    \KLbar(\theta)\ \le\ \frac18\,(\theta-\theta_0)^{\top}G_n(\theta-\theta_0),
\end{align}
while, on the ball $\|\theta-\theta_0\|\le r$,
\begin{align}
\label{eq:pflog:curv}
    \Dbar(\theta)\ \ge\ \frac{\alpha\eta_r}{2}\,(\theta-\theta_0)^{\top}G_n(\theta-\theta_0)
    \ \ge\ \frac{\alpha\eta_r\lambda_{\min}(G_n)}{2}\,\|\theta-\theta_0\|^2 .
\end{align}
Moreover, for every unit vector $v$ the map $t\mapsto\Dbar(\theta_0+tv)$ is nondecreasing on $t\ge0$, so that
\begin{align}
\label{eq:pflog:sep}
    \inf_{\|\theta-\theta_0\|\ge r}\Dbar(\theta)=\inf_{\|\theta-\theta_0\|=r}\Dbar(\theta)=m_{r,\alpha}
    \ \ge\ \frac{\alpha\eta_r\lambda_{\min}(G_n)r^2}{2}\ >\ 0 .
\end{align}
Finally, for $\rho_{\Sigma}=\mathcal{N}(\theta_0,\Sigma)$ with $\Sigma\succ0$,
\begin{align}
\label{eq:pflog:Erho}
    \mathbb{E}_{\rho_{\Sigma}}\bigl[\KLbar(\theta)\bigr]\ \le\ \frac18\,{\rm Tr}\bigl(G_n\Sigma\bigr),
    \qquad
    KL(\rho_{\Sigma}\|\pi)=\frac12\Bigl[\frac{{\rm Tr}(\Sigma)}{\tau^2}+\frac{\|\theta_0\|^2}{\tau^2}-d+\log\det\bigl(\tau^2\Sigma^{-1}\bigr)\Bigr].
\end{align}
\end{lem}

\begin{proof}
By~\eqref{eq:pflog:bregman} and Taylor's theorem, for each $i$ there is a $\zeta_i$ between $u_{0,i}$ and $u_i$ such that
\begin{align*}
    KL(P_{u_{0,i}}\|P_{u_i})=\frac12\,\psi''(\zeta_i)\,(u_i-u_{0,i})^2\ \le\ \frac18\,(u_i-u_{0,i})^2,
\end{align*}
since $\psi''\le1/4$ on $\mathbb{R}$. Averaging over $i$ and using~\eqref{eq:pflog:gramform} gives~\eqref{eq:pflog:quad}, on the whole parameter space.
\vspace{0.15cm}

For the lower bound, let $\|\theta-\theta_0\|\le r$. By the Cauchy-Schwarz and the triangle inequalities, we have that:
\begin{align*}
    |u_i|\ \le\ \|z_i\|\bigl(\|\theta_0\|+\|\theta-\theta_0\|\bigr)\ \le\ B(R_0+r),
    \qquad
    |u_{0,i}|\ \le\ BR_0 .
\end{align*}
The interpolation point $\alpha u_i+(1-\alpha)u_{0,i}$ lies in the same interval, on which
\begin{align*}
    \psi''(u)=\frac{1}{2+e^{u}+e^{-u}}\ \ge\ \eta_r,
\end{align*}
since $u\mapsto e^{u}+e^{-u}$ is even and nondecreasing on $[0,\infty)$. Thus $x\mapsto\psi(x)-\eta_rx^2/2$ is convex on this interval, and its Jensen gap at the weights $(\alpha,1-\alpha)$ is nonnegative. Writing $u=u_i$ and $u_0=u_{0,i}$, we have that:
\begin{align*}
    \alpha\psi(u)+(1-\alpha)\psi(u_0)-\psi\bigl(\alpha u+(1-\alpha)u_0\bigr)
    &\ \ge\ \frac{\eta_r}{2}\Bigl(\alpha u^2+(1-\alpha)u_0^2-\bigl(\alpha u+(1-\alpha)u_0\bigr)^2\Bigr)
    \\
    &\ =\ \frac{\eta_r}{2}\,\alpha(1-\alpha)(u-u_0)^2,
\end{align*}
where the last identity follows by direct expansion. Dividing by $1-\alpha$ and using~\eqref{eq:pflog:renyiclosed}, we obtain
\begin{align*}
    \mathcal{D}_{\alpha}(P_u\|P_{u_0})\ \ge\ \frac{\alpha\eta_r}{2}\,(u-u_0)^2 .
\end{align*}
Averaging over $i$, and using~\eqref{eq:pflog:gramform} and $G_n\succeq\lambda_{\min}(G_n)I_d$, gives~\eqref{eq:pflog:curv}.
\vspace{0.15cm}

The derivative of~\eqref{eq:pflog:renyiclosed} in $u$ is
\begin{align*}
    \frac{\partial}{\partial u}\mathcal{D}_{\alpha}(P_u\|P_{u_0})
    =\frac{\alpha}{1-\alpha}\Bigl(\psi'(u)-\psi'\bigl(\alpha u+(1-\alpha)u_0\bigr)\Bigr),
\end{align*}
which has the sign of $u-u_0$. Indeed, $\psi'$ is nondecreasing and $\alpha u+(1-\alpha)u_0$ lies between $u_0$ and $u$. Along the ray $\theta(t)=\theta_0+tv$, each natural parameter $u_i(t)=u_{0,i}+t\,z_i^{\top}v$ moves linearly away from $u_{0,i}$ on one fixed side. Thus every summand of~\eqref{eq:pflog:add} is nondecreasing in $t$. Moreover, every point outside the ball of radius $r$ lies farther along a ray from $\theta_0$ than the point where that ray meets the sphere. This gives the first identity in~\eqref{eq:pflog:sep}. On the sphere,~\eqref{eq:pflog:gramform} gives
\begin{align*}
    \frac1n\sum_{i=1}^n\bigl(z_i^{\top}(\theta-\theta_0)\bigr)^2\ge\lambda_{\min}(G_n)r^2>0,
\end{align*}
so some $i$ has $u_i\neq u_{0,i}$. For that $i$ the divergence is strictly positive, since $\psi$ is strictly convex. Finally, $\Dbar$ is continuous and the sphere is compact, so the infimum is attained and positive, and~\eqref{eq:pflog:curv} on the sphere gives the announced lower bound.
\vspace{0.15cm}

Finally, the centered second moment of $\rho_{\Sigma}$ is $\Sigma$. Thus, by cyclicity of the trace,
\begin{align*}
    \mathbb{E}_{\theta\sim\rho_{\Sigma}}\bigl[(\theta-\theta_0)^{\top}G_n(\theta-\theta_0)\bigr]={\rm Tr}(G_n\Sigma).
\end{align*}
Together with~\eqref{eq:pflog:quad}, this gives the first part of~\eqref{eq:pflog:Erho}. Note that the trace of $G_n\Sigma$ appears here, rather than an operator norm. The second part is the Gaussian-Gaussian Kullback-Leibler identity between $\mathcal{N}(\theta_0,\Sigma)$ and $\mathcal{N}(0,\tau^2I_d)$.
\end{proof}

The prior is Gaussian and the lower bound~\eqref{eq:pflog:curv} is a quadratic form. The partition function therefore requires one Gaussian integral. We record it in the matrix form that is used both here and in Appendix~\ref{app:compact}.

\begin{lem}
\label{lem:logistic:mgf}
For every symmetric $M\succeq0$, every $\tau>0$, every $d\ge1$ and every $\theta_0\in\mathbb{R}^d$,
\begin{align}
\label{eq:pflog:mgf}
    \mathbb{E}_{\theta\sim\mathcal{N}(0,\tau^2I_d)}\Bigl[e^{-\frac12(\theta-\theta_0)^{\top}M(\theta-\theta_0)}\Bigr]
    &=\det\bigl(I_d+\tau^2M\bigr)^{-1/2}\exp\Bigl(-\frac12\theta_0^{\top}M\bigl(I_d+\tau^2M\bigr)^{-1}\theta_0\Bigr)
    \nonumber\\
    &\ \le\ \det\bigl(I_d+\tau^2M\bigr)^{-1/2}.
\end{align}
\end{lem}

\begin{proof}
Let us write $A=\tau^{-2}I_d+M\succ0$. Completing the square, we have that:
\begin{align*}
    \tau^{-2}\|\theta\|^2+(\theta-\theta_0)^{\top}M(\theta-\theta_0)
    =\bigl(\theta-A^{-1}M\theta_0\bigr)^{\top}A\bigl(\theta-A^{-1}M\theta_0\bigr)+\theta_0^{\top}\bigl(M-MA^{-1}M\bigr)\theta_0.
\end{align*}
Moreover,
\begin{align*}
    M-MA^{-1}M=MA^{-1}(A-M)=\tau^{-2}MA^{-1}=M(I_d+\tau^2M)^{-1}.
\end{align*}
The expectation is therefore
\begin{align*}
    \bigl(2\pi\tau^2\bigr)^{-d/2}(2\pi)^{d/2}\det(A)^{-1/2}
    \exp\Bigl(-\frac12\theta_0^{\top}M\bigl(I_d+\tau^2M\bigr)^{-1}\theta_0\Bigr),
\end{align*}
and $\det(\tau^2A)=\det(I_d+\tau^2M)$. The exponential factor is at most $1$, because $M$ and $(I_d+\tau^2M)^{-1}$ commute and $M(I_d+\tau^2M)^{-1}\succeq0$.
\end{proof}

The next lemma is the main step in the removal of the logarithmic factor. It gives an upper bound on $\log Z^{\mathcal{D}}_\beta$, that is, a lower bound on $-\log Z^{\mathcal{D}}_\beta$, which compensates the complexity of a competitor concentrated at scale $1/n$. Note that it requires no restriction of the parameter. Indeed, the exterior of the local ball is handled by the separation~\eqref{eq:pflog:sep}.

For a prior $\pi$ on $\mathbb{R}^d$ not depending on the responses, and for $\beta\ge0$, the R\'enyi-localized prior is
\begin{align}
\label{eq:pflog:renyitilt}
    \pi^{\mathcal{D}}_{-\beta}({\rm d}\theta)=\frac{e^{-\beta\Dbar(\theta)}\pi({\rm d}\theta)}{Z^{\mathcal{D}}_\beta},
    \qquad
    Z^{\mathcal{D}}_\beta=\mathbb{E}_{\pi}\bigl[e^{-\beta\Dbar(\theta)}\bigr]\in(0,1] .
\end{align}
It is a probability measure equivalent to $\pi$, since $\Dbar$ is finite and continuous on $\mathbb{R}^d$, and it does not depend on the responses either, since $\Dbar$ is a function of $\theta$, of the design and of $\theta_0$ alone. This makes it an admissible prior in Theorem~\ref{thm:thm2}. At the Gaussian prior its normalizing constant obeys the following bound.

\begin{lem}
\label{lem:logistic:Z}
Fix the design as in~\eqref{eq:logistic:gram}, let $\pi=\mathcal{N}(0,\tau^2I_d)$ on $\mathbb{R}^d$, unrestricted, and fix $r>0$. Then, for every $\beta\ge0$,
\begin{align}
\label{eq:pflog:Ztwo}
    \log Z^{\mathcal{D}}_\beta&\ \le\ -\frac12\log\det\bigl(I_d+\alpha\eta_r\beta\tau^2G_n\bigr)+\mathcal{R}_{\beta,r},
    \nonumber\\
    \mathcal{R}_{\beta,r}&=e^{-\beta m_{r,\alpha}}\det\bigl(I_d+\alpha\eta_r\beta\tau^2G_n\bigr)^{1/2},
\end{align}
and $\mathcal{R}_{\beta,r}$ is the remainder $\mathcal{R}_{n,r}$ of~\eqref{eq:logistic:Rnr} at $\beta=n(1-\alpha)/2$.
\end{lem}

\begin{proof}
We split the expectation at $\|\theta-\theta_0\|=r$. We bound the inner region by~\eqref{eq:pflog:curv}, with the indicator dropped, and the outer region by~\eqref{eq:pflog:sep}. We then apply Lemma~\ref{lem:logistic:mgf} at $M=\alpha\eta_r\beta G_n$, with its exponential factor, which is at most $1$, dropped. We have that:
\begin{align*}
    Z^{\mathcal{D}}_\beta
    &=\mathbb{E}_{\pi}\bigl[e^{-\beta\Dbar(\theta)}\mathds{1}\left(\|\theta-\theta_0\|\le r\right)\bigr]+\mathbb{E}_{\pi}\bigl[e^{-\beta\Dbar(\theta)}\mathds{1}\left(\|\theta-\theta_0\|> r\right)\bigr]
    \\
    &\le\mathbb{E}_{\theta\sim\pi}\Bigl[e^{-\frac{\alpha\eta_r\beta}{2}(\theta-\theta_0)^{\top}G_n(\theta-\theta_0)}\Bigr]+e^{-\beta m_{r,\alpha}}
    \ \le\ A_\beta+B_\beta,
    \\
    &\qquad\qquad A_\beta=\det\bigl(I_d+\alpha\eta_r\beta\tau^2G_n\bigr)^{-1/2},
    \quad
    B_\beta=e^{-\beta m_{r,\alpha}} .
\end{align*}
Finally, $\log(1+x)\le x$ gives
\begin{align*}
    \log(A_\beta+B_\beta)=\log A_\beta+\log\Bigl(1+\frac{B_\beta}{A_\beta}\Bigr)\le\log A_\beta+\frac{B_\beta}{A_\beta},
\end{align*}
which is~\eqref{eq:pflog:Ztwo}, with remainder $B_\beta/A_\beta$.
\end{proof}

We now state the localized inequality in the R\'enyi risk.  Here the prior is tilted by the risk that already appears on the left-hand side of the PAC-Bayes inequality. The tilted term therefore moves to the left-hand side without using Assumption~\ref{asm:asm1}. This exchange dispenses with the divergence comparison and with the parameter restriction. The lemma is stated for a possibly randomized estimator and with the objective gap, which is the form used by Corollary~\ref{cor:logistic-sgd}.

\begin{lem}
\label{lem:logistic:loc}
Let $\alpha\in(0,1)$ and let the model be the well-specified one of Subsection~\ref{subsec:logistic}. Let $\pi$ be a prior on $\mathbb{R}^d$ not depending on the responses, let $\mathcal{F}\subset\mathcal{M}_1^+$, let $\rho_0\in\mathcal{F}$ be deterministic with $\mathbb{E}_{\rho_0}[\KLbar]<\infty$ and $KL(\rho_0\|\pi)<\infty$, and let $\hat\rho$ be an element of $\mathcal{F}$, possibly randomized and data-dependent, whose randomization is independent of the responses, satisfying~\eqref{eq:logistic:Delta} at $\rho_0$. Then, for any $0\le\beta<n(1-\alpha)$,
\begin{align}
\label{eq:pflog:loc}
    \bigl[n(1-\alpha)-\beta\bigr]\,\mathbb{E}\,\mathbb{E}_{\theta\sim\hat\rho}\bigl[\Dbar(\theta)\bigr]
    \ \le\ \alpha n\,\mathbb{E}_{\theta\sim\rho_0}\bigl[\KLbar(\theta)\bigr]+KL(\rho_0\|\pi)
    +\log Z^{\mathcal{D}}_\beta+\Delta .
\end{align}
\end{lem}

\begin{proof}
We start from the bound of Theorem~\ref{thm:thm2}, applied at the conditional laws $P_{u_{0,i}}$ of~\eqref{eq:pflog:add} and taken at the reference measure $\pi^{\mathcal{D}}_{-\beta}$ in place of $\pi$. This gives
\begin{align}
\label{eq:pflog:working}
    n(1-\alpha)\,\mathbb{E}\,\mathbb{E}_{\theta\sim\hat\rho}\bigl[\Dbar(\theta)\bigr]
    \ \le\ \alpha\,\mathbb{E}\,\mathbb{E}_{\theta\sim\hat\rho}\bigl[r_n(\theta,\theta_0)\bigr]+\mathbb{E}\bigl[KL\bigl(\hat\rho\|\pi^{\mathcal{D}}_{-\beta}\bigr)\bigr],
\end{align}
and we expand the last term. Note that only the independence of the responses is used. Moreover, let us write $h(\theta)=-\alpha r_n(\theta,\theta_0)+n(1-\alpha)\Dbar(\theta)$. The derivation bounds, pointwise in $\mathcal{S}$, the quantity $\int h\,\hat\rho-KL(\hat\rho\|\pi^{\mathcal{D}}_{-\beta})$ by its supremum over all probability measures, and it takes the expectation only then. This supremum does not involve $\hat\rho$. Therefore~\eqref{eq:pflog:working} still holds when $\hat\rho$ depends, besides the responses, on a randomization independent of them. In that case, $\mathbb{E}$ is the expectation over both. Finally, $\Dbar$ is finite and continuous on $\mathbb{R}^d$, since each conditional law has its parameter in $(0,1)$. Thus the tilted prior~\eqref{eq:pflog:renyitilt} is a probability measure equivalent to $\pi$, and $\hat\rho\ll\pi$ implies $\hat\rho\ll\pi^{\mathcal{D}}_{-\beta}$.
\vspace{0.15cm}

The decomposition~\eqref{eq:decomp} in the proof of Theorem~\ref{thm:thm4}, used with $\beta\Dbar$ in place of $\beta\KLbar$, gives
\begin{align*}
    KL\bigl(\rho\|\pi^{\mathcal{D}}_{-\beta}\bigr)=KL(\rho\|\pi)+\beta\,\mathbb{E}_{\theta\sim\rho}\bigl[\Dbar(\theta)\bigr]+\log Z^{\mathcal{D}}_\beta .
\end{align*}
The left-hand side of~\eqref{eq:pflog:working} is finite. Indeed, applying~\eqref{eq:pflog:working} at the reference measure $\pi$ itself, and then~\eqref{eq:logistic:Delta} and the two finiteness hypotheses on $\rho_0$, we have that:
\begin{align*}
    n(1-\alpha)\,\mathbb{E}\,\mathbb{E}_{\theta\sim\hat\rho}\bigl[\Dbar(\theta)\bigr]\le\mathbb{E}\bigl[J_{\mathcal{S}}(\hat\rho)\bigr]\le\mathbb{E}\bigl[J_{\mathcal{S}}(\rho_0)\bigr]+\Delta<\infty .
\end{align*}
The term $\beta\,\mathbb{E}\,\mathbb{E}_{\theta\sim\hat\rho}[\Dbar(\theta)]$ may therefore be moved to the left-hand side. Using the above decomposition at $\rho=\hat\rho$ and rearranging the terms, we have that:
\begin{align*}
    \bigl[n(1-\alpha)-\beta\bigr]\,\mathbb{E}\,\mathbb{E}_{\theta\sim\hat\rho}\bigl[\Dbar\bigr]
    &\ \le\ \mathbb{E}\Bigl[\alpha\,\mathbb{E}_{\hat\rho}\bigl[r_n(\theta,\theta_0)\bigr]+KL(\hat\rho\|\pi)\Bigr]+\log Z^{\mathcal{D}}_\beta
    \\
    &\ =\ \mathbb{E}\bigl[J_{\mathcal{S}}(\hat\rho)\bigr]+\log Z^{\mathcal{D}}_\beta .
\end{align*}
Next, we use the assumed optimality~\eqref{eq:logistic:Delta} of $\hat\rho$, which bounds the first term by $\mathbb{E}[J_{\mathcal{S}}(\rho_0)]+\Delta$. Moreover, for a $\rho$ that does not depend on the sample, and because the model is well-specified,
\begin{align*}
    \mathbb{E}_{\mathcal{S}}\mathbb{E}_{\theta\sim\rho}\bigl[r_n(\theta,\theta_0)\bigr]=n\,\mathbb{E}_{\theta\sim\rho}\bigl[\KLbar(\theta)\bigr],
    \qquad\text{so that}\qquad
    \mathbb{E}\bigl[J_{\mathcal{S}}(\rho)\bigr]=\alpha n\,\mathbb{E}_{\rho}\bigl[\KLbar\bigr]+KL(\rho\|\pi).
\end{align*}
At $\rho=\rho_0$ this is~\eqref{eq:pflog:loc}. Note that the left-hand coefficient is positive on the condition that $\beta<n(1-\alpha)$.
\end{proof}

Let us investigate the above bound. At $\beta=0$, and at a competitor $\rho_0$ that nearly minimizes $\rho\mapsto\mathbb{E}[J_{\mathcal{S}}(\rho)]$, we recover the bound of Corollary~\ref{cor:cor1} in the R\'enyi risk, with the objective gap added. Thus Lemma~\ref{lem:logistic:loc} generalizes it in the same way as Theorem~\ref{thm:thm4}. We keep the normalizer $\log Z^{\mathcal{D}}_\beta$ on the right-hand side, so that the two terms that cancel are displayed side by side.

\begin{proof}[Proof of Theorem~\ref{thm:logistic}]
We first evaluate the competitor term of~\eqref{eq:pflog:loc} at a Gaussian competitor. For $\Sigma\succ0$ and $\rho_{\Sigma}=\mathcal{N}(\theta_0,\Sigma)\in\mathcal{F}$, the two parts of~\eqref{eq:pflog:Erho} give
\begin{align}
\label{eq:pflog:comp}
    \alpha n\,\mathbb{E}_{\rho_{\Sigma}}\bigl[\KLbar\bigr]+KL(\rho_{\Sigma}\|\pi)
    &\ \le\ \frac12{\rm Tr}\bigl(H_n\Sigma\bigr)-\frac12\log\det\Sigma-\frac d2+\frac d2\log\tau^2+\frac{\|\theta_0\|^2}{2\tau^2},
    \nonumber\\
    H_n&=\tau^{-2}I_d+\frac{\alpha n}{4}G_n .
\end{align}
The map $\Sigma\mapsto\tfrac12{\rm Tr}(H_n\Sigma)-\tfrac12\log\det\Sigma$ is convex on the positive-definite cone, with derivative $(H_n-\Sigma^{-1})/2$. It is therefore minimized at $\Sigma=H_n^{-1}=\Sigma_n^{\star}$ of~\eqref{eq:logistic:Sigma}. At this point,
\begin{align*}
    \frac12{\rm Tr}\bigl(H_n\Sigma_n^{\star}\bigr)=\frac d2,
    \qquad
    -\log\det\Sigma_n^{\star}+d\log\tau^2=\log\det\bigl(\tau^2H_n\bigr),
\end{align*}
and the first identity cancels the term $-d/2$ of~\eqref{eq:pflog:comp}. Using $\|\theta_0\|\le R_0$, we have that:
\begin{align}
\label{eq:pflog:compstar}
    \alpha n\,\mathbb{E}_{\rho_{\star}}\bigl[\KLbar\bigr]+KL(\rho_{\star}\|\pi)
    \ \le\ \frac12\log\det\Bigl(I_d+\frac{\alpha n\tau^2}{4}G_n\Bigr)+\frac{R_0^2}{2\tau^2} .
\end{align}

\vspace{0.15cm}

We now apply Lemma~\ref{lem:logistic:loc} at the localization $\beta^{\mathcal{D}}=n(1-\alpha)/2$, which is admissible since $\beta^{\mathcal{D}}<n(1-\alpha)$. We take $\rho_0=\rho_{\star}$, we use the bound~\eqref{eq:pflog:compstar}, and we insert the partition bound~\eqref{eq:pflog:Ztwo}. Its remainder at $\beta^{\mathcal{D}}$ is $\mathcal{R}_{n,r}$. With $n(1-\alpha)-\beta^{\mathcal{D}}=n(1-\alpha)/2$, we have that:
\begin{align}
\label{eq:pflog:twodet}
    \frac{n(1-\alpha)}{2}\,\mathbb{E}\,\mathbb{E}_{\theta\sim\hat\rho}\bigl[\Dbar\bigr]
    \ \le\ \frac12\log\det\bigl(I_d+a_nG_n\bigr)-\frac12\log\det\bigl(I_d+b_nG_n\bigr)
    +\frac{R_0^2}{2\tau^2}+\mathcal{R}_{n,r}+\Delta,
\end{align}
with $a_n=\alpha n\tau^2/4$ and $b_n=\alpha\eta_r\beta^{\mathcal{D}}\tau^2$, the $b_n$ of~\eqref{eq:logistic:Rnr}.
\vspace{0.15cm}

It remains to cancel the two determinants. Let $\lambda_1,\dots,\lambda_d$ be the eigenvalues of $G_n$, which are nonnegative. The difference of the two log-determinants is
\begin{align}
\label{eq:pflog:eigsum}
    \log\det\bigl(I_d+a_nG_n\bigr)-\log\det\bigl(I_d+b_nG_n\bigr)
    =\sum_{j=1}^d\log\frac{1+a_n\lambda_j}{1+b_n\lambda_j} .
\end{align}
Moreover, for every $a,b>0$, the map $u\mapsto(1+au)/(1+bu)$ is monotone on $u\ge0$, between its value $1$ at $u=0$ and its limit $a/b$. Thus
\begin{align}
\label{eq:pflog:ratio}
    \frac{1+a_n\lambda_j}{1+b_n\lambda_j}\ \le\ \max\Bigl(1,\frac{a_n}{b_n}\Bigr),
    \qquad
    \frac{a_n}{b_n}=\frac{1}{2\eta_r(1-\alpha)}\ \ge\ 2.
\end{align}
 Hence, by~\eqref{eq:pflog:eigsum} and~\eqref{eq:pflog:ratio},
\begin{align}
\label{eq:pflog:detdiff}
    \log\det\bigl(I_d+a_nG_n\bigr)-\log\det\bigl(I_d+b_nG_n\bigr)\ \le\ d\,\log\frac{1}{2\eta_r(1-\alpha)}  .
\end{align}
Inserting~\eqref{eq:pflog:detdiff} into~\eqref{eq:pflog:twodet} and multiplying by $2/[n(1-\alpha)]$ gives~\eqref{eq:logistic:rate}, whose first term is $C_{\alpha,r}d/n$. Taking logarithms in $\mathcal{R}_{n,r}\le1/n$ gives the equivalence with~\eqref{eq:logistic:tail}.
\end{proof}

It remains to compute $\Delta$ for the projected stochastic gradient method. The estimator only appears in Lemma~\ref{lem:logistic:loc} through its optimality~\eqref{eq:logistic:Delta} against the competitor $\rho_0$. The method minimizes the objective of~\citet[Section 3]{alquier2019concentration},
\begin{align}
\label{eq:logistic:F}
    F(m,C)=-\alpha\int\sum_{i=1}^n\log\sigma\bigl(Y_i z_i^{\top}\theta\bigr)\,\Phi\bigl({\rm d}\theta;m,CC^{\top}\bigr)+KL\bigl(\Phi(\cdot;m,CC^{\top})\big\|\pi\bigr),
\end{align}
where $\Phi(\cdot;m,\Sigma)$ is the Gaussian law of mean $m$ and covariance $\Sigma$. The bounds above rather use the unnormalized objective $J_{\mathcal{S}}$ of Subsection~\ref{subsec:logistic}, and the two objectives differ by a term that does not depend on the candidate.

\begin{proof}[Proof of Corollary~\ref{cor:logistic-sgd}]
Write $\rho_{m,C}=\Phi(\cdot;m,CC^{\top})$. The log-likelihood ratio is
\begin{align*}
    r_n(\theta,\theta_0)=\sum_{i=1}^n\log\sigma\bigl(Y_iz_i^{\top}\theta_0\bigr)-\sum_{i=1}^n\log\sigma\bigl(Y_iz_i^{\top}\theta\bigr),
\end{align*}
and thus definition~\eqref{eq:logistic:F} gives
\begin{align}
\label{eq:pflog:FJ}
    J_{\mathcal{S}}\bigl(\rho_{m,C}\bigr)=F(m,C)+\alpha\sum_{i=1}^n\log\sigma\bigl(Y_iz_i^{\top}\theta_0\bigr).
\end{align}
The additive term does not depend on $(m,C)$.
\vspace{0.15cm}

The competitor is feasible. The pair $(m,C)=(\theta_0,(\Sigma_n^{\star})^{1/2})$ gives $\rho_{\star}$ of~\eqref{eq:logistic:Sigma}, and it belongs to the search set~\eqref{eq:logistic:ball}. Indeed, the covariance and the radius constraints read
\begin{align*}
    (\Sigma_n^{\star})^{1/2}\succeq\sqrt{\varsigma}\,I_d
    \iff\varsigma\le\lambda_{\min}(\Sigma_n^{\star}),
    \qquad
    \|\theta_0\|^2+{\rm Tr}(\Sigma_n^{\star})\ \le\ R_{\mathrm{sgd}}^2,
\end{align*}
and these two conditions follow from~\eqref{eq:logistic:feasible} and $\|\theta_0\|\le R_0$. The set~\eqref{eq:logistic:ball} is convex and compact. Thus the Euclidean projection onto it is well defined, and an average of feasible points is feasible. It follows that $\bar C_T\bar C_T^{\top}\succeq\varsigma I_d\succ0$, and that $\rho_T=\rho_{\bar m_T,\bar C_T}$ is a full-support Gaussian with $KL(\rho_T\|\pi)<\infty$. Below, $\|\cdot\|$ and $\langle\cdot,\cdot\rangle$ denote the Euclidean norm of~\eqref{eq:logistic:ball} and its inner product, on pairs of a vector and a symmetric matrix.
\vspace{0.15cm}

The objective $F$ is convex on the pairs $(m,C)$ with $C=C^{\top}\succ0$. Indeed, $\theta\mapsto-\log\sigma(yz^{\top}\theta)$ is convex, the map $(m,C)\mapsto m+C\xi$ is linear, and, by the Gaussian identity of~\eqref{eq:pflog:Erho}, we have that:
\begin{align*}
    KL(\rho_{m,C}\|\pi)=\frac{\|m\|^2+\|C\|_F^2}{2\tau^2}-\log\det C+\frac d2\bigl(\log\tau^2-1\bigr).
\end{align*}
This is a convex quadratic plus $-\log\det C$, which is convex on positive definite matrices.
\vspace{0.15cm}

We now check that~\eqref{eq:logistic:grad} is an unbiased estimate of the gradient of $F$. Since
\begin{align*}
    \nabla u_{\mathcal{S}}(\theta)=-\alpha\sum_{i=1}^n\bigl(1-\sigma(Y_iz_i^{\top}\theta)\bigr)Y_iz_i,
\end{align*}
the bound $\|z_i\|\le B$ of~\eqref{eq:logistic:gram} gives $\|v\|\le\alpha nB$. We may therefore differentiate under the expectation in $\xi$. In the direction of a symmetric matrix $H$, the derivative of $C\mapsto u_{\mathcal{S}}(m+C\xi)$ is $\langle v,H\xi\rangle={\rm Tr}({\rm sym}(v\xi^{\top})H)$. Moreover, the gradient of the Kullback-Leibler term is $(\tau^{-2}m,\tau^{-2}C-C^{-1})$. Thus the expectation of~\eqref{eq:logistic:grad} over $\xi$ is the gradient of $F$.
\vspace{0.15cm}

We now bound its second moment. Note that $\|{\rm sym}(v\xi^{\top})\|_F\le\|v\|\|\xi\|$ and that, on $\mathbb{B}$, $\|m\|^2+\|C\|_F^2\le R_{\mathrm{sgd}}^2$ and $\|C^{-1}\|_F^2\le d/\varsigma$. Using $\|a+b\|^2\le2\|a\|^2+2\|b\|^2$ twice, we have, at every point of $\mathbb{B}$,
\begin{align*}
    \mathbb{E}_{\xi}\bigl[\|g_m\|^2+\|g_C\|_F^2\bigr]
    &\le 2\,\mathbb{E}_{\xi}\bigl[\|v\|^2(1+\|\xi\|^2)\bigr]+\frac{4}{\tau^4}\bigl(\|m\|^2+\|C\|_F^2\bigr)+4\|C^{-1}\|_F^2
    \\
    &\le 2\alpha^2n^2B^2(1+d)+\frac{4R_{\mathrm{sgd}}^2}{\tau^4}+\frac{4d}{\varsigma}=G_{\mathrm{sgd}}^2 .
\end{align*}
\vspace{0.15cm}

Let $x^{\star}$ be a minimizer of $F$ on $\mathbb{B}$, which exists since $\mathbb{B}$ is compact and $F$ is continuous on it. The projection onto $\mathbb{B}$ is nonexpansive and $x^{\star}\in\mathbb{B}$, so that
\begin{align*}
    \|x_{t+1}-x^{\star}\|^2\ \le\ \|x_t-x^{\star}\|^2-2\eta_T\langle g_t,x_t-x^{\star}\rangle+\eta_T^2\|g_t\|^2 .
\end{align*}
Conditionally on $\xi_0,\dots,\xi_{t-1}$, the point $x_t$ is fixed, $g_t$ is an unbiased estimate of $\nabla F(x_t)$, and its second moment is at most $G_{\mathrm{sgd}}^2$. Moreover, $\langle\nabla F(x_t),x_t-x^{\star}\rangle\ge F(x_t)-F(x^{\star})$ by convexity. Taking expectations, we have that:
\begin{align*}
    2\eta_T\,\mathbb{E}_{\mathrm{alg}}\bigl[F(x_t)-F(x^{\star})\bigr]
    \ \le\ \mathbb{E}_{\mathrm{alg}}\bigl[\|x_t-x^{\star}\|^2\bigr]-\mathbb{E}_{\mathrm{alg}}\bigl[\|x_{t+1}-x^{\star}\|^2\bigr]+\eta_T^2G_{\mathrm{sgd}}^2,
\end{align*}
where $\mathbb{E}_{\mathrm{alg}}$ is the expectation over the randomization of the method at fixed data. Summing over $t=0,\dots,T-1$, the distance terms telescope. We then use $\|x_0-x^{\star}\|\le2R_{\mathrm{sgd}}$, the convexity of $F$ at the average $\bar x_T$, and the step size $\eta_T$ of~\eqref{eq:logistic:sgdconst}. We obtain
\begin{align}
\label{eq:pflog:sgdrate}
    \mathbb{E}_{\mathrm{alg}}\bigl[F(\bar x_T)\bigr]-\inf_{\mathbb{B}}F
    \ \le\ \frac{2R_{\mathrm{sgd}}^2}{\eta_T T}+\frac{\eta_T G_{\mathrm{sgd}}^2}{2}\ =\ \Delta_T .
\end{align}
\vspace{0.15cm}

Since $\inf_{\mathbb{B}}F\le F(\theta_0,(\Sigma_n^{\star})^{1/2})$ pointwise in $\mathcal{S}$, inequalities~\eqref{eq:pflog:sgdrate} and~\eqref{eq:pflog:FJ} give
\begin{align*}
    \mathbb{E}_{\mathrm{alg}}\bigl[J_{\mathcal{S}}(\rho_T)\bigr]\ \le\ J_{\mathcal{S}}(\rho_{\star})+\Delta_T,
\end{align*}
where the additive term of~\eqref{eq:pflog:FJ} cancels on the two sides. Taking $\mathbb{E}_{\mathcal{S}}$, the iterate $\rho_T$ therefore satisfies~\eqref{eq:logistic:Delta} at $\rho_0=\rho_{\star}$, with the $\Delta_T$ of~\eqref{eq:logistic:sgdconst}. This is the hypothesis of Theorem~\ref{thm:logistic}, whose conclusion therefore holds with $\Delta=\Delta_T$. Finally, we substitute $\Delta_T$ into~\eqref{eq:logistic:rate}, and, under~\eqref{eq:logistic:tail}, we bound its third term by $2/[n^2(1-\alpha)]$. This gives~\eqref{eq:logistic:sgd}, in which $\mathbb{E}$ abbreviates $\mathbb{E}_{\mathcal{S}}\mathbb{E}_{\mathrm{alg}}$.
\end{proof}
\section{Discussion}
\label{sec:discussion}

\subsection{The Dimension Assumption}
\label{subsec:catodim}

Assumption~\ref{asm:asm2} requires the localized prior $\pi_{-\beta}$ to
concentrate at the rate $\beta^{-\kappa_\pi}$, and we call the resulting
constant $d_\pi$ a dimension. Two facts justify this name. First, $d_\pi$ is additive
across independent coordinates, as a dimension should be. Second, for
one-dimensional compact models, a two-sided quadratic control on the
Kullback-Leibler divergence gives Assumption~\ref{asm:asm2} with
$\kappa_\pi=1$ and an explicit constant. These two facts also complete
Subsection~\ref{subsec:expfam}, where Assumption~\ref{asm:asm1} was
verified for Examples~\ref{exm:poisson} and~\ref{exm:bernoulli} while
Assumption~\ref{asm:asm2} was left open. Both are stated for independent
and identically distributed observations from a well-specified model.
Thus, throughout this subsection,
$\KLbar(\theta)=KL(P_{\theta_0}\|P_\theta)$ and the oracle term vanishes.
The proof of Lemma~\ref{lem:one_dim} is in Section~\ref{sec:proof}.
\begin{lem}
\label{lemmadim}
Assume that $\Theta$ satisfies Assumption~\ref{asm:asm2} uniformly, in other words, for any $\theta_0$,
\begin{align*}
        \sup_{\beta> 0}\beta^{\kappa_\pi} \mathbb{E}_{\theta\sim \pi_{-\beta}}\left[ KL(P_{\theta_0}\|P_{\theta})\right]\le d_{\pi}
    \end{align*}
where $\kappa_\pi$ and $d_\pi$ do not depend on $\theta_0$.
Define a new statistical model $\Theta_p := \Theta^p$ that parametrizes product distributions, with the product prior,
\begin{align*}
    P_\theta := P_{\theta_1}\otimes \dots \otimes P_{\theta_p},
    \qquad
    \theta=(\theta_1,\dots,\theta_p)\in\Theta_p,
    \qquad
    \pi^{(p)}=\pi^{\otimes p}.
\end{align*} Then, for any $\theta_0 \in \Theta_p$,
\begin{align*}
        \sup_{\beta> 0}\beta^{\kappa_\pi} \mathbb{E}_{\theta\sim \pi^{(p)}_{-\beta}}\left[ KL(P_{\theta_0}\|P_{\theta})\right]\le d_{\pi}\, p .
    \end{align*}
\end{lem}
Let us write $V_j(\theta_j)=KL(P_{\theta_{0,j}}\|P_{\theta_j})$. The
Kullback-Leibler divergence tensorizes over the coordinates of
$\theta_0=(\theta_{0,1},\dots,\theta_{0,p})$ and $\theta=(\theta_1,\dots,\theta_p)$ in
$\Theta_p$, and thus we have that:
\begin{align*}
  KL(P_{\theta_0}\|P_{\theta})
  = \sum_{j=1}^p V_j(\theta_j).
\end{align*}
The localized prior built on $\pi^{(p)}$ then factorizes,
\begin{align*}
  \pi^{(p)}_{-\beta}({\rm d}\theta)\ \propto\ \exp\Bigl(-\beta\sum_{j=1}^p V_j(\theta_j)\Bigr)\prod_{j=1}^p\pi({\rm d}\theta_j)
  \ =\ \prod_{j=1}^p e^{-\beta V_j(\theta_j)}\pi({\rm d}\theta_j),
\end{align*}
so that
\begin{align*}
    \pi^{(p)}_{-\beta}=\bigotimes_{j\le p}\pi^{[j]}_{-\beta},
    \qquad
    \pi^{[j]}_{-\beta}({\rm d}\theta_j)\propto e^{-\beta V_j(\theta_j)}\pi({\rm d}\theta_j).
\end{align*}
Note that these one-coordinate marginals differ from one another when the $\theta_{0,j}$ do. Taking
expectations coordinate by coordinate, we have that:
\begin{align*}
  \sup_{\beta>0}\beta^{\kappa_\pi}\mathbb{E}_{\theta\sim\pi^{(p)}_{-\beta}}\left[KL(P_{\theta_0}\|P_{\theta})\right]
  \ \le\ \sum_{j=1}^p\ \sup_{\beta>0}\beta^{\kappa_\pi}\mathbb{E}_{\theta_j\sim\pi^{[j]}_{-\beta}}\left[V_j(\theta_j)\right]
  \ \le\ p\,d_\pi,
\end{align*}
where each term is bounded by $d_\pi$ through the uniform form of
Assumption~\ref{asm:asm2}, applied in the one-coordinate model at
$\theta_{0,j}$.
\vspace{0.15cm}

We now verify Assumption~\ref{asm:asm2} for one-dimensional compact models such as Examples~\ref{exm:poisson} and~\ref{exm:bernoulli}. Lemma~\ref{lemmadim} then extends it to multidimensional models. Motivated by the discussion in the exponential family, we consider models where the KL divergence satisfies
\begin{align}
\label{hyp-dim}
   \frac{m}{2} \|\theta-\theta_0\|^2\le KL(P_{\theta_0}\|P_{\theta}) \leq \frac{L}{2} \|\theta-\theta_0\|^2
\end{align}
for some $0<m<L<+\infty$.

\begin{lem}
\label{lem:one_dim}
When the prior is uniform on $\Theta=(-M,M)$ and $\theta_0 \in (-M,M)$, under the assumption that~\eqref{hyp-dim} holds, with $\kappa = L/m$,
$$
    \sup_{\beta>0}\beta\,\mathbb{E}_{\theta\sim \pi_{-\beta}}\left[ KL(P_{\theta_0}\|P_{\theta})\right] \leq \frac{\kappa^{\frac{3}{2}}}{2},
$$
and thus Assumption~\ref{asm:asm2} is satisfied with $\kappa_\pi=1$ and $d_\pi=\kappa^{3/2}/2$.
\end{lem}
\vspace{0.15cm}

Lemma~\ref{lem_exp_1} gives the condition~\eqref{hyp-dim} for an
exponential family whose log-partition function is $m$-strongly convex with
$L$-Lipschitz gradient. Therefore, Examples~\ref{exm:poisson}
and~\ref{exm:bernoulli} satisfy Assumption~\ref{asm:asm2} on any compact
natural-parameter set, with the same condition number $\kappa=L/m$ that
already appeared in Assumption~\ref{asm:asm1} there. Moreover, by
Lemma~\ref{lemmadim}, a $p$-fold product of such models has
$\kappa_\pi=1$ and $d_\pi=p\,\kappa^{3/2}/2$. In other words, the
complexity grows linearly with the number of coordinates.

\subsection{The Effective Sample Size}
\label{subsec:ess}

Following~\citet{lhuillier2023semiparametric}, we call $n\alpha$ the effective
sample size. It already appears in the PAC-Bayes bound of
Theorem~\ref{thm:thm2}, with a fixed prior and in any model. Note that, by the
skew symmetry recalled after Assumption~\ref{asm:asm1},
$(1-\alpha)\Dbar(\theta)/\alpha$ is the average of the
$\mathcal{D}_{1-\alpha}(P_i^{\circ}\|P_{i,\theta})$, which increase to
$KL(P_i^{\circ}\|P_{i,\theta})$ as $\alpha\to0$~\citep{van2014renyi}.

\begin{prop}
\label{prop:ess-fixed}
For any prior $\pi$, any $\alpha\in(0,1)$ and any $\mathcal{F}\subset\mathcal{M}_1^+$,
\begin{align}
\label{eq:ess-fixed}
    \frac{1-\alpha}{\alpha}\,\mathbb{E}_{\mathcal{S}}\mathbb{E}_{\theta\sim\tilde{\rho}_{n,\alpha}}\bigl[\Dbar(\theta)\bigr]
    \ \le\ \inf_{\rho\in\mathcal{F}}\left[\mathbb{E}_{\theta\sim\rho}\bigl[\KLbar(\theta)\bigr]+\frac{KL(\rho\|\pi)}{n\alpha}\right].
\end{align}
At $\mathcal{F}=\mathcal{M}_1^+$, $\tilde{\rho}_{n,\alpha}=\pi_{n,\alpha}$, and
the right-hand side of~\eqref{eq:ess-fixed} is
\begin{align*}
    -\frac{1}{n\alpha}\log\int e^{-n\alpha\KLbar(\theta)}\,\pi({\rm d}\theta).
\end{align*}
In particular, if, for some $\varepsilon>0$,
\begin{align*}
    \pi\bigl(\{\theta:\KLbar(\theta)\le\varepsilon^2\}\bigr)\ge e^{-n\alpha\varepsilon^2},
\end{align*}
then
\begin{align*}
    \frac{1-\alpha}{\alpha}\,\mathbb{E}_{\mathcal{S}}\mathbb{E}_{\theta\sim\pi_{n,\alpha}}\bigl[\Dbar(\theta)\bigr]\ \le\ 2\varepsilon^2 .
\end{align*}
\end{prop}

We apply Theorem~\ref{thm:thm2} at $\hat\rho=\tilde{\rho}_{n,\alpha}$, which
minimizes over $\mathcal{F}$ the objective
\begin{align*}
    \alpha\,\mathbb{E}_{\theta\sim\rho}\bigl[r_n(\theta)\bigr]+KL(\rho\|\pi),
\end{align*}
we use $\mathbb{E}_{\mathcal{S}}[r_n(\theta)]=n\KLbar(\theta)$, and
we multiply by $(1-\alpha)/\alpha$. The value of the infimum at
$\mathcal{F}=\mathcal{M}_1^+$ is the variational formula
of~\citet{donsker1976asymptotic} recalled in Subsection~\ref{subsec:pacbayes}.
For the last statement, we take for $\rho$ the restriction of $\pi$ to
$\{\KLbar\le\varepsilon^2\}$. In the independent and identically distributed
case, and at $\mathcal{F}=\mathcal{M}_1^+$,~\eqref{eq:ess-fixed}
is~\citet[Theorem 5.1]{zhang2006klentropy} at $\rho=\alpha$, since his
divergence $D^{Re}_{\rho}(q\|p)$ is $\mathcal{D}_{1-\rho}(q\|p)/(1-\rho)$ and
his posterior $\pi_{1/\lambda}$ is $\pi_{n,\alpha}$ at $\lambda=1/\alpha$.
\vspace{0.15cm}

Proposition~\ref{prop:ess-fixed} makes no assumption on the model, and its
prior mass condition is the one of~\citet[Theorem
4.1]{lhuillier2023semiparametric}, on a Kullback-Leibler neighborhood only.
Both proofs control the numerator of the fractional posterior by the first
identity below. They differ in the normalizing constant $Z$ of $\pi_{n,\alpha}$, for
which the PAC-Bayes route uses the second bound, due
to~\citet{donsker1976asymptotic}:
\begin{align*}
    \int p_{\theta}^{\alpha}\,(p^{\circ})^{1-\alpha}\,{\rm d}\mu&=e^{-(1-\alpha)\mathcal{D}_{\alpha}(P_\theta\|P^{\circ})},
    \\
    \log Z=\log\mathbb{E}_{\theta\sim\pi}\bigl[e^{-\alpha r_n(\theta)}\bigr]&\ge-\alpha\,\mathbb{E}_{\theta\sim\rho}\bigl[r_n(\theta)\bigr]-KL(\rho\|\pi).
\end{align*}
\citet[Lemma B.1]{lhuillier2023semiparametric} bound $Z$ from below in
probability, by Chebyshev's inequality. This requires the second moment of the
log-likelihood ratio, and gives a contraction result in probability. The
expectation of the second bound only involves
$\mathbb{E}_{\theta\sim\rho}[\KLbar(\theta)]$. Thus, the PAC-Bayes route gives
a bound in expectation, for any $\hat\rho$, and in particular for the
variational approximations. Note that Theorem~\ref{thm:thm1-p} uses the same
Chebyshev step, and thus requires Assumption~\ref{asm:asm3}.
\vspace{0.15cm}

Moreover, for $\alpha\le1/2$ and for every $\alpha\in(0,1)$ respectively,
\begin{align*}
    \mathcal{D}_{1-\alpha}(P\|Q)\ \ge\ \mathcal{D}_{1/2}(P\|Q)\ \ge\ H^2(P,Q),
    \qquad
    \mathcal{D}_{1-\alpha}(P\|Q)\ \ge\ \frac{1-\alpha}{2}\,\|P-Q\|_{1}^{2},
\end{align*}
by the monotonicity in the order, the identity recalled after
Assumption~\ref{asm:asm1}, and~\citet[Theorem 31]{van2014renyi}. Combined with
the bounds on the prior mass of Kullback-Leibler neighborhoods obtained in
the proofs of~\citet[Propositions 4.4, 4.7 and
4.8]{lhuillier2023semiparametric}, Proposition~\ref{prop:ess-fixed} thus gives,
in expectation and for $\alpha$ bounded away from $1$, the rates of these
propositions. Their Theorem 4.3 also covers $\alpha_n\to1$, under a testing
condition.
\vspace{0.15cm}

At the parametric endpoint, a competitor at scale $1/(n\alpha)$ has a
complexity $KL(\rho\|\pi)$ of order $d_\pi\log(n\alpha)$, so
that~\eqref{eq:ess-fixed} is of order $d_\pi\log(n\alpha)/(n\alpha)$. The
localized prior removes this logarithm and, under Assumption~\ref{asm:asm1},
bounds the Kullback-Leibler risk itself. In Theorem~\ref{thm:thm1}, $\alpha$
only appears through $c(\alpha)$ and through the factors $\alpha$ and
$1-\alpha$. We therefore write
\begin{align*}
    \gamma(\alpha)=\alpha\, c(\alpha).
\end{align*}
We have $\gamma(\alpha)=1$ in the Gaussian models (Lemma~\ref{lem_gaussian_asm1}),
$\gamma(\alpha)=\kappa$ in the exponential families of Lemma~\ref{lem_exp_1},
and $\gamma(\alpha)=2(1+\epsilon)/(1-\epsilon)$ near $\theta_0$ in the smooth
models (Corollary~\ref{cor:fisher_divergences_remainder}).

\begin{cor}
\label{cor:ess-general}
Let $\alpha\in(0,1)$ be such that $n\alpha(1-\alpha)\ge2\gamma(\alpha)$, and
assume that the model contains a $\tilde\theta$ with
$0<\KLbar(\tilde\theta)<\infty$. Under Assumptions~\ref{asm:asm1}
and~\ref{asm:asm2},
\begin{align}
\label{eq:ess-general}
   \mathbb{E}_{\mathcal{S}}\mathbb{E}_{\theta\sim\pi_{n,\alpha}}[\KLbar(\theta)]
   \le \frac{2\gamma(\alpha)-(1-\alpha)}{1-\alpha}
   \left[\inf_{\tilde\theta\in\Theta}\KLbar(\tilde\theta)
   +d_\pi\left(\frac{2\gamma(\alpha)}{(1-\alpha)\,n\alpha}\right)^{\kappa_\pi}\right].
\end{align}
Under Assumptions~\ref{asm:asm1} and~\ref{asm:asm4}, the variational
approximation $\tilde{\rho}_{n,\alpha}$ satisfies
\begin{align*}
   \mathbb{E}_{\mathcal{S}}\mathbb{E}_{\theta\sim\tilde{\rho}_{n,\alpha}}[\KLbar(\theta)]
   \le \frac{2\gamma(\alpha)}{\alpha(1-\alpha)}
   \left[\inf_{\tilde\theta\in\Theta}\KLbar(\tilde\theta)
   +d_\pi\left(\frac{2\gamma(\alpha)}{(1-\alpha)\,n\alpha}\right)^{\kappa_\pi}\right].
\end{align*}
\end{cor}

Indeed, at $\beta^{\star}=n(1-\alpha)/(2c(\alpha))$, we have that:
\begin{align*}
    \beta^{\star}c(\alpha)=\frac{n(1-\alpha)}{2},
    \qquad
    \alpha\, c(\alpha)\, n=n\gamma(\alpha),
    \qquad
    \beta^{\star}=\frac{n\alpha(1-\alpha)}{2\gamma(\alpha)},
\end{align*}
and we substitute these identities in the two prefactors of
Theorem~\ref{thm:thm1}. The condition $n\alpha(1-\alpha)\ge2\gamma(\alpha)$
gives $\beta^{\star}\ge1$, and~\eqref{eq:ess-floor} below gives
$\beta^{\star}\le\alpha n$.
\vspace{0.15cm}

In other words, the rate is that of a sample of size
$n\alpha/\gamma(\alpha)=n/c(\alpha)$. At a fixed $\alpha$, it is the rate
$n^{-\kappa_\pi}$ of Theorem~\ref{thm:thm1}. Along a sequence $\alpha_n\to0$
with $n\alpha_n\to\infty$, and with $\gamma$ bounded, it is
$(n\alpha_n)^{-\kappa_\pi}$. As $\alpha\to1$, the bound deteriorates like
$(1-\alpha)^{-(1+\kappa_\pi)}$.
\vspace{0.15cm}

No model gives a larger effective sample size. Indeed, the skew symmetry and
the monotonicity recalled after Assumption~\ref{asm:asm1} give
$(1-\alpha)\Dbar(\theta)\le\alpha\KLbar(\theta)$. Together with
Assumption~\ref{asm:asm1}, we have that:
\begin{align*}
    \KLbar(\theta)\ \le\ c(\alpha)\,\Dbar(\theta)\ \le\ \frac{\gamma(\alpha)}{1-\alpha}\,\KLbar(\theta).
\end{align*}
At a $\tilde\theta$ with $0<\KLbar(\tilde\theta)<\infty$, this gives
\begin{align}
\label{eq:ess-floor}
   \gamma(\alpha)\ \ge\ 1-\alpha ,
\end{align}
and thus $n/c(\alpha)\le n\alpha/(1-\alpha)$. Moreover,~\eqref{eq:ess-floor}
gives $\gamma(\alpha)\le2\gamma(\alpha)-(1-\alpha)\le2\gamma(\alpha)$. In the
well-specified case, the constant in front of $d_\pi(n\alpha)^{-\kappa_\pi}$
in~\eqref{eq:ess-general} is therefore between
\begin{align*}
    2^{\kappa_\pi}\Bigl(\frac{\gamma(\alpha)}{1-\alpha}\Bigr)^{1+\kappa_\pi}
    \qquad\text{and}\qquad
    2^{1+\kappa_\pi}\Bigl(\frac{\gamma(\alpha)}{1-\alpha}\Bigr)^{1+\kappa_\pi},
\end{align*}
and, for $\alpha$ bounded away from $1$, it is of order one if and only if
$\gamma$ is bounded.
\vspace{0.15cm}

For instance, for a $p$-fold product of the models of
Examples~\ref{exm:poisson} or~\ref{exm:bernoulli} on a compact interval of
natural parameters, with the uniform prior, $\gamma=\kappa$ and
Subsection~\ref{subsec:catodim} gives $(\kappa_\pi,d_\pi)=(1,p\kappa^{3/2}/2)$.
Thus, as soon as $n\alpha(1-\alpha)\ge2\kappa$,
\begin{align*}
    \mathbb{E}_{\mathcal{S}}\mathbb{E}_{\theta\sim\pi_{n,\alpha}}\bigl[KL(P_{\theta_0}\|P_{\theta})\bigr]
    \ \le\ \frac{\bigl(2\kappa-(1-\alpha)\bigr)\,p\,\kappa^{5/2}}{(1-\alpha)^2\,n\alpha}.
\end{align*}
This is a parametric rate in the effective sample size, without logarithmic
factor, in models that are not Gaussian.
\vspace{0.15cm}

In the Gaussian sequence model of Subsection~\ref{subsec:gaussian},
$\gamma\equiv1$ and~\eqref{eq:ess-general} is
Corollary~\ref{cor:seq:combined}. Along any sequence $\alpha_n\to0$ with
$n\alpha_n\to\infty$, it gives, with $\kappa=2b/(2b+1)$,
\begin{align*}
    \mathbb{E}_{\mathcal{S}}\mathbb{E}_{\theta\sim\pi_{n,\alpha_n}}\left[KL(P_{\theta_0}\|P_{\theta})\right]
    \le 2^{\kappa}\left(C_b L^{\frac{1}{2b+1}}+\frac{1}{2}\right)(n\alpha_n)^{-\kappa}\bigl(1+o(1)\bigr),
\end{align*}
and~\eqref{eq:seq:lower} shows that this order is attained. Up to the scale
$\tau$, the prior~\eqref{eq:seq:prior} is the series prior of Example 1
of~\citet{lhuillier2023semiparametric} with regularity $b$. For this prior and
a truth of regularity $b$, their Corollary 4.9(i) and Proposition 4.7 give the
$\ell_2$ contraction rate $(n\alpha_n)^{-b/(2b+1)}$, which is the square root
of the order above. The two results are of a different nature. Ours is a
PAC-Bayes bound, non-asymptotic, in expectation and in the Kullback-Leibler
risk, with its dependence on the radius $L$. Theirs is a contraction result in
probability, which holds uniformly in $\alpha_n\in(0,1]$ under a testing
condition (their Theorem 4.3).
\vspace{0.15cm}

The prior mass condition at $e^{-n\varepsilon_n^2}$
of~\citet[Theorem 3.1]{bhattacharya2016bayesian} gives a worse dependence on
$\alpha$, as~\citet{lhuillier2023semiparametric} observe after their
Theorem 4.1. Indeed, through $c(\alpha)=1/\alpha$, the bound
of~\citet[Corollary 3.2]{bhattacharya2016bayesian} gives a Kullback-Leibler
risk of order
$\varepsilon_n^2/[\alpha(1-\alpha)]$, where $\varepsilon_n$ does not depend on
$\alpha$. When $\varepsilon_n^2$ is of order $n^{-\kappa_\pi}$, this exceeds
$(n\alpha)^{-\kappa_\pi}$ by the factor $\alpha^{-(1-\kappa_\pi)}$. As
$\alpha\to1$, however, their constant grows like $(1-\alpha)^{-1}$, and ours
like $(1-\alpha)^{-(1+\kappa_\pi)}$.

\begin{rem}
\label{rem:alpha-one}
All the bounds of this paper degenerate as $\alpha\to1$. Indeed, they are
derived from Theorem~\ref{thm:thm2}, whose proof rests on the identity
\begin{align*}
    \mathbb{E}_{\mathcal{S}}\bigl[e^{-\alpha r_n(\theta)}\bigr]=e^{-n(1-\alpha)\Dbar(\theta)}.
\end{align*}
At $\alpha=1$, we only have $\mathbb{E}_{\mathcal{S}}[e^{-r_n(\theta)}]\le1$ for
every $\theta$, which gives no control of the risk. Of course, the posterior
does contract at $\alpha=1$, under testing conditions.~\citet[Theorem
2.1]{ghosal2000convergence} ask for a sieve of small metric entropy, whose
complement has exponentially small prior mass, and for prior mass on a
Kullback-Leibler neighborhood of the truth. Under tests on such a sieve, and
prior mass on a R\'enyi neighborhood of order larger than one (see
Remark~\ref{rem:log-delta}),~\citet[Theorem 2.1]{zhang2020convergence} obtain the rate of the
variational approximation of the posterior at $\alpha=1$. With a rescaled
prior, they also obtain the minimax rate without logarithmic factor in the Gaussian sequence
model. The PAC-Bayes bounds of this paper use no testing or entropy condition,
and they do not cover $\alpha=1$.
\end{rem}

\subsection{Logistic Regression}
\label{subsec:logistic-discussion}

For the model, the prior and the family of
Subsection~\ref{subsec:logistic},~\citet[Corollary 3.3]{alquier2019concentration}
obtain a bound of order $d\log(nd)/n$ in the same R\'enyi risk, and their
Corollary 3.4 adds the optimization error of the finite stochastic-gradient
iterate, as~\eqref{eq:logistic:sgd} does. Thus, the localization replaces their
statistical term by $d/n$. Note however that their logistic theorem assumes an
independent and identically distributed random design with moment conditions,
and holds with large probability for the exact variational approximation.
Theorem~\ref{thm:logistic} is conditional on a fixed design and holds
in expectation. Rates without logarithmic factor are classical for the exact
posterior in parametric models, see~\citet[Section 5]{ghosal2000convergence}
and~\citet[Section 10.2.2]{ghosal2017fundamentals}, and~\citet{lhuillier2023semiparametric}
obtain such rates for fractional posteriors with some Gaussian process priors.
These results concern the posterior itself rather than its variational
approximation.
\vspace{0.15cm}

Let us now assume that the design is random, with independent
$Z_i$ such that $\|Z_i\|\le B$ almost surely and, for some $\lambda_0>0$,
\begin{align*}
    \frac{1}{n}\sum_{i=1}^n\mathbb{E}\bigl[Z_iZ_i^{\top}\bigr]\ \succeq\ 2\lambda_0I_d .
\end{align*}
The constants $\eta_r$ and $C_{\alpha,r}$ of Theorem~\ref{thm:logistic} do not depend on
the design, and, on the event $\mathcal{E}=\{\lambda_{\min}(G_n)\ge\lambda_0\}$,~\eqref{eq:logistic:constants} gives
\begin{align*}
    m_{r,\alpha}\ \ge\ \frac{\alpha\,\eta_r\,\lambda_0\,r^2}{2},
    \qquad
    \det\bigl(I_d+b_nG_n\bigr)\ \le\ \bigl(1+b_nB^2\bigr)^{d}.
\end{align*}
On $\mathcal{E}^c$, the R\'enyi risk remains bounded. Indeed, for
$\alpha\in(0,1)$ and two distributions $P$, $Q$ on $\{-1,+1\}$, we have
$\sum_yP(\{y\})^{\alpha}\ge1$, and thus
\begin{align*}
    \mathcal{D}_\alpha(P\|Q)\ \le\ \log\frac{1}{\min_{y}Q(\{y\})},
    \qquad
    \Dbar(\theta)\ \le\ \log\bigl(1+e^{BR_0}\bigr)
    \quad\text{for every }\theta\in\mathbb{R}^d .
\end{align*}
Integrating Theorem~\ref{thm:logistic} over the design, for a $\Delta$ that does not depend on it, we obtain
\begin{align*}
    \mathbb{E}\,\mathbb{E}_{\theta\sim\hat\rho}\bigl[\Dbar(\theta)\bigr]
    &\le C_{\alpha,r}\frac dn+\frac{R_0^2}{n(1-\alpha)\tau^2}+\frac{2(\bar{\mathcal{R}}_{n,r}+\Delta)}{n(1-\alpha)}+\log\bigl(1+e^{BR_0}\bigr)\,\mathbb{P}(\mathcal{E}^c),
    \\
    \bar{\mathcal{R}}_{n,r}&=e^{-n(1-\alpha)\alpha\eta_r\lambda_0r^2/4}\bigl(1+b_nB^2\bigr)^{d/2},
\end{align*}
where $\mathbb{E}$ now also integrates over the design, and $\bar{\mathcal{R}}_{n,r}$ bounds $\mathcal{R}_{n,r}$ on $\mathcal{E}$. Finally, the matrix Bernstein
inequality~\citep[Theorem 1.4]{tropp2012user}, applied to the independent
matrices $n^{-1}(\mathbb{E}[Z_iZ_i^{\top}]-Z_iZ_i^{\top})$, gives
\begin{align*}
    \mathbb{P}\bigl(\mathcal{E}^c\bigr)\ \le\ d\exp\left(-\frac{n\lambda_0^2}{2B^4+2B^2\lambda_0/3}\right).
\end{align*}
\vspace{0.15cm}

Theorem~\ref{thm:logistic} is stated in the R\'enyi risk because
$\KLbar/\Dbar$ is unbounded on $\mathbb{R}^d$ (Appendix~\ref{app:klrenyi}). On
a ball, a truncated Gaussian competitor aligned with $G_n$ gives a
non-conjugate example of Assumption~\ref{asm:asm4}, with $\kappa_\pi=1$ and
$d_\pi=\mathcal{O}(d)$, and the rate $d/n$ in the Kullback-Leibler risk for the
exact variational approximation (Appendix~\ref{app:compact}). Whether the output
of a stochastic gradient method attains this rate remains, to our knowledge,
an open question. By the remark that closes Appendix~\ref{app:compact}, it
would suffice that a projected method on the truncated family $\mathcal{F}^R$
reaches the normalized objective within $\mathcal{O}(d/n)$ of its infimum.
\vspace{0.15cm}

The constant $C_{\alpha,r}$ only depends on $\alpha$ and $\eta_r$, because the
covariance of the competitor~\eqref{eq:logistic:Sigma} is aligned with $G_n$.
An isotropic competitor $\mathcal{N}(\theta_0,sI_d)$ would leave the condition
number ${\rm Tr}(G_n)/[d\lambda_{\min}(G_n)]$ of the design in the constant.
Note also that the radius $r$ is free. A larger $r$ gives a larger separation
$m_{r,\alpha}$, which only appears in $\mathcal{R}_{n,r}$, and a smaller
curvature $\eta_r$, which appears in $C_{\alpha,r}$.

\appendix

\section{The Compact Case: Assumption~\ref{asm:asm4} and the Kullback-Leibler Risk}
\label{app:compact}

In this appendix, we prove a Kullback-Leibler counterpart of
Theorem~\ref{thm:logistic} on a compact parameter set, for the exact
variational approximation. It also gives an example
of Assumption~\ref{asm:asm4} in a model where the localized prior is not
Gaussian. In such a model, the assumption does not reduce to
Assumption~\ref{asm:asm2}, as it does in the conjugate examples of
Subsection~\ref{subsec:gaussian}, and it has to be checked on a competitor
that we exhibit by hand.
\vspace{0.15cm}

The model is that of~\eqref{eq:logistic:model} with the design
of~\eqref{eq:logistic:gram}, restricted to a ball, with the prior restricted to
it and a variational family of truncated full-covariance Gaussians:
\begin{align}
\label{eq:app:setting}
    \Theta_R=\bigl\{\|\theta\|\le R\bigr\},
    \quad
    \|\theta_0\|\le R_0<R,
    \quad
    \delta_R:=R-R_0,
    \quad
    \pi_R=\mathcal{N}(0,\tau^2I_d)(\,\cdot\mid\Theta_R),
\end{align}
\begin{align}
\label{eq:app:family}
    \mathcal{F}^R&=\bigl\{\mathcal{N}(m,\Sigma)(\,\cdot\mid\Theta_R):m\in\mathbb{R}^d,\ \Sigma\succ0\bigr\},
    \nonumber\\
    \eta_R&=\sigma(BR)\bigl(1-\sigma(BR)\bigr)\in(0,1/4],
    \qquad
    p_0=\tfrac12 .
\end{align}
On $\Theta_R$, the natural parameter satisfies $|z_i^{\top}\theta|\le BR$, and
thus the curvature of the log-partition is bounded on both sides there. Note
that $\eta_R$ is the curvature floor over the whole ball $\Theta_R$, whereas the
$\eta_r$ of~\eqref{eq:logistic:constants} is the floor over the local ball of
radius $r$ around $\theta_0$; the two coincide when $r=R-R_0$ and
$\|\theta_0\|=R_0$.
Taylor's theorem applied to~\eqref{eq:pflog:bregman}, as in the proof of
Lemma~\ref{lem:logistic:quad}, then gives the two-sided quadratic comparison
\begin{align}
\label{eq:app:sandwich}
    \frac{\eta_R}{2}(\theta-\theta_0)^{\top}G_n(\theta-\theta_0)
    \ \le\ \KLbar(\theta)\ \le\
    \frac18(\theta-\theta_0)^{\top}G_n(\theta-\theta_0),
    \qquad\theta\in\Theta_R,
\end{align}
and, by Lemma~\ref{lem_exp_1} at the condition number $\kappa=1/(4\eta_R)$ of
$\psi$ on $[-BR,BR]$, it gives Assumption~\ref{asm:asm1} on $\Theta_R$ with
\begin{align}
\label{eq:app:calpha}
    c(\alpha)=\frac{1}{4\eta_R\alpha},
\end{align}
a constant that involves the design only through the uniform bound $B$.

We read Assumption~\ref{asm:asm4} at the prior $\pi_R$. The localized
prior of Assumption~\ref{asm:asm2} and its normalizing constant are then
\begin{align*}
    \pi_{R,-\beta}({\rm d}\theta)\propto e^{-\beta\KLbar(\theta)}\pi_R({\rm d}\theta),
    \qquad
    Z_\beta=\mathbb{E}_{\theta\sim\pi_R}\bigl[e^{-\beta\KLbar(\theta)}\bigr].
\end{align*}

\begin{prop}
\label{prop:app:asm4}
Let the setting be~\eqref{eq:app:setting}--\eqref{eq:app:family}, let
$\alpha\in(0,1)$ and let
\begin{align}
\label{eq:app:comp}
    \Sigma_\beta=\frac{4p_0}{n+\beta}\,G_n^{-1},
    \qquad
    Q_\beta=\mathcal{N}(\theta_0,\Sigma_\beta),
    \qquad
    \rho_\beta=Q_\beta(\,\cdot\mid\Theta_R)\in\mathcal{F}^R .
\end{align}
If
\begin{align}
\label{eq:app:n0}
    n\ \ge\ \frac{4p_0\,{\rm Tr}(G_n^{-1})}{(1-p_0)\,\delta_R^2}
    \ =\ \frac{4\,{\rm Tr}(G_n^{-1})}{\delta_R^2},
\end{align}
then $Q_\beta(\Theta_R)\ge p_0$ for every $\beta>0$ and
Assumption~\ref{asm:asm4} holds on $(0,\alpha n]$ with $\kappa_\pi=1$ and
\begin{align}
\label{eq:app:Dpi}
    d_\pi=D_{\mathrm{log}}(\alpha),
    \qquad
    D_{\mathrm{log}}(x)=\frac d2\Bigl[1+x\Bigl(1+\log_{+}\frac{1}{4p_0\eta_R}\Bigr)\Bigr]+x\,C_R,
    \qquad
    C_R=\frac{R^2}{2\tau^2}+\log\frac{1}{p_0},
\end{align}
where $\log_{+}(u)=\max(\log u,0)$. In particular $d_\pi=\mathcal{O}(d)$ as
soon as $\eta_R^{-1}$ is of order one and $R^2/\tau^2=\mathcal{O}(d)$.
\end{prop}

\begin{proof}
If $X\sim Q_\beta$, then
\begin{align*}
    \mathbb{E}\bigl[\|X-\theta_0\|^2\bigr]={\rm Tr}(\Sigma_\beta)=\frac{4p_0{\rm Tr}(G_n^{-1})}{n+\beta}.
\end{align*}
Note that $\|X-\theta_0\|<\delta_R$ implies $X\in\Theta_R$, since
$\|\theta_0\|\le R_0$. Markov's inequality and~\eqref{eq:app:n0} therefore
give
\begin{align*}
    p_\beta:=Q_\beta(\Theta_R)\ \ge\ 1-\frac{4p_0{\rm Tr}(G_n^{-1})}{(n+\beta)\delta_R^2}
    \ \ge\ 1-\frac{4p_0{\rm Tr}(G_n^{-1})}{n\,\delta_R^2}\ \ge\ p_0 .
\end{align*}
\vspace{0.15cm}

We now check the two ingredients of Assumption~\ref{asm:asm4} at
$\rho_\beta$. The upper half of~\eqref{eq:app:sandwich},
$\rho_\beta\le p_\beta^{-1}Q_\beta$ and
${\rm Tr}(G_n\Sigma_\beta)=4p_0d/(n+\beta)$ give
\begin{align}
\label{eq:app:risk}
    \mathbb{E}_{\rho_\beta}\bigl[\KLbar(\theta)\bigr]
    \ \le\ \frac{1}{8p_\beta}\,{\rm Tr}(G_n\Sigma_\beta)
    \ \le\ \frac{d}{2(n+\beta)} ,
\end{align}
all the eigenvalues of $G_n$ contributing equally. Let us write $q_\beta$ for the
density of $Q_\beta$ and $\varphi_\tau$ for that of
$\Pi_\tau=\mathcal{N}(0,\tau^2I_d)$. On $\Theta_R$, we have that:
\begin{align*}
    \log\frac{q_\beta(\theta)}{\varphi_\tau(\theta)}
    =\frac12\log\det\bigl(\tau^2\Sigma_\beta^{-1}\bigr)+\frac{\|\theta\|^2}{2\tau^2}
    -\frac12(\theta-\theta_0)^{\top}\Sigma_\beta^{-1}(\theta-\theta_0),
\end{align*}
and thus, dropping the last term, which is nonpositive, and using
$\|\theta\|\le R$, we get
\begin{align}
\label{eq:app:KL}
    KL(\rho_\beta\|\pi_R)\ \le\ \frac12\log\det\bigl(\tau^2\Sigma_\beta^{-1}\bigr)
    +\frac{R^2}{2\tau^2}+\log\frac{1}{p_\beta}+\log\Pi_\tau(\Theta_R).
\end{align}
\vspace{0.15cm}

The normalizer gives the cancellation. Indeed, by the lower half
of~\eqref{eq:app:sandwich} and Lemma~\ref{lem:logistic:mgf} at
$M=\eta_R\beta G_n$, with the exponential factor of~\eqref{eq:pflog:mgf}
dropped, we have that:
\begin{align}
\label{eq:app:Z}
    Z_\beta=\mathbb{E}_{\pi_R}\bigl[e^{-\beta\KLbar(\theta)}\bigr]
    &\ \le\ \frac{1}{\Pi_\tau(\Theta_R)}\int_{\mathbb{R}^d}e^{-\frac{\eta_R\beta}{2}(\theta-\theta_0)^{\top}G_n(\theta-\theta_0)}\Pi_\tau({\rm d}\theta)
    \nonumber\\
    &\ \le\ \frac{\det\bigl(I_d+\eta_R\beta\tau^2G_n\bigr)^{-1/2}}{\Pi_\tau(\Theta_R)} .
\end{align}
The tilt identity~\eqref{eq:decomp} at the prior $\pi_R$ states that
\begin{align*}
    KL(\rho\|\pi_{R,-\beta})=KL(\rho\|\pi_R)+\beta\,\mathbb{E}_{\theta\sim\rho}\bigl[\KLbar(\theta)\bigr]+\log Z_\beta .
\end{align*}
Therefore, the quantity bounded in Assumption~\ref{asm:asm4} is
\begin{align*}
    \mathcal{L}_\beta(\rho):=\mathbb{E}_{\rho}\bigl[\KLbar\bigr]+\frac{KL(\rho\|\pi_{R,-\beta})}{n}
    =\Bigl(1+\frac\beta n\Bigr)\mathbb{E}_{\rho}\bigl[\KLbar\bigr]+\frac{KL(\rho\|\pi_R)}{n}+\frac{\log Z_\beta}{n} .
\end{align*}
We insert~\eqref{eq:app:risk},~\eqref{eq:app:KL} and~\eqref{eq:app:Z} at
$\rho=\rho_\beta$. The term $+\log\Pi_\tau(\Theta_R)$ in
$KL(\rho_\beta\|\pi_R)$ and the term $-\log\Pi_\tau(\Theta_R)$ in
$\log Z_\beta$ cancel. Moreover, $(1+\beta/n)/(n+\beta)=1/n$ and
$p_\beta\ge p_0$, so that
\begin{align*}
    \mathcal{L}_\beta(\rho_\beta)\ \le\ \frac{d}{2n}+\frac{C_R}{n}
    +\frac{1}{2n}\log\det\Bigl[\tau^2\Sigma_\beta^{-1}\bigl(I_d+\eta_R\beta\tau^2G_n\bigr)^{-1}\Bigr].
\end{align*}
No prior mass remains. It has been replaced by the truncation mass of the
competitor, through $C_R$ of~\eqref{eq:app:Dpi}.
\vspace{0.15cm}

By~\eqref{eq:app:comp}, $\tau^2\Sigma_\beta^{-1}=\tau^2(n+\beta)G_n/(4p_0)$,
and the two determinants cancel eigenvalue by eigenvalue. For every
eigenvalue $\lambda_j$ of $G_n$,
\begin{align*}
    \frac{\tau^2(n+\beta)\lambda_j/(4p_0)}{1+\eta_R\beta\tau^2\lambda_j}
    \ \le\ \frac{n+\beta}{4p_0\eta_R\beta},
\end{align*}
by $\tau^2\lambda_j/(1+\eta_R\beta\tau^2\lambda_j)\le1/(\eta_R\beta)$.  Replacing a
possibly negative logarithm by its positive part, we get
\begin{align}
\label{eq:app:L8}
    \mathcal{L}_\beta(\rho_\beta)\ \le\ \frac{d}{2n}\Bigl[1+\log_{+}\Bigl(\frac{n+\beta}{4p_0\eta_R\beta}\Bigr)\Bigr]+\frac{C_R}{n}.
\end{align}
\vspace{0.15cm}

We now put~\eqref{eq:app:L8} in the form $d_\pi/\beta^{\kappa_\pi}$, on the
range of localizations that Theorem~\ref{thm:thm1} evaluates. Let us put
$x=\beta/n$ and multiply~\eqref{eq:app:L8} by $\beta=nx$:
\begin{align*}
    \beta\,\mathcal{L}_\beta(\rho_\beta)\ \le\ \frac d2\,x\Bigl[1+\log_{+}\Bigl(\frac{1+x}{4p_0\eta_R x}\Bigr)\Bigr]+x\,C_R .
\end{align*}
Since $(1+x)/x\ge1$, the inequalities $\log_{+}(ab)\le\log a+\log_{+}b$ for
$a\ge1$ and $x\log(1+1/x)\le1$ give
\begin{align*}
    \beta\,\mathcal{L}_\beta(\rho_\beta)\ \le\ \frac d2\Bigl[1+x\Bigl(1+\log_{+}\frac{1}{4p_0\eta_R}\Bigr)\Bigr]+x\,C_R
    \ =\ D_{\mathrm{log}}(x),
\end{align*}
where $D_{\mathrm{log}}$ is nondecreasing. Thus
$\beta\mathcal{L}_\beta(\rho_\beta)\le D_{\mathrm{log}}(\alpha)$ for every
$\beta\in(0,\alpha n]$. Since the example is well-specified,
$\inf_{\tilde\theta\in\Theta_R}\KLbar(\tilde\theta)=0$.
This proves the assumption with $\kappa_\pi=1$ and
$d_\pi=D_{\mathrm{log}}(\alpha)$.
\end{proof}

With $p_0=1/2$, the threshold~\eqref{eq:app:n0} becomes
$n\ge4{\rm Tr}(G_n^{-1})/\delta_R^2$, and $4p_0\eta_R=2\eta_R$.
Assumption~\ref{asm:asm1} with the constant~\eqref{eq:app:calpha} and
Assumption~\ref{asm:asm4} with~\eqref{eq:app:Dpi} are the two inputs of
Theorem~\ref{thm:thm1}. The theorem then gives the parametric rate in the
Kullback-Leibler risk.

\begin{cor}
\label{cor:app:kl}
Under the hypotheses of Proposition~\ref{prop:app:asm4}, the minimizer
$\tilde{\rho}_{n,\alpha}$ of the variational objective over $\mathcal{F}^R$,
against the prior $\pi_R$, satisfies
\begin{align}
\label{eq:app:rate}
    \mathbb{E}_{\mathcal{S}}\mathbb{E}_{\theta\sim\tilde{\rho}_{n,\alpha}}\bigl[\KLbar(\theta)\bigr]
    \ \le\ \Bigl(\frac{2c(\alpha)}{1-\alpha}\Bigr)^{2}\frac{D_{\mathrm{log}}(\alpha)}{n}
    \ =\ C_{\mathrm{s}}\,\frac dn,
    \qquad
    C_{\mathrm{s}}=\Bigl(\frac{1}{2\eta_R\alpha(1-\alpha)}\Bigr)^{2}\frac{D_{\mathrm{log}}(\alpha)}{d},
\end{align}
with $c(\alpha)=1/(4\eta_R\alpha)$ of~\eqref{eq:app:calpha}. The constant
$C_{\mathrm{s}}$ is of order one under the conditions that make
$D_{\mathrm{log}}(\alpha)$ of order $d$, and thus this bound is
$\mathcal{O}(d/n)$ as well.
\end{cor}

\begin{proof}
We apply Theorem~\ref{thm:thm1} in its variational form, at
$\beta=\beta^{\star}=n(1-\alpha)/(2c(\alpha))$, with $c(\alpha)$ given
by~\eqref{eq:app:calpha} and
$(\kappa_\pi,d_\pi)=(1,D_{\mathrm{log}}(\alpha))$ given by
Proposition~\ref{prop:app:asm4}. The oracle term vanishes, and we obtain the first expression
in~\eqref{eq:app:rate}. Inserting
$c(\alpha)=1/(4\eta_R\alpha)$ gives the displayed $C_{\mathrm{s}}$. Note
that the localization $\beta^{\star}$ lies in $(0,\alpha n]$, where
Proposition~\ref{prop:app:asm4} verifies the assumption. Indeed,
$\eta_R\le1/4$ gives
$c(\alpha)=1/(4\eta_R\alpha)\ge1/\alpha\ge(1-\alpha)/\alpha$.
\end{proof}

The same argument applies to an approximate minimizer $\hat\rho\in\mathcal{F}^R$, whose value of the normalized objective of Section~\ref{section:prelim} exceeds its infimum over $\mathcal{F}^R$ by $\varepsilon'$. It satisfies the relaxed optimality at the single point where the exact one is used, and dividing by the coefficient $1/[2c(\alpha)]$ at $\beta^{\star}$ gives
\begin{align*}
    \mathbb{E}_{\mathcal{S}}\mathbb{E}_{\theta\sim\hat\rho}\bigl[\KLbar(\theta)\bigr]\le C_{\mathrm{s}}\,\frac dn+2c(\alpha)\,\varepsilon' .
\end{align*}
The gap $\Delta$ of Theorem~\ref{thm:logistic} is measured in the unnormalized objective $J_{\mathcal{S}}$ instead, and it therefore appears there divided by $n(1-\alpha)$.

\section{The Kullback-Leibler Risk and the Support of the Prior}
\label{app:klrenyi}

This appendix explains why the Kullback-Leibler statement of
Appendix~\ref{app:compact} restricts the support of the prior, while
Theorem~\ref{thm:logistic} does not. The restriction comes from
Assumption~\ref{asm:asm1} rather than from the localization. The ratio of
the two risks is unbounded on $\mathbb{R}^d$. Let us fix a
design point $z$ with $z^{\top}\theta_0=u_0$, write $q_u=\sigma(u)$ for the
probability of the value $+1$ at the natural parameter $u$, and let
$u=z^{\top}\theta\to\infty$. The two divergences then separate,
\begin{align}
\label{eq:app:sep}
   KL(P_{u_0}\|P_u)&=(1-\sigma(u_0))\,u\,(1+o(1))\;\longrightarrow\;\infty,
   \qquad
   \mathcal{D}_\alpha(P_u\|P_{u_0})\;\longrightarrow\;\log\frac{1}{\sigma(u_0)} ,
\end{align}
the first limit following from
\begin{align*}
    KL(P_{u_0}\|P_u)=q_{u_0}\log\frac{q_{u_0}}{q_u}+(1-q_{u_0})\log\frac{1-q_{u_0}}{1-q_u}
\end{align*}
with $1-q_u=\sigma(-u)\sim e^{-u}$. The second follows from the closed
form~\eqref{eq:pflog:renyiclosed}, whose numerator tends to
$(1-\alpha)(\psi(u_0)-u_0)$ as $u\to\infty$. Let us now consider a ray
$\theta_t=\theta_0+tv$, with $v$ a unit vector. Since $G_n\succ0$, we have
$v^{\top}G_nv=n^{-1}\sum_i(z_i^{\top}v)^2>0$. Therefore some
$z_i^{\top}v\neq0$, and the corresponding natural parameter
$z_i^{\top}\theta_t$ tends to $+\infty$ or $-\infty$. We use~\eqref{eq:app:sep} in the first case, and its mirror image at
$u\to-\infty$ in the second, where
$KL(P_{u_0}\|P_u)=\sigma(u_0)|u|+\mathcal{O}(1)$ and
$\mathcal{D}_\alpha(P_u\|P_{u_0})\to\psi(u_0)$. In both cases we obtain $\KLbar(\theta_t)\to\infty$. Moreover,
$\Dbar(\theta_t)$ stays bounded. Indeed, by Lemma~\ref{lem:logistic:quad}, each
summand of $\Dbar(\theta_t)$ is nondecreasing in $t$, and by the two
limits above each has a finite limit, so each is bounded on $[0,\infty)$.
Finally, $\Dbar(\theta_t)>0$ for $t>0$, so the ratio is well defined. It
follows that
\begin{align*}
    \sup_{\theta\in\mathbb{R}^d}\frac{\KLbar(\theta)}{\Dbar(\theta)}=\infty .
\end{align*}
Therefore, no finite $c(\alpha)$ satisfies Assumption~\ref{asm:asm1} on the
unrestricted parameter space. For this reason,
Examples~\ref{exm:poisson} and~\ref{exm:bernoulli} restrict the natural
parameter to a compact set. The Poisson case separates in the same way as
its mean parameter tends to $0$. On $\Theta_R$, the restriction restores the comparison,
and~\eqref{eq:app:calpha} is finite through the curvature floor $\eta_R>0$.
\vspace{0.15cm}

The prior and the variational family must moreover be restricted together. If
$\pi_R(\Theta_R)=1$ and $\rho$ is a Gaussian with full support on
$\mathbb{R}^d$, then $\rho$ charges
$\mathbb{R}^d\setminus\Theta_R$, where $\pi_R$ puts no mass. Therefore
$KL(\rho\|\pi_R)=\infty$, and $\rho$ is not a competitor at all. A compact
Kullback-Leibler analysis therefore needs compatible supports,
\begin{align*}
    {\rm supp}(\rho)\subseteq{\rm supp}(\pi_R)\subseteq\Theta_R,
\end{align*}
and Appendix~\ref{app:compact} accordingly uses both the restricted prior $\pi_R$
and the truncated Gaussian family $\mathcal{F}^R$, and why its competitor is
the truncation $\rho_\beta=Q_\beta(\,\cdot\mid\Theta_R)$.
\vspace{0.15cm}

Theorem~\ref{thm:logistic} avoids both requirements, because it is stated in
the R\'enyi risk. The tilt of Lemma~\ref{lem:logistic:loc} produces directly
the risk that appears on the left-hand side, and the prior is the unrestricted
Gaussian, so that no support condition arises. It is worth noting that the
restriction is a limitation of the present route from the R\'enyi to the
Kullback-Leibler risk, through a comparison asked to hold globally. It is not
an impossibility statement.

\end{document}